\pdfoutput=1
\documentclass[11pt]{article}

\usepackage[letterpaper,margin=1in]{geometry}
\usepackage[round]{natbib}
\usepackage[utf8]{inputenc}
\usepackage[T1]{fontenc}
\usepackage{amsmath,amssymb,amsthm,mathtools}
\usepackage{bm}
\usepackage{enumitem}
\usepackage{array,booktabs,longtable}
\usepackage{nicefrac}
\usepackage{microtype}
\usepackage{xcolor}
\usepackage{graphicx}
\usepackage{float}
\usepackage{hyperref}
\usepackage{url}
\usepackage{cleveref}
\hypersetup{
  colorlinks=true,
  linkcolor=blue,
  citecolor=blue,
  urlcolor=blue
}

\theoremstyle{plain}
\newtheorem{theorem}{Theorem}[section]
\newtheorem{proposition}[theorem]{Proposition}
\newtheorem{lemma}[theorem]{Lemma}

\theoremstyle{definition}

\newtheorem{remark}[theorem]{Remark}

\crefname{assumption}{assumption}{assumptions}
\Crefname{assumption}{Assumption}{Assumptions}

\DeclareMathOperator{\tr}{tr}
\DeclareMathOperator{\Cov}{Cov}
\DeclareMathOperator{\Var}{Var}

\newcommand{\E}{\mathbb{E}}
\newcommand{\R}{\mathbb{R}}
\newcommand{\N}{\mathcal{N}}
\newcommand{\Prob}{\mathbb{P}}
\newcommand{\dx}{d_x}
\newcommand{\dy}{d_y}
\newcommand{\dop}{d_{\oplus}}
\newcommand{\Id}[1]{I_{#1}}
\newcommand{\T}{^{\!\top}}
\newcommand{\op}{\mathrm{op}}
\newcommand{\Wa}{\frac{W_1}{\sqrt{2\dx}}}
\newcommand{\WaT}{\frac{W_1\T}{\sqrt{2\dx}}}
\newcommand{\Wcat}{W^{\oplus}}
\newcommand{\Wop}{\frac{W^{\oplus}}{\sqrt{2\dx}}}
\newcommand{\WopT}{\frac{(W^{\oplus})\T}{\sqrt{2\dx}}}
\newcommand{\Zcat}{Z^{\oplus}}
\newcommand{\Zop}{\frac{Z^{\oplus}}{\sqrt n}}
\newcommand{\ZopT}{\frac{(Z^{\oplus})\T}{\sqrt n}}
\newcommand{\Eop}{\frac{v_\ell\Omega}{\sqrt n}}
\newcommand{\EopT}{\frac{v_\ell\Omega\T}{\sqrt n}}
\newcommand{\Cone}{\mathcal{C}_1}
\newcommand{\Ctwo}{\mu_1\Tdet}
\newcommand{\CtwoT}{\mu_1\Tdet\T}
\newcommand{\Tdet}{T}

\newcommand{\Kop}{\mathcal{K}^{\oplus}}
\newcommand{\Aop}{\mathcal{A}}
\newcommand{\Gt}{\mathcal{G}}
\newcommand{\Kbar}{\overline{K}}

\newcommand{\Ltest}{\mathcal{L}_{\mathrm{test}}^{t}}
\newcommand{\Ltrain}{\mathcal{L}_{\mathrm{train}}^{t}}
\newcommand{\Anorm}{\|\bar A\|_F^2}%
\newcommand{\Lridge}{\mathcal{L}_{\mathrm{ridge}}^{t}}

\newcommand{\ME}{\mathrm{ME}}
\newcommand{\PV}{\mathrm{PV}}
\newcommand{\RA}{\mathrm{RA}}
\newcommand{\RE}{\mathrm{RE}}
\newcommand{\Utilde}{\widetilde{U}}
\newcommand{\Vtilde}{\widetilde{V}}
\newcommand{\Ahat}{\widehat{A}_{\lambda}}
\newcommand{\act}{\varrho}
\newcommand{\Lpen}{L^{\mathrm{B}}}
\newcommand{\Bpen}{\mathcal{B}^{\mathrm{B}}}
\newcommand{\etapen}{\eta^{\mathrm{B}}}

\title{Memorization and Malign Generalization in Conditional Diffusion Models with Random Features}
\author{
  Gwangho Kim\\
  Department of Computer Science\\
  Hanyang University, Seoul\\
  \texttt{ggh1999@hanyang.ac.kr}
  \and
  Sungyoon Lee\\
  Department of Computer Science\\
  Hanyang University, Seoul\\
  \texttt{sungyoonlee@hanyang.ac.kr}
}
\date{}

\begin{document}
\maketitle

\begin{abstract}
Conditional diffusion models generate diverse, novel, and high-quality samples under prescribed conditions. 
However, theoretical understanding of their memorization and generalization remains limited, while recent works have characterized these behaviors primarily in unconditional settings. 
In this work, we analyze a random-feature conditional score model in the high-dimensional proportional limit, deriving asymptotic expressions for training and test losses. 
By decomposing the test loss, we show that in the overparameterized regime, increasing model width improves prediction of the condition-dependent mean while reducing within-condition prediction variance, a phenomenon we term ``malign generalization.'' 
We further show that in the infinite-width limit, the model predicts the conditional mean as accurately as ridge regression of the data on the condition, while its within-condition variance vanishes. 
Finally, analyzing the training loss reveals that more informative conditions lead to memorization of training samples at smaller widths. 
These theoretical findings are supported by experiments with U-Net architectures on realistic data.

\end{abstract}

\section{Introduction}
\label{sec:introduction}
Diffusion models \citep{sohl2015deep,ho2020denoising,song2021scorebased}
have become a widely adopted paradigm for generative modeling across multiple
modalities \citep{rombach2022high,zhang2023survey,ho2022video,hu2026elf}.
Conditional diffusion models further enable generation under prescribed conditions, supporting applications in text-to-image synthesis \citep{saharia2022photorealistic}, molecule generation \citep{hoogeboom2022equivariant}, protein design \citep{watson2023novo}, PDE forecasting and data assimilation \citep{shysheya2024conditional}, and uncertainty quantification \citep{han2022card}.

\begin{figure}[!t]
  \centering
  \includegraphics[width=\textwidth]{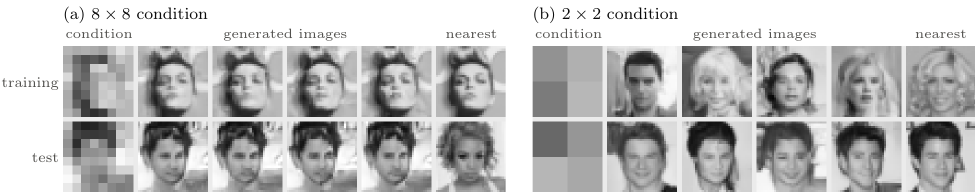}
  \caption{CelebA super-resolution to $32\times32$; each row shows a condition, four generated samples, and the training image nearest to any of them. Models in (a) and (b) are trained identically except for condition resolution. (a) With the informative $8\times8$ condition, samples for a training condition copy its training image, while samples for a test condition are novel and aligned with the condition but nearly identical (malign generalization). (b) With the weaker $2\times2$ condition, samples are diverse under both training and test conditions.}
  \label{fig:celeba-teaser}
\end{figure}

Understanding how diffusion models generalize beyond their training data is
of considerable interest from both practical and theoretical perspectives. From a practical
perspective, characterizing when memorization occurs and how to mitigate it is
crucial for addressing privacy and copyright concerns
\citep{carlini2023extracting,somepalli2023diffusion,
somepalli2023understanding,webster2023reproducible,
wen2024detecting,kim2026localizing}. From a theoretical perspective, explaining such generalization is nontrivial because the minimizer of the empirical denoising score matching (DSM) objective over unrestricted score functions is the \textit{empirical score function (ESF)}:
\begin{equation*}
    s_t^{\mathrm{emp}}(x)
    := \nabla_x \log p_t^{\mathrm{emp}}(x),
\end{equation*}
where $p_t^{\mathrm{emp}}$ is the forward-noised empirical distribution at diffusion time $t$, with $p_0^{\mathrm{emp}}(x):=\frac1n\sum_{i=1}^n\delta(x-x_i)$ for training samples $\{x_i\}_{i=1}^n$.
The ideal reverse process driven by this score recovers the empirical data distribution and therefore reproduces training samples rather than generating novel ones \citep{pidstrigach2022score}. Recent studies have investigated mechanisms by which practical score estimators deviate from the ESF and nevertheless generate novel, high-quality samples
\citep{scarvelis2025closedform,wu2025taking,bonnaire2025,
george2026,kadkhodaie2024generalization,kamb25a,
niedoba25a,lukoianov2026locality}.

Despite this progress, memorization and generalization in conditional diffusion
models remain largely unexplored. Consider a training set
$\mathcal{D}_n=\{(x_i,y_i)\}_{i=1}^n$ of pairs drawn i.i.d.\ from a joint distribution
$P_{xy}$, where $y_i$ is the condition paired with the sample $x_i$. Throughout this paper,
we use the term \textit{conditional diffusion model} to refer to settings satisfying the
following two properties:
\begin{itemize}
    \item Each training condition $y_i$ is distinct and paired with a single
    sample $x_i$.
    Accordingly, the empirical conditional distribution at $y_i$ is
    $p_0^{\mathrm{emp}}(x\mid y_i)=\delta(x-x_i)$. After forward noising, its
    conditional ESF at diffusion time $t$ is
    \begin{equation*}
        s_t^{\mathrm{emp}}(x,y_i)
            := \nabla_x \log p_t^{\mathrm{emp}}(x\mid y_i)
            = -\frac{x-a_t x_i}{h_t},
    \end{equation*}
    with $a_t,h_t$ defined in~\eqref{eq:ou}.
    Consequently, at the training conditions, the training DSM loss is, up to
    the factor $h_t$, exactly the mean-squared error between the learned score
    and the conditional ESF.

    \item At test time, the model is evaluated at the conditions $y$ of test pairs $(x,y)$ drawn i.i.d.\ from $P_{xy}$, so $y\notin\{y_i\}_{i=1}^n$ and it should generalize to them.
\end{itemize}
Both properties hold almost surely for continuously distributed conditions, and approximately in many practical applications such as text-to-image generation.
Standard class-conditional models satisfy neither property, since each class is shared by many training samples and only training conditions are used at test time. Consequently, prior work on unconditional diffusion models or class-conditional models may not directly apply to the conditional setting considered here.

Conditional diffusion models are also closely related to regression and uncertainty
quantification~\citep{han2022card}, as their samples for a condition estimate both the conditional
mean and the spread around it.
This raises two questions: (i) At test conditions, do overparameterized conditional models exhibit
benign overfitting as in regression~\citep{bartlett2020benign}, or learn only part of the conditional
distribution?
(ii) How does the information carried by the condition affect memorization of training samples paired with training conditions?
Here, we define memorization as the replication of a training sample by a generated sample, whether or not the true conditional distribution itself concentrates near it.
Our contributions addressing these questions are as follows:
\begin{itemize}
    \item Section~\ref{sec:data-model} introduces a data model that is simple and
    tractable yet captures key phenomena of conditional diffusion models in
    realistic settings.
    \item Section~\ref{sec:test-decomposition} derives asymptotically precise
    expressions for the test loss of a random-feature
    network~\citep{rahimi2007random} in the high-dimensional proportional regime,
    and its decomposition into conditional-mean and residual errors. Answering (i),
    in the overparameterized regime wider models predict the conditional mean more
    accurately while their prediction variance falls below that of the true
    denoiser, which we term \emph{malign generalization}.
    \item Section~\ref{sec:large-width} further shows that, as $\psi_p\to\infty$ at fixed
    $\psi_n$, the prediction variance vanishes, while the conditional mean becomes as accurate
    as ridge regression of $x$ on $y$.
    \item Section~\ref{sec:train-replication} examines the training loss. Answering (ii), a more informative
    condition lowers it, so memorization sets in at
    smaller widths.
    \item Section~\ref{sec:experiments} shows both phenomena in full reverse-process
    generation with the random-feature model and with U-Nets on CelebA.
\end{itemize}

\subsection{Related work}
\label{sec:related-work}
\paragraph{Distribution estimation.}
One line of theory bounds the sampling error given an $L^2$-accurate score
\citep{lee2022convergence, chen2022sampling, lee2023convergence, benton2024nearly};
another bounds the score-estimation error itself, giving end-to-end rates for smooth
densities and linear subspaces \citep{oko2023diffusion, chen2023score}, later extended to
nonlinear manifolds \citep{tang24a, azangulov2024convergence}.
Conditional counterparts cover reward-directed conditioning \citep{yuan2023reward}, classifier-free guidance \citep{fu2024unveil}, manifold-supported data \citep{tang2024conditional}, and conditional transformers \citep{hu2025statistical}.
These data models satisfy both properties in \Cref{sec:introduction}, but the results are non-asymptotic bounds or minimax rates at fixed dimension; they do not describe the proportional regime we study, nor the memorization that arises once capacity exceeds the sample size.

\paragraph{Memorization and generalization.}
A complementary line asks why trained networks depart from the ESF at all.
Proposed mechanisms include smoothing of the empirical score
\citep{scarvelis2025closedform, chen2026interpolation, farghly2026diffusion}, implicit
regularization from large learning rates \citep{wu2025taking},
and inductive biases of the network such as locality
\citep{kadkhodaie2024generalization, kamb25a, niedoba25a, lukoianov2026locality}.
A second strand asks when the departure occurs, identifying transitions between memorization and
generalization that are governed by sample size and model capacity
\citep{zhang2023emergence, buchanan2026edge, ye2026provable},
and bounding the generalization gap in terms of both \citep{li2023generalization}.
Closest to our setting, \citet{george2026} derive precise learning curves for denoising score matching with random features in the proportional limit; \citet{george2026asymptotic} extend them to manifold data and \citet{urfin2026double} to an arbitrary finite number of noise samples per data point.
\citet{merger2025generalization} analyze linear diffusion models on Gaussian data with power-law covariance spectra.
\citet{bonnaire2025} and \citet{latourelle2026generalization} analyze the training dynamics of random-feature and kernel models, respectively, and show that finite training time acts as an implicit regularization that delays memorization.
These results suggest that overfitting is not benign in diffusion models, and \citet{farghly2026benign} show that this holds in more general settings. 
All of these analyses are unconditional, so neither of the two properties in \Cref{sec:introduction} holds.
We extend the random-feature analysis of \citet{george2026} to conditional diffusion models satisfying both properties.

\section{Preliminaries}
\label{sec:preliminaries}
In this section, we briefly review diffusion models and fix the notation used throughout the paper.
\paragraph{Diffusion process.}
Let $P(\cdot\mid y)$ denote the conditional distribution of $x\in\mathbb{R}^{d_x}$ given the
condition $y\in\mathbb{R}^{d_y}$ under $P_{xy}$. Given $y$, the Ornstein--Uhlenbeck process
$dX_t=-X_t\,dt+\sqrt{2}\,dB_t$ with $X_0\sim P(\cdot\mid y)$ and a standard Wiener process
$B_t$ has conditional marginals $P_t(\cdot\mid y)$, with densities $p_t(\cdot\mid y)$, that
converge to $\N(0,\Id{\dx})$ as $t\to\infty$. For a clean sample $x\sim P(\cdot\mid y)$,
\begin{equation}
a_t := e^{-t},\qquad h_t := 1-e^{-2t},\qquad
x_t(\epsilon) := a_t x+\sqrt{h_t}\,\epsilon,\qquad
\epsilon\sim\N(0,\Id{\dx}),
\label{eq:ou}
\end{equation}
with $\epsilon$ independent of $(x,y)$. New samples are generated by running, up to a large
time $T$, the reverse-time process
\begin{equation}
d\overleftarrow{X}_t
=\big[\overleftarrow{X}_t
+2\nabla\log p_{T-t}(\overleftarrow{X}_t\mid y)\big]dt
+\sqrt{2}\,d\bar B_t
\label{eq:reverse-sde}
\end{equation}
from $\overleftarrow{X}_0\sim P_T(\cdot\mid y)$, where $\bar B_t$ is a standard Wiener process
and the unknown conditional score $\nabla\log p_t(\cdot\mid y)$ is replaced by a network
$s_\theta^t(x_t,y)$. By Tweedie's formula, $\E[x\mid x_t,y]=\{x_t+h_t\nabla\log p_t(x_t\mid y)\}/a_t$,
so $s_\theta^t$ induces the denoiser
\begin{equation}
x_\theta^t(x_t,y):=\frac{x_t+h_t\,s_\theta^t(x_t,y)}{a_t}.
\label{eq:denoiser}
\end{equation}

\paragraph{Conditional score matching.}
Following prior work \citep{cui2024analysis, george2026, bonnaire2025}, we study score learning at a fixed diffusion time $t$, with a specific focus on the conditional setting. We do not consider learning the unconditional score with a null condition, as used in classifier-free guidance \citep{ho2022classifier}.

Let $\mathcal{D}_n=\{(x_i,y_i)\}_{i=1}^n$ consist of $n$ i.i.d. data points $(x_i,y_i)\sim P_{xy}$.
Writing $x_{i,t}(\epsilon):=a_tx_i+\sqrt{h_t}\epsilon$, the training loss for conditional diffusion models is
\begin{equation*}
\Ltrain(\theta)
\coloneqq \frac{1}{\dx n}\sum_{i=1}^n\E_\epsilon
\left\|\sqrt{h_t}\,s_\theta^t(x_{i,t}(\epsilon),y_i)+\epsilon\right\|^2.
\end{equation*}

Unlike in unconditional models, where the training objective differs from the empirical score error by additive constants, here the training loss is, up to the factor $h_t$, the mean-squared error with respect to the conditional ESF, the score of
$P^{\mathrm{emp}}_t(\cdot\mid y_i)=\N(a_tx_i,h_t\Id{\dx})$, as noted in Section~\ref{sec:introduction}. Integrated over time, it
bounds how far the generated distribution at the training conditions is from this
empirical distribution. As in \citet{song2021maximum}, let $\widehat P_t(\cdot\mid y_i)$
be the marginals of the backward process~\eqref{eq:reverse-sde} with the learned score,
started from $\widehat P_T\approx P^{\mathrm{emp}}_T$. Then,
for $0<t_0<T$ and up to this mismatch at time $T$, Girsanov's theorem \citep{oksendal2013stochastic} gives
\begin{equation*}
\frac1n\sum_{i=1}^n\mathrm{KL}\bigl(P^{\mathrm{emp}}_{t_0}(\cdot\mid y_i)\,\big\|\,\widehat P_{t_0}(\cdot\mid y_i)\bigr)
\le\dx\int_{t_0}^T\frac{\Ltrain}{h_t}\,dt,
\end{equation*}
and the total variation distance is bounded through Pinsker's inequality. A small
training loss along the reverse chain thus forces the samples for $y_i$ towards $x_i$.
We therefore use the training loss as a measure of memorization at the training conditions,
and illustrate our results at a fixed small $t$, since the weight $1/h_t$ makes small times
dominate this bound; Appendix~\ref{sec:app-interpretation} illustrates them at other $t$.

Moreover, the forward process adds noise only to $x_i$, while $y_i$ remains unchanged. 
This introduces a qualitative distinction from unconditional settings, especially given that diffusion noise plays a central role in preventing benign overfitting \citep{farghly2026benign}.

\section{Problem setup}
In this section, we describe our theoretical setting and define the quantities of interest analyzed in Section~\ref{sec:theorem}.
\label{sec:setting}
\subsection{Data model}
\label{sec:data-model}
Let $U_\star\in\R^{\dx\times\dy}$ have orthonormal columns, with
$0<\dy\le\dx$, and put $\Pi:=U_\star U_\star\T$ and
$\Pi_\perp:=\Id{\dx}-\Pi$. For independent standard Gaussian vectors
$y\in\R^{\dy}$ and $\xi\in\R^{\dx}$, consider
\begin{equation}
x=By+\Sigma^{1/2}\xi,\qquad B=cU_\star,\qquad
\Sigma=(1-c^2)\Pi+\Pi_\perp,\qquad 0\le c\le1.
\label{eq:data-model}
\end{equation}
Thus, $x \mid y \sim \mathcal{N}(cU_\star y,\Sigma)$, while the marginal
distribution $x \sim \mathcal{N}(0,I_{d_x})$ remains unchanged
regardless of $c$ and $d_y$. The condition explains variance $c^2$ along each
observed direction, in the range of $\Pi$, leaving residual variance $1-c^2$ there and
variance $1$ on each unobserved direction, in its orthogonal complement. The parameter $c$ controls the
accuracy of the condition, whereas $\psi_y:=\lim\dy/\dx\in(0,1]$ controls
the fraction of observed directions. These are distinct ways to strengthen
the condition at a fixed marginal of $x$.

Two special cases are standard inverse problems. At $\dy=\dx$ and
$U_\star=I$, the isotropic model $B=cI$, $\Sigma=(1-c^2)I$ is Gaussian denoising:
equivalently, $y=cx+\sqrt{1-c^2}\,\xi'$ with $\xi'\sim\N(0,\Id{\dx})$
independent of $x$.
At $c=1$, $y=U_\star\T x$ is a noiseless observation of $\dy$ orthonormal
directions of $x$, an idealized model of super-resolution. We define the
signal-to-noise ratio of the condition as the ratio of the variance of $x$ it
explains to the residual variance,
\begin{equation}
\mathrm{SNR}(\psi_y,c):=\lim\frac{\tr(BB\T)}{\tr\Sigma}=\frac{c^2\psi_y}{1-c^2\psi_y},
\label{eq:SNR}
\end{equation}
so that $\mathrm{SNR}(1,c)=c^2/(1-c^2)$ in Gaussian denoising and
$\mathrm{SNR}(\psi_y,1)=\psi_y/(1-\psi_y)$ in super-resolution. In the
isotropic case we write $\mathrm{SNR}$ for $\mathrm{SNR}(1,c)$.

With $\eta_t:=x_t-a_tBy$ and $\Sigma_t:=a_t^2\Sigma+h_tI$, the forward distribution is
\begin{equation*}
x_t\mid y\sim\N(a_tBy,\Sigma_t),\quad
\nabla\log p_t(x_t\mid y)=-\Sigma_t^{-1}\eta_t.
\label{eq:testdraw}
\end{equation*}
We fix $t>0$, so $\Sigma_t$ is invertible even at $c=1$.
Write $\upsilon_{\Sigma,d}^2:=\dx^{-1}\tr\Sigma=1-c^2\dy/\dx$ and
$\upsilon_\Sigma^2:=1-c^2\psi_y$ for its limit.

\subsection{Random-feature model}
\label{sec:rf-score-learning}

We approximate the conditional score with the two-layer random-feature model \citep{rahimi2007random}
\begin{equation}
s_A(x,y\mid W_1,W_2) := \frac{A}{\sqrt p}\,
\act\!\left(\frac{W_1x}{\sqrt{2\dx}}+\frac{W_2y}{\sqrt{2\dy}}\right).
\label{eq:model}
\end{equation}

Here $\act$ acts coordinatewise. The normalizations by $\sqrt{2d_x}$ and $\sqrt{2d_y}$ ensure that the preactivation has unit variance. The hidden weights
$W_1\in\R^{p\times \dx}$ and $W_2\in\R^{p\times \dy}$ have independent standard Gaussian
entries, are drawn independently of the training sample, and remain fixed; only
$A\in\R^{\dx\times p}$ is trained. We work in the proportional high-dimensional limit
$n/\dx\to\psi_n$, $p/\dx\to\psi_p$, and $\dy/\dx\to\psi_y$.
At fixed $t$, we consider the minimizer of the ridge-regularized DSM loss:
\begin{equation}
\Lridge(A)
:=\Ltrain(A)+\frac{h_t\lambda}{\dx p}\|A\|_F^2,
\quad \lambda>0 .
\label{eq:loss}
\end{equation}

\paragraph{Activation assumption.}
For $z\sim\N(0,1)$, assume that $\act(z)$ admits a Hermite expansion in
$L^2$. Following \citet{bonnaire2025,george2026}, we assume
$\E_z[\act(z)]=0$ and define
$\mu_1:=\E_z[z\act(z)]$.

\paragraph{Proof techniques.}
Following the overall strategy of \citet{george2026}, our analysis uses two main ingredients: the Gaussian equivalence principle (GEP) \citep{gerace2020generalisation, goldt2022gaussian, mei2022generalization} and the
theory of linear pencils \citep{bodin2024}. We use the GEP to replace the
nonlinear functions of the training preactivations by their Gaussian equivalent
(Appendix~\ref{sec:hermite}). To evaluate the resulting resolvent traces, we
construct a six-block linear pencil and apply Result~3.1 of \citet{bodin2024}.
\citet{george2026} use a specialized instance of the same framework, whereas our conditional setting requires its general block formulation.

\subsection{Quantities of interest}
\label{sec:quantities-of-interest}
Let $\Ahat$ denote the minimizer of $\Lridge$.
While optimization is performed on the regularized objective, we focus on the unregularized empirical DSM loss $\Ltrain(\Ahat)$ alongside the test DSM loss:
\begin{equation}
\Ltest(A) \coloneqq \frac{1}{\dx} \E \left\| \sqrt{h_t} \, s_A(x_t(\epsilon), y) + \epsilon \right\|^2,
\label{eq:testerr}
\end{equation}
where the expectation is taken over the diffusion noise $\epsilon$ and an independent test pair $(x, y) \sim P_{xy}$.

To provide further insights into $\Ltest(A)$, we express it in terms of the denoiser $x_A := (x_t + h_t s_A(x_t, y)) / a_t$ induced by $s_A$ via \eqref{eq:denoiser}. Specifically, the test loss admits the orthogonal decomposition:
\begin{equation*}
\Ltest(A) = \frac{a_t^2}{h_t \dx} \E \|x_A - x\|^2 = \frac{a_t^2}{h_t} \big\{ \ME(A) + \RE(A) \big\},
\end{equation*}
where the \emph{conditional-mean error} ($\ME$) and \emph{residual error} ($\RE$) are defined as
\begin{align}
\ME(A) &:= \frac{1}{\dx} \E \big\| \E[x_A \mid y] - By \big\|^2, \label{eq:me} \\
\RE(A) &:= \frac{1}{\dx} \E \big\| x_A - \E[x_A \mid y] - (x - By) \big\|^2. \label{eq:re}
\end{align}
Here, $\E[x_A \mid y]$ denotes the conditional expectation of $x_A$ over both $\epsilon$ and $x \sim P(\cdot \mid y)$ at a fixed condition $y$. Intuitively, $\ME$ quantifies how accurately the denoiser captures the true conditional mean when averaged over within-condition variation, whereas $\RE$ evaluates how faithfully the sample-wise fluctuations of $x_A$ match the true residual $x - By$.

To further understand the variability of the model under test conditions, we introduce the \emph{prediction variance}:
\begin{equation}
\PV(A) := \frac{1}{\dx} \E \big\| x_A - \E[x_A \mid y] \big\|^2.
\label{eq:pv}
\end{equation}
The residual error then expands as $\RE(A) = \upsilon_{\Sigma,d}^2 + \PV(A) - 2\RA(A)$, where $\RA(A) := \dx^{-1} \E \big[(x - By)\T (x_A - By)\big]$ represents the residual alignment.

\begin{remark}
\label{rem:single-step}
Formally, $\ME$ and $\PV$ describe single-step denoising properties at a fixed time $t$ rather than the full multi-step trajectories of the reverse process. Nevertheless, they conceptually capture the accuracy and within-condition diversity of generated samples, and their trends in width also appear in samples from the full reverse process (Section~\ref{sec:exp-rfnn}).
\end{remark}  

\section{Main results}
\label{sec:theorem}
Section~\ref{sec:test-decomposition} derives the limits of the quantities in
Section~\ref{sec:quantities-of-interest} and investigates the model's behavior at test
conditions in the overparameterized regime. Section~\ref{sec:large-width} describes this behavior as $\psi_p\to\infty$ at fixed $\psi_n$.
Section~\ref{sec:train-replication} derives the
training loss for large $\psi_p$ and $\psi_n$ to investigate the effect of the condition
($c,\psi_y$) on memorization.

\begin{figure}[t]
  \centering
  \includegraphics[width=\textwidth]{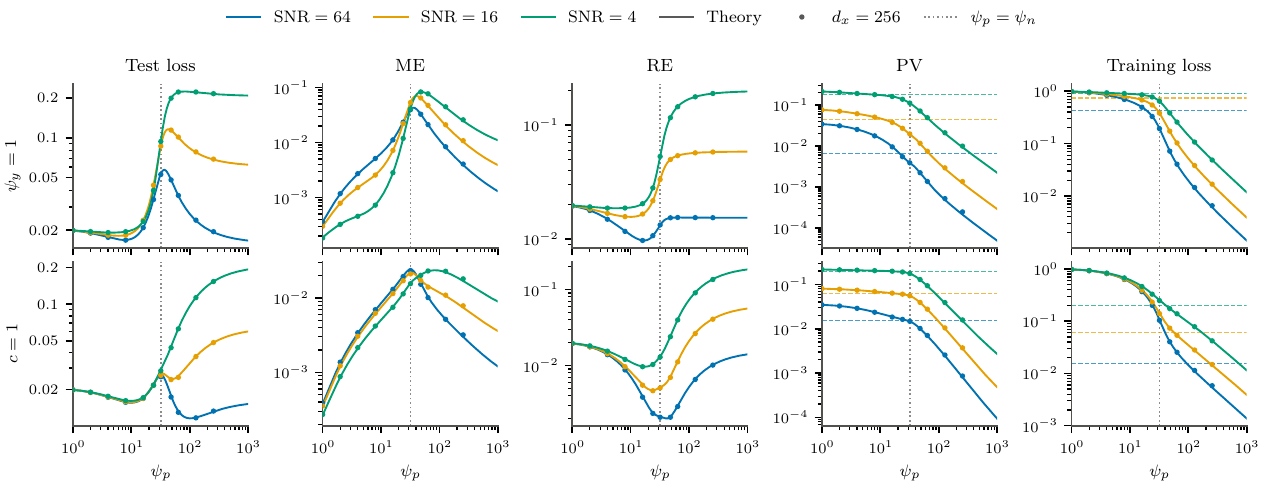}
  \caption{Test loss $(h_t/a_t^2) \Ltest = \ME + \RE$, ME~\eqref{eq:me}, RE~\eqref{eq:re}, PV~\eqref{eq:pv}, and training loss $\Ltrain$ from Theorem~\ref{thm:main} against width $\psi_p$ for different SNRs~\eqref{eq:SNR} at $t=0.01$, $\act = \tanh$, $\lambda=10^{-5}$, and $\psi_n=32$. 
\textbf{Top}: Gaussian denoising ($\psi_y=1$), varying $c$ so that $\mathrm{SNR}(1,c) \in \{4,16,64\}$. 
\textbf{Bottom}: super-resolution ($c=1$), varying $\psi_y$ so that $\mathrm{SNR}(\psi_y,1) \in \{4,16,64\}$. 
 Dashed lines mark the $\PV$ and the training loss of the true denoiser $\E[x \mid x_t,y]$.   Dots denote finite random-feature simulations at $d_x=256$, averaged over five realizations; their min--max bars are smaller than the dots. Appendix~\ref{sec:app-heatmaps} shows them as heatmaps over a wider range of SNR and for other $\psi_n$.}
  \label{fig:width-dependent-errors}
\end{figure}
\subsection{Test loss and malign generalization}
\label{sec:test-decomposition}
We first split the quantities of Section~\ref{sec:quantities-of-interest} into linear and nonlinear parts. 
Write $g:=W_1x_t/\sqrt{2\dx}+W_2y/\sqrt{2\dy}\in\R^p$ for the preactivation in
\eqref{eq:model}, so that $s_{\Ahat}=\frac{\Ahat}{\sqrt p}\act(g)$. 
With $W_1$ and $W_2$ fixed, every entry of $g$ is asymptotically standard Gaussian over a test draw $(y,x_t)$.
Writing $\act(z)=\mu_1z+r(z)$ entrywise, with $\mu_1z$ the first Hermite component and $r$ the remaining nonlinear term, splits the score into $s_{\Ahat}^{\mathrm{lin}}:=\mu_1\frac{\Ahat}{\sqrt p}g$ and
$s_{\Ahat}^{\mathrm{nl}}:=\frac{\Ahat}{\sqrt p}r(g)$, and the denoiser
$x_{\Ahat}=(x_t+h_ts_{\Ahat})/a_t$ into $x_{\Ahat}^{\mathrm{lin}}:=(x_t+h_ts_{\Ahat}^{\mathrm{lin}})/a_t
=G_xx_t+G_yy$ and $x_{\Ahat}^{\mathrm{nl}}:=(h_t/a_t)s_{\Ahat}^{\mathrm{nl}}$, where
\begin{equation}
G_x:=\nabla_{x_t}x_{\Ahat}^{\mathrm{lin}}=\frac1{a_t}\Big(I+h_t\mu_1\frac{\Ahat}{\sqrt p}\Wa\Big),\qquad
G_y:=\nabla_yx_{\Ahat}^{\mathrm{lin}}=\frac{h_t\mu_1}{a_t}\frac{\Ahat}{\sqrt p}\frac{W_2}{\sqrt{2\dy}} .
\label{eq:main-Gx}
\end{equation}
The linear component has conditional mean $\E[x_{\Ahat}^{\mathrm{lin}}\mid y]=\widehat By$ with
$\widehat B:=a_tG_xB+G_y$. Let $z=z_y+z_\eta$ with independent $z_y\sim\N(0,1-\tau/2)$ and
$z_\eta\sim\N(0,\tau/2)$, the asymptotic parts of an entry of $g$ determined by $y$ and by
$\eta_t=x_t-a_tBy$, where $\tau:=a_t^2\upsilon_\Sigma^2+h_t$. We set $v^2:=\E[r(z)^2]$,
$v_y^2:=\E_{z_y}\big[\E_{z_\eta}[r(z_y+z_\eta)]^2\big]$, and $\bar A:=\Ahat/\sqrt{\dx p}$, the normalized second-layer weights, so that the weight norm is $\Anorm=\|\Ahat\|_F^2/(\dx p)$.

\begin{lemma}
\label{lem:main-decomp}
With $\simeq$ denoting equivalence in high dimensions, defined in Appendix~\ref{sec:proofs}, the nonlinear component of the score satisfies
\begin{gather*}
\frac1\dx\E\|s_{\Ahat}^{\mathrm{nl}}\|^2\simeq v^2\Anorm,\qquad
\frac1\dx\E\big\|\E[s_{\Ahat}^{\mathrm{nl}}\mid y]\big\|^2\simeq v_y^2\Anorm,
\end{gather*}
and
\begin{align*}
\ME(\Ahat)&\simeq\frac{\|\widehat B-B\|_F^2}{\dx}
+\frac{h_t^2}{a_t^2}\,v_y^2\Anorm,\\
\PV(\Ahat)&\simeq\frac1\dx\tr\big(G_x\Sigma_tG_x\T\big)
+\frac{h_t^2}{a_t^2}\,(v^2-v_y^2)\Anorm,\\
\RA(\Ahat)&\simeq\frac{a_t}\dx\tr\big(\Sigma G_x\big),\\
\Ltrain(\Ahat)&\simeq\frac{a_t}\dx\tr G_x-h_t\lambda \Anorm.
\end{align*}
\end{lemma}

Both $\ME$ and $\PV$ thus consist of a
part carried by the linear component of the denoiser and a part carried by its nonlinear
component, proportional to $\Anorm$: the conditional-mean energy
$\dx^{-1}\E\|\E[x_{\Ahat}^{\mathrm{nl}}\mid y]\|^2$ in $\ME$ and the within-condition variance
$\dx^{-1}\E\|x_{\Ahat}^{\mathrm{nl}}-\E[x_{\Ahat}^{\mathrm{nl}}\mid y]\|^2$ in $\PV$.
The linear part of $\ME$ is the error of the conditional-mean map $\widehat B$.
Given $y$, $x_t\sim\N(a_tBy,\Sigma_t)$, so the linear component
$x_{\Ahat}^{\mathrm{lin}}=G_xx_t+G_yy$ has conditional covariance $G_x\Sigma_tG_x\T$; the
linear part of $\PV$ is its normalized trace. 
Lemma~\ref{lem:main-decomp} still involves the random quantities $G_x$, $\widehat B$, and $\Anorm$; Theorem~\ref{thm:main} gives their deterministic limits.

\begin{theorem}[Proportional limits]
\label{thm:main}
Under the data model~\eqref{eq:data-model}, the activation assumption
of Section~\ref{sec:rf-score-learning}, and the assumptions of Appendix~\ref{sec:proofs}, fix $t>0$ and $\lambda>0$. Let $\Kbar$ and $\Kbar_\Sigma$ be the deterministic
functions of $(\mathbf q,z)$ given by~\eqref{eq:appendix-readouts}, defined through the solution of the
fixed-point equations~\eqref{eq:appendix-fixed-m}--\eqref{eq:appendix-fixed-rho}.
Evaluate them and their derivatives at $(\mathbf q,z)=(\mathbf0,-\lambda)$. Then the terms of
Lemma~\ref{lem:main-decomp} have the limits
\begin{align*}
\lim\E[\Anorm]&=\mu_1^2\,\partial_z\Kbar,\\
\lim\E\Big[\frac{\|\widehat B-B\|_F^2}{\dx}\Big]&=-\frac{h_t^2\mu_1^4}{a_t^2}\,\partial_{q_{\mathrm m}}\Kbar,\\
\lim\E\Big[\frac1\dx\tr\big(G_x\Sigma_tG_x\T\big)\Big]&=\upsilon_\Sigma^2+\frac{h_t}{a_t^2}
-\frac{2h_t^2\mu_1^2}{a_t^2}\Kbar
-2h_t\mu_1^2\Kbar_\Sigma-\frac{h_t^2\mu_1^4}{a_t^2}\,\partial_{q_{\mathrm r}}\Kbar,\\
\lim\E\Big[\frac{a_t}{\dx}\tr\big(\Sigma G_x\big)\Big]&=\upsilon_\Sigma^2-h_t\mu_1^2\Kbar_\Sigma,\\
\lim\E\Big[\frac{a_t}{\dx}\tr G_x\Big]&=1-h_t\mu_1^2\Kbar .
\end{align*}
Consequently, the limits of $\ME$, $\PV$, and $\RA$ follow
from Lemma~\ref{lem:main-decomp}, that of
$\RE$ from these, and
$\lim\E[\Ltest(\Ahat)]=\frac{a_t^2}{h_t}\lim\E[\ME(\Ahat)+\RE(\Ahat)]$.
The limit of $\Ltrain$ likewise follows from Lemma~\ref{lem:main-decomp}.
\end{theorem}

\begin{remark}
We expect the above quantities to concentrate around their expectations in the proportional limit; as in \citet{george2026}, a rigorous proof is beyond the scope of this work.
Finite random-feature simulations at $\dx=256$ support this
(Figure~\ref{fig:width-dependent-errors}).
\end{remark}
\begin{figure}[t]
  \centering
  \includegraphics[width=0.55\textwidth]{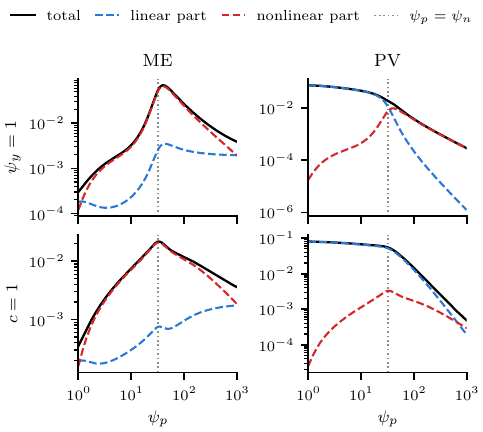}
  \caption{$\ME$ and $\PV$ from Theorem~\ref{thm:main} (black) split into the linear part
  (blue) and the nonlinear part (red) of Lemma~\ref{lem:main-decomp}, at $\mathrm{SNR}=16$ in
  the setting of Figure~\ref{fig:width-dependent-errors}.
  }
  \label{fig:me-pv-split}
\end{figure}
Figure~\ref{fig:width-dependent-errors} evaluates Theorem~\ref{thm:main} across the two specializations in Section~\ref{sec:data-model} at three values of $\mathrm{SNR}(\psi_y,c)$.
The test loss shows a double-descent peak near $\psi_p=\psi_n$ at high SNR, as in standard regression~\citep{belkin2019reconciling, hastie2022surprises, mei2022generalization}, and it is tempting to conclude that benign overfitting~\citep{bartlett2020benign} can occur in conditional diffusion models at high SNR.
Indeed, the conditional-mean error (ME) appears to exhibit benign overfitting, decreasing even in the overparameterized regime.
However, the residual error (RE) and prediction variance (PV) reveal a different picture.
As $\psi_p$ grows, the RE increases or plateaus, while the PV continues to decline well below that of the true denoiser, indicating that the model fails to capture the variation of $x$ given $y$.
Thus, in the overparameterized regime, a wider model predicts the conditional mean more accurately while increasingly failing to capture the within-condition variance---a phenomenon we term \textit{malign generalization}.

Figure~\ref{fig:me-pv-split} splits $\ME$ and $\PV$ into their linear and nonlinear parts.
As in random-feature ridge regression~\citep{mei2022generalization}, the weight norm
$\Anorm$ increases up to the interpolation threshold $\psi_p\approx\psi_n$ and decreases beyond
it. By Lemma~\ref{lem:main-decomp}, the nonlinear parts of $\ME$ and $\PV$ are constant
multiples of this norm and share its peak. Near the peak, the nonlinear part of $\ME$ exceeds its linear part by more than an order of magnitude, so the peak of $\ME$ and its shape around the peak come mostly from the nonlinear part.
Unlike in ridgeless random-feature regression, the peak is smoothed even as $\lambda\to0$: averaging over the diffusion noise acts as an implicit ridge. This ridge grows with $t$, and at
large $t$ the peak at $\psi_p\approx\psi_n$ disappears, while a peak appears near $p\approx\dx+\dy$
instead (Figure~\ref{fig:me-t} in Appendix~\ref{sec:app-fixed-n-pv}).
Beyond the peak, both parts of $\PV$ keep decreasing, whereas the linear part of $\ME$ tends to a nonzero value;
Section~\ref{sec:large-width} derives these limits.

\subsection{Large-width limit at fixed \texorpdfstring{$\psi_n$}{psi\_n}}
\label{sec:large-width}
The following proposition describes malign generalization as $\psi_p\to\infty$ at fixed $\psi_n$.

\begin{proposition}
\label{prop:main-fixed-n}
Under the assumptions of Theorem~\ref{thm:main}, suppose additionally that $s^2,v_\ell^2,\mu_1^2>0$ (Appendix~\ref{sec:activation-coefficients}). Fix $\psi_n,\lambda>0$, and let
$\psi_p\to\infty$ after the proportional limit. Then
$\Ltrain(\Ahat)\to0$ and $\Anorm\to0$, so $\dx^{-1}\|G_x\|_F^2\to0$, $\PV(\Ahat)\to0$, and the nonlinear part of $\ME(\Ahat)$ vanishes.
Moreover, $\ME(\Ahat)$ converges to the risk of ridge regression of $x$ on $y$ from the $n$
training pairs, with ridge $\lambda_{\mathrm{eff}}:=2\psi_yv_\ell^2/(\psi_n\mu_1^2)$.
\end{proposition}

Thus, at infinite width, the model predicts the conditional mean as accurately as ridge regression of
$x$ on $y$ does, but captures none of the variation of $x$ given $y$.
Indeed, $G_x\to0$ and, at test inputs, the nonlinear part vanishes, so the denoiser ignores $x_t$, and its conditional mean $G_yy$ is a regression
of $x$ on $y$ fitted to the $n$ training pairs, on which the nonlinearity of the activation acts
as the ridge $\lambda_{\mathrm{eff}}$. Note that Lemma~\ref{lem:main-decomp} relates the
nonlinear part to $\Anorm$ only at test inputs; at the training inputs, the nonlinear part does
not vanish, so the model still memorizes the training samples.
Appendix~\ref{sec:app-interpretation} proves these limits,
and Figure~\ref{fig:me-t} shows that, for $t$ not too small, even the linear part of $\ME$ exceeds this limit before approaching it.

Both $G_x\to0$ and the decrease of $\PV$ come from a mechanism specific to denoising. At a training input
$x_{i,t}(\epsilon)=a_tx_i+\sqrt{h_t}\,\epsilon$, the target is the same clean sample $x_i$ for
every noise draw $\epsilon$, so the denoiser fits the training samples only by becoming
insensitive to $x_t$ near them: by Stein's identity, $\Ltrain(\Ahat)+h_t\lambda\Anorm$ is
exactly the average of $a_t\dx^{-1}\tr\nabla_{x_t}x_{\Ahat}$ over the training inputs
(Lemma~\ref{lem:stein-train}). For a denoiser linear in $x_t$, this Jacobian is the same at
every input, so a small training loss forces the same small Jacobian at test inputs and hence a
small $\PV$. For a general nonlinear denoiser, the Jacobian at the training inputs need not constrain that at
test inputs. In our model, however, the training loss is asymptotically set by $G_x$
(Lemma~\ref{lem:main-decomp}), the same matrix that sets the linear part of $\PV$, so a small
training loss forces a small linear part of $\PV$ (Lemma~\ref{lem:pv-train-bound}). 
As the width grows, the training loss decreases, and the linear part of $\PV$ falls below the $\PV$ of the true denoiser; beyond the interpolation threshold, the nonlinear part decreases as
well (Figure~\ref{fig:me-pv-split}). Both vanish in the limit of
Proposition~\ref{prop:main-fixed-n}.

\subsection{Training loss and memorization}
\label{sec:train-replication}
The last column of Figure~\ref{fig:width-dependent-errors} evaluates the training loss of
Theorem~\ref{thm:main} and shows how strengthening the condition affects it. In both rows, the training loss
decreases with $\psi_p$, and a higher SNR gives a lower loss at every width. Figure~\ref{fig:heatmap-condition} in Appendix~\ref{sec:app-heatmaps} varies $c$ and $\psi_y$ jointly at fixed $\psi_p/\psi_n$. To see the effect of $c$ and $\psi_y$ explicitly, we consider large $\psi_p$ and $\psi_n$.

\begin{proposition}[Training loss at large $\psi_p$ and $\psi_n$]
\label{thm:train}
Under the assumptions of Theorem~\ref{thm:main}, suppose additionally that
$s^2,v_\ell^2,\mu_1^2>0$ (Appendix~\ref{sec:activation-coefficients}), write $r=\psi_p/\psi_n$, and
let $\rho_\star\in(0,1)$ be the positive root of
\begin{equation*}
v_\ell^2\rho_\star^2+[s^2+v_\ell^2(r-1)]\rho_\star-s^2=0.
\end{equation*}
Then, for $\psi_p,\psi_n\gg1$ and $0<\lambda\ll1$, uniformly in $r\in(0,\infty)$,
\begin{align}
\lim\E[\Ltrain(\Ahat)] &= R_{\mathrm{train}}(r)+O(\psi_p^{-1}+\psi_n^{-1}+\lambda), \notag \\
R_{\mathrm{train}}(r) &= \mathcal M\Big[\frac{a_t^2\rho_\star\varsigma}{h_t+a_t^2\rho_\star\varsigma}\Big],
\label{eq:Rtrain}
\end{align}
where $\mathcal M[f(\varsigma)]:=\psi_yf(1-c^2)+(1-\psi_y)f(1)$ is the average of $f$ over
the eigenvalues $\varsigma$ of $\Sigma$, the residual variances along the observed and unobserved directions.
\end{proposition}
Here $1-\rho_\star$ can be read as the effective fraction of training samples that the features fit, and $\rho_\star$ depends on $r$, $t$, and the activation, but not on $c$ or $\psi_y$ (Appendix~\ref{sec:app-rho}).

Thus, for large $\psi_p$ and $\psi_n$, $R_{\mathrm{train}}$ describes how the training loss
depends on width and on the condition:
\begin{itemize}
\item \emph{Population-score training loss as $r\to0^+$.} As $r\to0^+$, that is, for
$\psi_n\gg\psi_p\gg1$, $\rho_\star\to1$ and $R_{\mathrm{train}}$ equals the training loss
of the population score, $\mathcal M[a_t^2\varsigma/(h_t+a_t^2\varsigma)]$ (dashed in Figure~\ref{fig:width-dependent-errors}). This vanishes as
$\mathrm{SNR}\to\infty$, where the population score both generalizes and memorizes.
\item \emph{Full memorization as $r\to\infty$.} Since $\rho_\star$ decreases in $r$,
$R_{\mathrm{train}}$ decreases monotonically and vanishes as $r\to\infty$, that is, for
$\psi_p\gg\psi_n\gg1$.
\item \emph{Effect of the condition.} While $1-\rho_\star$ does not depend on $c$ or
$\psi_y$, the unfitted fraction $\rho_\star$ enters $R_{\mathrm{train}}$ multiplied by
$\varsigma$, and the summand $a_t^2\rho_\star\varsigma/(h_t+a_t^2\rho_\star\varsigma)$
increases in $\rho_\star\varsigma$. For fixed
$\psi_y$, increasing $c$ lowers the residual variance $\varsigma=1-c^2$ of the observed
directions, so $R_{\mathrm{train}}$ decreases in $c$. For fixed $c>0$, increasing
$\psi_y$ turns unobserved directions, with $\varsigma=1$, into observed ones, with
$\varsigma=1-c^2<1$, so $R_{\mathrm{train}}$ decreases in $\psi_y$. Both hold at every $r$.
\end{itemize}

\begin{remark}
At $c=1$, the observed directions do not contribute, and
$R_{\mathrm{train}}=(1-\psi_y)\,a_t^2\rho_\star/(h_t+a_t^2\rho_\star)$ is proportional to the
training loss of the population score, $(1-\psi_y)\,a_t^2/(h_t+a_t^2)$, so that $R_{\mathrm{train}}(r)/R_{\mathrm{train}}(0^+)$ does not depend on $\psi_y$.
If memorization were measured relative to the population score, that is, as fitting only
the part of $x_i$ that the condition leaves unexplained, it would be the same at every
$\mathrm{SNR}(\psi_y,1)$.
Replication, however, concerns the entire training image: prior
works~\citep{carlini2023extracting, somepalli2023understanding} define a sample as a copy when
it is close to the full image $x_i$, and the bound in Section~\ref{sec:preliminaries} controls
the distance to the empirical distribution itself. We therefore use the absolute training loss
as the measure of memorization.
For $\psi_y=1$, the condition acts beyond the number of directions:
$R_{\mathrm{train}}(r)/R_{\mathrm{train}}(0^+)=\rho_\star(h_t+a_t^2\varsigma)/(h_t+a_t^2\rho_\star\varsigma)$
increases in $\varsigma=1-c^2$, so a stronger condition also lowers the training loss relative
to the population score, as the residual variance becomes small compared with the diffusion
noise $h_t$.
\end{remark}

\section{Experiments}
\label{sec:experiments}
We evaluate samples from the full reverse process of the random-feature model~\eqref{eq:model} and of a U-Net~\citep{ronneberger2015u} on $32\times32$ CelebA. Experimental details and qualitative examples are in Appendices~\ref{sec:app-exp-details} and~\ref{sec:app-additional-results}.

\paragraph{Metrics.}
Each metric corresponds to a quantity in Section~\ref{sec:theorem}.
\emph{Memorization}, the counterpart of the training loss in
Section~\ref{sec:train-replication}, is measured at training conditions: following
\citet{yoon2023diffusion}, a sample $\hat x$ generated for a training condition is
considered memorized if $\|\hat x-x^{(1)}\|/\|\hat x-x^{(2)}\|<1/3$, where $x^{(1)}$ and
$x^{(2)}$ are its nearest and second-nearest training samples in the $L^2$ sense. \emph{Variance}, the counterpart of
$\PV$~\eqref{eq:pv}, is the sample variance across the samples generated for one test condition, averaged over coordinates and conditions. The \emph{conditional error}, the counterpart of $\ME$~\eqref{eq:me}, is the error of the mean of the samples for a test condition with respect to the true conditional mean; on CelebA, where the conditional mean is unknown, the original test image takes its place (Appendix~\ref{sec:app-celeba-details}).

\begin{figure}[t]
  \centering
  \includegraphics[width=0.92\textwidth]{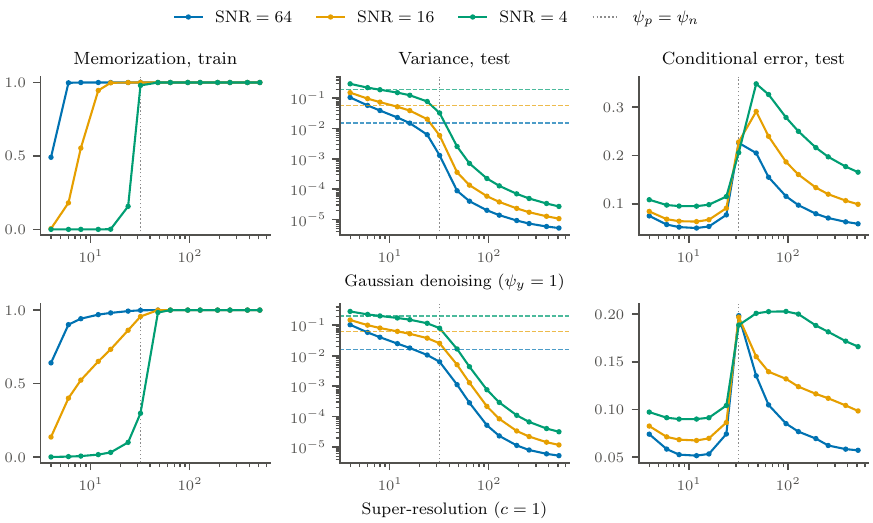}
  \caption{Full reverse-process generation with random-feature models (setting of
  Figure~\ref{fig:width-dependent-errors}, with $\dx=64$). Memorization is measured at
  training conditions (train), and variance and conditional error at test conditions (test).}
  \label{fig:rfnn-width}
\end{figure}

\subsection{Random features on Gaussian data}
\label{sec:exp-rfnn}
We sample with the random-feature model~\eqref{eq:model}, minimizing~\eqref{eq:loss} at every step of the reverse chain, on the data of Section~\ref{sec:data-model}. Figure~\ref{fig:rfnn-width} shows that the trends of Section~\ref{sec:theorem} appear in samples from the full reverse process: a higher SNR makes memorization set in at smaller widths; the variance falls below the true conditional variance $\upsilon_\Sigma^2$ and keeps decreasing with width; and the conditional error peaks near $\psi_p=\psi_n$ and decreases beyond it. Beyond the peak, wider models thus become more accurate but less diverse, the malign generalization of Section~\ref{sec:test-decomposition}. At test conditions, by contrast, none of the samples is a copy of a training image in any setting. The sweep starts at $\psi_p=4$, since for $\psi_p\le2$ the fitted score at large $t$ underestimates the noise and the generated samples blow up.

\begin{figure}[t]
  \centering
  \includegraphics[width=0.92\textwidth]{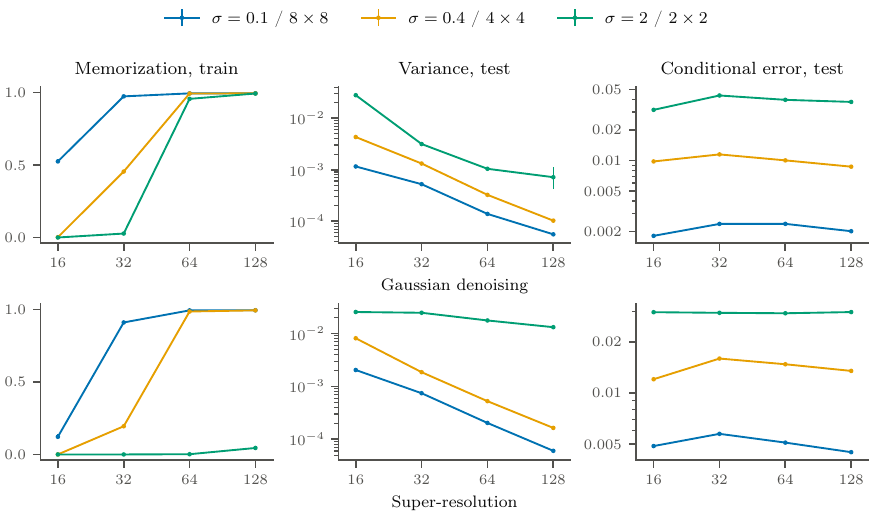}
  \caption{Full reverse-process generation with U-Nets on $32\times32$ CelebA, against the
      width $W\in\{16,32,64,128\}$ at $n=32768$, every model trained for $200$k steps.
      Markers are the mean and bars the min--max range over three seeds. Train and test have the same meaning as in Figure~\ref{fig:rfnn-width}.}
  \label{fig:celeba-width}
\end{figure}

\subsection{Denoising and super-resolution on CelebA}
\label{sec:exp-celeba}
We train U-Nets on grayscale $32\times32$ CelebA \citep{liu2015deep}, with the condition as a second input channel. In Gaussian denoising the condition is $y=x+\sigma\xi$ with $\sigma\in\{0.1,0.4,2\}$; in super-resolution it is the image average-pooled to $8\times8$, $4\times4$ or $2\times2$.
A larger $\sigma$ acts as a smaller $c$, and coarser pooling as a smaller $\psi_y$.

Figure~\ref{fig:celeba-width} shows the trends of Figure~\ref{fig:rfnn-width} in a U-Net: a stronger condition makes memorization set in at smaller widths, while the $2\times2$ condition is barely memorized up to $W=128$; the variance falls with width; and the conditional error shows a mild peak near the onset of memorization and decreases beyond it (Appendix~\ref{sec:app-celeba-details} compares other error metrics), again the malign generalization of Section~\ref{sec:test-decomposition}. 
As in Section~\ref{sec:exp-rfnn}, samples for test conditions are rarely copies of training images: averaged over seeds, their memorized fraction stays below $0.2\%$ in every setting except the $2\times2$ condition at $W=128$.\footnote{\label{fn:2x2-test}There, $1.9\%$ of the samples for test conditions are memorized ($4.5\%$ at training conditions): with only four dimensions, a test condition can lie close to a training condition.} 
Figure~\ref{fig:celeba-teaser} illustrates these behaviors in super-resolution.

\section{Conclusion}
\label{sec:conclusion}
We introduced a tractable data model~\eqref{eq:data-model} for conditional diffusion models
and derived the proportional limits of the test loss, its decomposition, and the training loss of a random-feature model. At test conditions, wider overparameterized models achieve lower conditional error while their prediction variance keeps decreasing, which we term malign generalization; at training conditions, more informative conditions induce memorization at smaller widths. Full reverse-process generation with random features and U-Nets shows the same trends. Our results call for further investigation into memorization and generalization mechanisms specific to conditional diffusion models.

\paragraph{Limitations and future work.}
\label{sec:app-limitations}
First, the data model~\eqref{eq:data-model} assumes an isotropic marginal and a residual
covariance with two distinct eigenvalues. Because the proof evaluates the traces separately on the eigenspaces of $\Sigma$, we expect our analysis to extend to residual covariances with finitely many distinct eigenvalues, with $\mathcal M$ still given by the normalized trace in~\eqref{eq:direction-average}. Extending the theory to anisotropic marginals and to non-Gaussian data remains an open problem.

Second, both our random-feature model and the U-Nets receive the condition $y$ via
concatenation with $x_t$. Practical models typically inject the condition through
cross-attention or adaptive normalization~\citep{rombach2022high,dhariwal2021diffusion};
how these mechanisms affect memorization and generalization remains open.

Third, as in \citet{george2026, bonnaire2025, urfin2026double}, the score
is a random-feature model with a fixed first layer, trained separately at each time $t$, and
we take the minimizer of the ridge-regularized loss. Our analysis therefore captures neither
feature learning nor parameter sharing across times; we leave these to future work.

Fourth, we assume that the training pairs $(x_i,y_i)$ are drawn i.i.d.\ from $P_{xy}$.
Large text-to-image training sets violate this assumption through sampling biases such as duplicated
images and repeated or templated captions, and replication in these models has been linked to
such duplication~\citep{carlini2023extracting,webster2023reproducible}. We do not address these
cases; extending the analysis to several samples per condition is left for future work.

\bibliographystyle{plainnat}
\bibliography{references}

\clearpage
\appendix
\numberwithin{equation}{section}

\newcommand{\apptocsec}[1]{\par\addvspace{5pt}\noindent
  \hyperref[#1]{\textbf{\makebox[2em][l]{\ref*{#1}}\nameref*{#1}}}\hfill\textbf{\pageref*{#1}}\par}
\newcommand{\apptocsub}[2][]{\par\noindent\hspace*{2em}%
  \hyperref[#2]{\makebox[2.6em][l]{\ref*{#2}}\nameref*{#2}}%
  \if\relax\detokenize{#1}\relax\else\ \textit{(#1)}\fi
  \leaders\hbox to 0.8em{\hss.\hss}\hfill\pageref*{#2}\par}
\begin{center}{\LARGE\bfseries Appendix}\end{center}
\vspace{0.5em}
\subsection*{Contents}
\begingroup\hypersetup{linkcolor=black}
\apptocsec{sec:notation}
\apptocsec{sec:proofs}
\apptocsub{sec:exact}
\apptocsub{sec:hermite}
\apptocsub[proof of Lemma~\ref*{lem:main-decomp}]{sec:app-decomp}
\apptocsub{sec:source}
\apptocsub{sec:pencil}
\apptocsub[proof of Theorem~\ref*{thm:main}]{sec:spectralclosure}
\apptocsec{sec:app-train-test-ratio}
\apptocsub[proof of Proposition~\ref*{thm:train}]{sec:app-joint-train}
\apptocsub{sec:app-rho}
\apptocsec{sec:app-interpretation}
\apptocsub{sec:app-score-form}
\apptocsub[proof of Proposition~\ref*{prop:main-fixed-n}]{sec:app-fixed-n-pv}
\apptocsec{sec:app-exp-details}
\apptocsub{sec:app-finite-rfnn-details}
\apptocsub{sec:app-rfnn-details}
\apptocsub{sec:app-celeba-details}
\apptocsec{sec:app-additional-results}
\apptocsub{sec:app-data-sweep}
\apptocsub{sec:app-heatmaps}
\apptocsub{sec:app-celeba-samples}
\apptocsec{sec:ai-use}
\endgroup
\clearpage

\section{Notation and conventions}
\label{sec:notation}
\begin{longtable}{@{}p{.27\textwidth}p{.68\textwidth}@{}}
\toprule Symbol & Meaning \\ \midrule\endhead
$\dx,\dy,n,p$ & Data dimension, condition dimension, sample size, and feature count.\\
$\psi_n,\psi_p,\psi_y$ & Limits of $n/\dx$, $p/\dx$, and $\dy/\dx$.\\
$U_\star,\Pi,\Pi_\perp$ & Orthonormal condition embedding and the two projectors in the main-text model.\\
$B,\Sigma$ & Conditional-mean map and residual covariance in $x=By+\Sigma^{1/2}\xi$.\\
$\upsilon_{\Sigma,d}^2,\upsilon_\Sigma^2$ & $\dx^{-1}\tr\Sigma$ and its limit.\\
$\Sigma_t,\tau$ & $a_t^2\Sigma+h_tI$ and its limiting normalized trace.\\
$\Sigma_x,I_x$ & $\operatorname{diag}(\Sigma,0_{\dy})$ and $\operatorname{diag}(I_{\dx},0_{\dy})$.\\
$W_1,W_2,\Wcat$ & Independent Gaussian hidden weights and $\Wcat=[W_1\ W_2]$.\\
$\kappa_d,\Tdet,\Gamma$ & Condition scaling $\sqrt{\dx/\dy}$, clean design map, and augmented noisy covariance.\\
$\Lambda,\Gamma_{\mathrm{res}}$ & Conditional-mean and diffused-residual parts of $\Gamma$.\\
$\sigma_{\ell,t}^2,\sigma_t^2,\sigma_{y,t}^2$ & Clean, noisy, and condition-only preactivation variances.\\
$\mu_1,v^2,v_\ell^2,s^2,v_y^2$ & Activation coefficients: $\mu_1$ is defined in Section~\ref{sec:rf-score-learning}, the others in Appendix~\ref{sec:activation-coefficients}.\\
$U,V;\Utilde,\Vtilde,\Vtilde_\xi,\Vtilde_\eta$ & Training and test feature moments in Lemmas~\ref{lem:traceid} and~\ref{lem:test-traceid}.\\
$\Utilde_y$ & Second moment of conditional feature means (Lemma~\ref{lem:test-traceid}).\\
$Q,\mathbf q,z$ & Ridge resolvent, $\mathbf q=(q_{\mathrm m},q_{\mathrm r})$ in \eqref{eq:Uq}, and spectral parameter.\\
$K, K_\Sigma;\Kbar,\Kbar_\Sigma$ & Unweighted and residual-weighted trace functions and their limits.\\
$m,\rho$ & Two scalars solving the fixed-point equations.\\
$\ME,\RE,\PV,\RA$ & Conditional-mean error, residual error, prediction variance, and residual alignment.\\
$G_x,\widehat B,\Anorm$ & Jacobian in $x_t$ of the linear component, conditional-mean map, and weight norm of the learned denoiser (Section~\ref{sec:test-decomposition}).\\
\bottomrule
\end{longtable}

\section{Proofs of the decomposition and the proportional limits}
\label{sec:proofs}
This appendix proves Lemma~\ref{lem:main-decomp} (Appendix~\ref{sec:app-decomp}) and
Theorem~\ref{thm:appendix-final}, which restates Theorem~\ref{thm:main}
at fixed aspect ratios $\psi_p,\psi_n,\psi_y$ and gives the functions
$\Kbar$ and $\Kbar_\Sigma$ explicitly.
We work with the data model~\eqref{eq:data-model}, so
$B=cU_\star$ and $\Sigma=(1-c^2)\Pi+\Pi_\perp$. In particular,

\begin{equation*}
BB\T=c^2\Pi,\qquad \Sigma^{1/2}=\sqrt{1-c^2}\Pi+\Pi_\perp,\qquad
\upsilon_\Sigma^2=1-c^2\psi_y.
\end{equation*}
The cases $c=0$ and $\psi_y=1$ are covered by combining equal blocks or omitting
zero-weight blocks. At $c=1$, $\Sigma$ is singular but $\Sigma_t\succeq h_tI$.
The proof consists of the following steps: exact loss identities, Gaussian reductions,
resolvent derivatives, linearization, and evaluation on the two subspaces.

\paragraph{Design matrices.}
Let $Y=[y_1\ \cdots\ y_n]$, $Z_x=[\xi_1\ \cdots\ \xi_n]$, and
$X=BY+\Sigma^{1/2}Z_x$. Set $\dop=\dx+\dy$ and $\kappa_d=\sqrt{\dx/\dy}$.
The matrix $\Tdet$ constructs the clean inputs, and $L$ contains their projections
onto the hidden weights:
\begin{equation}
\begin{aligned}
\Wcat&:=[W_1\ W_2],\qquad \Zcat:=[Y;Z_x],\\
\Tdet&:=\begin{bmatrix}a_tB&a_t\Sigma^{1/2}\\\kappa_dI_{\dy}&0\end{bmatrix},
\qquad \Tdet\Zcat=[a_tX;\kappa_dY],\\
L&:=\Wop\Tdet\Zcat
=\frac{a_tW_1X}{\sqrt{2\dx}}+\frac{W_2Y}{\sqrt{2\dy}}.
\end{aligned}
\label{eq:feature}
\end{equation}
Let $Z_0\in\R^{\dx\times n}$ have independent standard Gaussian entries,
independently of the training sample and hidden weights. Adding diffusion noise gives
\begin{equation}
\Wop[X_t(Z_0);\kappa_dY]=L+\sqrt{h_t}\Wa Z_0,
\qquad X_t(Z_0)=a_tX+\sqrt{h_t}Z_0.
\label{eq:noisyL}
\end{equation}

\paragraph{Covariance matrices.}
Write $I_x:=\operatorname{diag}(I_{\dx},0_{\dy})$ and
$\Sigma_x:=\operatorname{diag}(\Sigma,0_{\dy})$.
For a test draw, the conditional residual is
\begin{equation*}
\eta_t=a_t\Sigma^{1/2}\xi+\sqrt{h_t}\epsilon\sim\N(0,\Sigma_t),
\qquad\Sigma_t=a_t^2\Sigma+h_tI.
\end{equation*}
Its limiting normalized trace is
\begin{equation}
\tau:=\lim\frac1\dx\tr\Sigma_t
=a_t^2\upsilon_\Sigma^2+h_t.
\label{eq:tau}
\end{equation}
The covariances of the noisy input, clean input, and conditional mean are, respectively,
\begin{align}
\Gamma&:=\Cov([x_t;\kappa_dy])
=\begin{bmatrix}
a_t^2(BB\T+\Sigma)+h_tI&a_t\kappa_dB\\
a_t\kappa_dB\T&\kappa_d^2I
\end{bmatrix},
\label{eq:Gamma}\\
\Tdet\Tdet\T&=\Cov([a_tx;\kappa_dy])=\Gamma-h_tI_x,
\label{eq:TTt}\\
\Lambda&:=\Cov([a_tBy;\kappa_dy])=\Gamma-h_tI_x-a_t^2\Sigma_x.
\label{eq:Lambda-mean}
\end{align}

\paragraph{Preactivation variances.}
For hidden-weight rows $w_1\T,w_2\T$, the following variances hold the weights fixed;
convergence is in probability over these weights. They correspond to an entry of $L$,
its conditional mean given $y$, and the noisy preactivation, respectively:
\begin{equation*}
\begin{aligned}
\Var_{x,y}\!\left(\frac{w_1\T a_tx}{\sqrt{2\dx}}
 +\frac{w_2\T y}{\sqrt{2\dy}}\right)&\xrightarrow{\Prob}\sigma_{\ell,t}^2,\\
\Var_y\!\left(\frac{a_tw_1\T By}{\sqrt{2\dx}}
 +\frac{w_2\T y}{\sqrt{2\dy}}\right)&\xrightarrow{\Prob}\sigma_{y,t}^2,\\
\Var_{x_t,y}\!\left(\frac{w_1\T x_t}{\sqrt{2\dx}}
 +\frac{w_2\T y}{\sqrt{2\dy}}\right)&\xrightarrow{\Prob}\sigma_t^2.
\end{aligned}
\end{equation*}
Gaussian weight concentration gives the deterministic limits
\begin{equation}
\begin{aligned}
\sigma_{\ell,t}^2&=\lim\frac1{2\dx}\tr(\Tdet\Tdet\T)
=\frac{a_t^2(c^2\psi_y+\upsilon_\Sigma^2)+1}{2}
=\frac{a_t^2+1}{2},\\
\sigma_{y,t}^2&=\lim\frac1{2\dx}\tr\Lambda
=\frac{a_t^2c^2\psi_y+1}{2},\\
\sigma_t^2&=\lim\frac1{2\dx}\tr\Gamma
=\sigma_{\ell,t}^2+h_t/2=\sigma_{y,t}^2+\tau/2=1.
\end{aligned}
\label{eq:sigmaiso}
\end{equation}
The terms $a_t^2c^2\psi_y$, $a_t^2\upsilon_\Sigma^2$, and $1$ are the limits of
$\dx^{-1}\E\|a_tBy\|^2$, $\dx^{-1}\E\|a_t\Sigma^{1/2}\xi\|^2$, and $\dy^{-1}\E\|y\|^2$,
respectively, while $h_t$ is the diffusion-noise variance.

\paragraph{Assumptions.}
As in \citet{george2026,bonnaire2025,urfin2026double}, we apply the Gaussian equivalence
principle~\citep{gerace2020generalisation,goldt2022gaussian,mei2022generalization} to the
nonlinear random-feature matrices in the loss calculations.
For the resulting linear pencil, we assume that the inverse-block trace limits
exist and satisfy the block deterministic-equivalent equation, including the finite
block splittings used for $K$ and $K_\Sigma$.
Near $(\mathbf q,z)=(\mathbf0,-\lambda)$, differentiation in $\mathbf q$ and $z$ is
assumed to commute with the proportional limit. Expressions with negative components of
$\mathbf q$ are understood by local analytic continuation from
$q_{\mathrm m},q_{\mathrm r}\ge0$.
A rigorous treatment of these steps, along the lines of
\citet{mei2022generalization}, is beyond the scope of this work.
Throughout, we use the limiting normalized expected traces of \citet[eq.~(3.13)]{bodin2024},
\begin{equation}
\mathrm{Tr}_d[A]:=\lim\frac1\dx\E\tr A,\qquad
\mathrm{Tr}_n[A]:=\lim\frac1n\E\tr A,
\label{eq:Trn}
\end{equation}
with the limits taken in the proportional high-dimensional regime. For matrices,
$M\simeq M'$ means $\mathrm{Tr}_d[X(M-M')Y]=0$ for the matrices $X,Y$ with which $M$ appears
below; for scalars, $X\simeq Y$ means $\lim\E[X-Y]=0$. The equivalences of the test moments
(Lemmas~\ref{lem:Vtilde}--\ref{lem:Uy}) and of $V$ (Lemma~\ref{lem:Vt}) use only Mehler's
formula and the concentration of the Gaussian weights, and in fact hold in probability. We
state all equivalences in the expected sense because the GEP for $U$ and the deterministic
equivalent of the linear pencil are used in that sense.

\subsection{Exact finite-dimensional identities}
\label{sec:exact}

\begin{lemma}[Ridge minimizer and exact identities]
\label{lem:traceid}
For the training sample, define
\[
\begin{aligned}
x_{i,t}(\epsilon)&:=a_tx_i+\sqrt{h_t}\epsilon,\\
U&:=\frac1n\sum_{i=1}^n\E_\epsilon\big[\act(\Wop \bigl[x_{i,t}(\epsilon);\,\kappa_dy_i\bigr])
\act(\Wop \bigl[x_{i,t}(\epsilon);\,\kappa_dy_i\bigr])\T\big],\\
V&:=\frac1n\sum_{i=1}^n\E_\epsilon\big[\act(\Wop \bigl[x_{i,t}(\epsilon);\,\kappa_dy_i\bigr])\epsilon\T\big].
\end{aligned}
\]
The unique minimizer of \eqref{eq:loss} is
\begin{equation}
\frac{\Ahat}{\sqrt p}=-\frac{1}{\sqrt{h_t}}V\T(U+\lambda \Id{p})^{-1}.
\label{eq:minimizer}
\end{equation}
Write $Q:=(U+\lambda\Id{p})^{-1}$. The normalized training loss is exactly
\begin{equation}
\Ltrain(\Ahat)
=1-\frac1\dx\tr(V\T QV)-\frac\lambda\dx\tr(V\T Q^2V).
\label{eq:traintraceid}
\end{equation}

Since $\Ahat/\sqrt p=-h_t^{-1/2}V\T Q$, the quantities of
Section~\ref{sec:test-decomposition} are
\begin{equation}
\begin{gathered}
\Anorm=\frac1{h_t\dx}\tr\big(V\T Q^2V\big),\qquad
G_x=\frac1{a_t}\Big(I-\sqrt{h_t}\,\mu_1V\T Q\Wa\Big),\\
\frac{\|\widehat B-B\|_F^2}{\dx}=\frac{h_t\mu_1^2}{a_t^2\dx}\tr\big(V\T Q\Wop\Lambda\WopT QV\big).
\end{gathered}
\label{eq:exact-quantities}
\end{equation}
\end{lemma}
\begin{proof}
Here $U$ and $V$ are the noise-averaged second moments of the training features. Expanding the original loss \eqref{eq:loss} gives
\begin{align}
\Lridge(A)
&=\frac{1}{\dx n}\sum_{i=1}^{n}\E_\epsilon
\left\|\sqrt{h_t}\frac{A}{\sqrt p}\act(\Wop \bigl[x_{i,t}(\epsilon);\,\kappa_dy_i\bigr])+\epsilon\right\|^2
+\frac{h_t\lambda}{\dx p}\|A\|_F^2\notag\\
&=\frac{h_t}{\dx p}\tr\!\big[A(U+\lambda\Id{p})A\T\big]
+\frac{2\sqrt{h_t}}{\dx\sqrt p}\tr(AV)+1.
\label{eq:lossquadratic}
\end{align}
Since $U\succeq0$ and $\lambda>0$, one has
$U+\lambda\Id{p}\succ0$, so this quadratic is strictly convex. Differentiating
\eqref{eq:lossquadratic} and setting the gradient to zero gives \eqref{eq:minimizer}.

For the training loss, remove the ridge term from \eqref{eq:lossquadratic} and substitute
\eqref{eq:minimizer} to obtain
\begin{align*}
\Ltrain(\Ahat)
&=1-\frac{2}{\dx}\tr(V\T QV)
+\frac1\dx\tr(V\T QUQV)\\
&=1-\frac1\dx\tr(V\T QV)-\frac\lambda\dx\tr(V\T Q^2V),
\end{align*}
where the second equality uses $QUQ=Q-\lambda Q^2$. This proves
\eqref{eq:traintraceid}.

For~\eqref{eq:exact-quantities}, substitute~\eqref{eq:minimizer} into
$\Anorm=\|\Ahat\|_F^2/(\dx p)$, into $G_x=a_t^{-1}(I+h_t\mu_1\frac{\Ahat}{\sqrt p}\Wa)$, and into
$\widehat B-B=(h_t\mu_1/a_t)\frac{\Ahat}{\sqrt p}\Wop T_y$ with $T_y:=[a_tB;\kappa_dI_{\dy}]$, which satisfies
$T_yT_y\T=\Lambda$ by~\eqref{eq:Lambda-mean}.
\end{proof}

\begin{lemma}[Exact test identities]
\label{lem:test-traceid}
Hold the training sample and the hidden weights fixed.
With $\varphi=\act(\Wop[x_t;\kappa_dy])$, put
\begin{equation}
\Utilde=\E[\varphi\varphi\T],\qquad
\Vtilde=\E[\varphi\epsilon\T],\qquad
\Vtilde_\xi=\E[\varphi\xi\T],\qquad
\Vtilde_\eta=\E[\varphi\eta_t\T]=a_t\Vtilde_\xi\Sigma^{1/2}+\sqrt{h_t}\Vtilde.
\label{eq:testmoments}
\end{equation}
Also define
\begin{equation}
\Utilde_y=\E[\E[\varphi\mid y]\E[\varphi\mid y]\T].
\label{eq:testmoments-y}
\end{equation}
By the conditional covariance identity,
\begin{equation*}
\Utilde-\Utilde_y=\E\big[\operatorname{Cov}(\varphi\mid y)\big].
\end{equation*}
The exact denoiser identities are
\begin{equation}
\begin{aligned}
\ME(\Ahat)&=\frac{h_t}{a_t^2\dx}\tr(V\T Q\Utilde_y QV),\\
\RA(\Ahat)&=\upsilon_{\Sigma,d}^2-
 \frac{\sqrt{h_t}}{a_t\dx}\tr(V\T Q\Vtilde_\xi\Sigma^{1/2}),\\
\PV(\Ahat)&=\frac1{a_t^2}\left[\frac1\dx\tr\Sigma_t
 -\frac{2\sqrt{h_t}}\dx\tr(V\T Q\Vtilde_\eta)
 +\frac{h_t}\dx\tr(V\T Q(\Utilde-\Utilde_y)QV)\right],\\
\RE(\Ahat)&=\upsilon_{\Sigma,d}^2+\PV(\Ahat)-2\RA(\Ahat).
\end{aligned}
\label{eq:me-traceid}
\end{equation}
\end{lemma}
\begin{proof}
By~\eqref{eq:minimizer}, $s_{\Ahat}=-h_t^{-1/2}V\T Q\varphi$.

Next, the denoiser definition and $\E[\eta_t\mid y]=0$ imply
\begin{equation*}
x_A-By=\frac{\eta_t+h_ts_A}{a_t},\qquad
\E[x_A-By\mid y]=\frac{h_t}{a_t}\E[s_A\mid y].
\end{equation*}
Consequently,
\begin{equation*}
\ME(A)=\frac{h_t^2}{a_t^2\dx}\E\big\|\E[s_A\mid y]\big\|^2.
\end{equation*}
Since $V\T Q$ is fixed with respect to the test draw, \eqref{eq:testmoments-y} gives
\begin{align*}
\frac1\dx\E\big\|\E[s_{\Ahat}\mid y]\big\|^2
&=\frac1{\dx h_t}\E\big\|V\T Q\E[\varphi\mid y]\big\|^2\\
&=\frac1{\dx h_t}\tr\!\big(V\T Q\Utilde_y QV\big).
\end{align*}
This proves the ME identity in \eqref{eq:me-traceid}.

For PV, center the denoiser at its conditional mean:
\begin{equation*}
x_A-\E[x_A\mid y]
=\frac{\eta_t+h_t(s_A-\E[s_A\mid y])}{a_t}.
\end{equation*}
Using $\E\langle\eta_t,\E[s_A\mid y]\rangle=0$ and
$\E\|\eta_t\|^2=\tr\Sigma_t$, expansion of the squared norm yields
\begin{align*}
\PV(A)=\frac1{a_t^2}\Bigg[&\frac1\dx\tr\Sigma_t
+\frac{2h_t}{\dx}\E[s_A\T\eta_t]\\
&+\frac{h_t^2}{\dx}\left(
\E\|s_A\|^2-\E\big\|\E[s_A\mid y]\big\|^2\right)\Bigg].
\end{align*}
For the fitted score, \eqref{eq:testmoments} gives
\begin{equation*}
\E[s_{\Ahat}\T\eta_t]
=-\frac1{\sqrt{h_t}}\tr\!\big(V\T Q\Vtilde_\eta\big),\qquad
\frac1\dx\E\|s_{\Ahat}\|^2
=\frac1{\dx h_t}\tr\!\big(V\T Q\Utilde QV\big).
\end{equation*}
Substituting these and the identity for $\E\|\E[s_{\Ahat}\mid y]\|^2$ above gives the
PV formula in \eqref{eq:me-traceid}.

Taking the inner product of $x_A-By$ with $\Sigma^{1/2}\xi$ gives the RA
identity. Finally the conditional-mean and centered terms are orthogonal in
expectation, so
$\dx^{-1}\E\|x_A-x\|^2=\ME+\PV+\upsilon_{\Sigma,d}^2-2\RA$.
This proves the RE formula and the decomposition in the main text.
\end{proof}

\subsection{Proportional-limit reductions of the feature moments}
\label{sec:hermite}
\label{sec:activation-coefficients}

Throughout, $z,z'\sim\N(0,1)$ are independent, and $H_k$ are the probabilists' Hermite polynomials, which satisfy
\[
\E[H_k(z)H_\ell(z)]=k!\delta_{k\ell}.
\]
For $\sigma>0$,
\begin{equation*}
\widehat{\mu}_k(\sigma):=\frac1{k!}\E[\act(\sigma z)H_k(z)],\qquad k\ge0,
\end{equation*}
so that $\act(\sigma z)=\sum_{k\ge0}\widehat{\mu}_k(\sigma)H_k(z)$ in $L^2$.
Averaging over the diffusion noise and the full conditional residual, respectively, define
\begin{equation*}
\begin{aligned}
\act_0(u)&:=\E_{z'}[\act(u+\sqrt{h_t/2}z')],\\
\act_{\tau}(u)&:=\E_{z'}[\act(u+\sqrt{\tau/2}z')].
\end{aligned}
\end{equation*}
Let $r(u):=\act(u)-\mu_1u$, and let $z_\ell\sim\N(0,\sigma_{\ell,t}^2)$,
$z_\epsilon\sim\N(0,h_t/2)$, $z_y\sim\N(0,\sigma_{y,t}^2)$, and $z_\eta\sim\N(0,\tau/2)$ be
independent. By~\eqref{eq:sigmaiso}, $z_\ell+z_\epsilon$ and $z_y+z_\eta$ are standard
Gaussian, split into the parts from the clean input and the diffusion noise, and from $y$
and $\eta_t$, respectively. The deterministic coefficients used in Lemma~\ref{lem:main-decomp} and
Theorem~\ref{thm:main} are
\begin{align}
v^2&:=\E_z\big[r(z)^2\big]=\E_z[\act(z)^2]-\mu_1^2,
\label{eq:v-coefficient}\\
v_\ell^2&:=\E_{z_\ell}\big[\E_{z_\epsilon}[r(z_\ell+z_\epsilon)]^2\big]
=\E_z[\act_0(\sigma_{\ell,t}z)^2]-\sigma_{\ell,t}^2\mu_1^2,
\label{eq:v0t}\\
v_y^2&:=\E_{z_y}\big[\E_{z_\eta}[r(z_y+z_\eta)]^2\big]
=\E_z[\act_{\tau}(\sigma_{y,t}z)^2]-\sigma_{y,t}^2\mu_1^2,
\label{eq:vy-coefficient}\\
s^2&:=\E_{z_\ell}\big[\Var_{z_\epsilon}[r(z_\ell+z_\epsilon)]\big]
=v^2-v_\ell^2\notag\\
&\phantom{:}=\E_z[\act(z)^2]-\E_z[\act_0(\sigma_{\ell,t}z)^2]
-\frac{h_t\mu_1^2}{2}.
\label{eq:sfloor}
\end{align}
The first equality in \eqref{eq:sfloor} uses
$\Var_{z_\epsilon}[r]=\E_{z_\epsilon}[r^2]-\E_{z_\epsilon}[r]^2$, so that $v_\ell^2$ and $s^2$ split
$v^2$ into the variation of the nonlinear part across clean inputs and across diffusion noise.
The second equalities in \eqref{eq:v0t} and~\eqref{eq:vy-coefficient} use
$\E_{z_\epsilon}[\act(z_\ell+z_\epsilon)]=\act_0(z_\ell)$ and, by Stein's identity,
$\E[z_\ell\act_0(z_\ell)]=\sigma_{\ell,t}^2\mu_1$, and likewise for $(z_y,z_\eta)$. The last
equality in \eqref{eq:sfloor} uses $\sigma_t^2-\sigma_{\ell,t}^2=h_t/2$. The coefficients $v^2$ and
$v_y^2$ are those of Section~\ref{sec:test-decomposition}.

We now compute the test moments of Lemma~\ref{lem:test-traceid}, which give
Lemma~\ref{lem:main-decomp} in Appendix~\ref{sec:app-decomp}, and the training moments $V$ and
$U$ of Lemma~\ref{lem:traceid}, which enter Theorem~\ref{thm:appendix-final}. We proceed entrywise: write the relevant pair as a
bivariate standard Gaussian, apply Mehler's formula, and retain the leading terms in the
proportional limit, under the assumptions stated at the start of this appendix. The reduction
of the nonlinear functions of $L$ in the training matrix $U$ additionally uses the GEP, as in
\citet{george2026}.

For $\gamma\in(-1,1)$, let
$\phi_\gamma$ denote the bivariate standard Gaussian density with correlation $\gamma$,
\begin{equation*}
\phi_\gamma(z_0,z_1):=\frac{1}{2\pi\sqrt{1-\gamma^2}}
\exp\!\left(-\frac{z_0^2+z_1^2-2\gamma z_0z_1}{2(1-\gamma^2)}\right).
\end{equation*}
We write $\E_{\phi_\gamma}$ for expectation over a pair $(z_0,z_1)$ with this density.
For $(z_0,z_1)\sim\phi_\gamma$, the Hermite expansion gives Mehler's formula
\begin{equation*}
\E_{\phi_\gamma}[f(z_0)g(z_1)]
=\sum_{k\ge0}\frac{\gamma^k}{k!}\,
\E[f(z_0)H_k(z_0)]\,\E[g(z_1)H_k(z_1)] .
\end{equation*}
Whenever one factor is linear, only the $k=1$ summand remains.

Recall $\kappa_d$, $\Wcat$, and $\Gamma$ from
\eqref{eq:feature} and \eqref{eq:Gamma}. Every entry of $\Wcat$ is standard Gaussian, and
\begin{equation}
\Cov(\Wop \bigl[x_t;\,\kappa_dy\bigr]\mid W_1,W_2)=\Wop\Gamma\WopT.
\label{eq:qcov}
\end{equation}
By \eqref{eq:sigmaiso} and concentration of Gaussian quadratic forms, each marginal
preactivation variance converges in probability to $\sigma_t^2$. Thus a marginal
preactivation may be replaced by $\sigma_t z$ at the order retained below.

\paragraph{Test moments.}
\begin{lemma}[High-dimensional equivalents of $\Vtilde, \Vtilde_\xi, \Vtilde_\eta$]
\label{lem:Vtilde}
Under the assumptions of this appendix,
\begin{equation}
\Vtilde\simeq\sqrt{h_t}\mu_1\Wa,\qquad
\Vtilde_\xi\simeq a_t\mu_1\Wa\Sigma^{1/2},\qquad
\Vtilde_\eta\simeq\mu_1\Wa\Sigma_t.
\label{eq:Vtilde}
\end{equation}
\end{lemma}
\begin{proof}
Condition on the hidden weights and write
\begin{equation*}
g_r:=(\Wop[x_t;\kappa_dy])_r,\qquad
\sigma_r^2:=(\Wop\Gamma\WopT)_{rr}.
\end{equation*}
Since $x_t=a_tBy+a_t\Sigma^{1/2}\xi+\sqrt{h_t}\epsilon$ and
$y$, $\xi$, and $\epsilon$ are independent standard Gaussian vectors,
\begin{align*}
\E[g_r\epsilon_j]
&=\sqrt{h_t}\frac{(W_1)_{rj}}{\sqrt{2\dx}},
&\gamma_{rj}^{\epsilon}
&:=\frac{\E[g_r\epsilon_j]}{\sigma_r}
=\frac{\sqrt{h_t}(W_1)_{rj}}{\sigma_r\sqrt{2\dx}},\\
\E[g_r\xi_j]
&=a_t\frac{(W_1\Sigma^{1/2})_{rj}}{\sqrt{2\dx}},
&\gamma_{rj}^{\xi}
&:=\frac{\E[g_r\xi_j]}{\sigma_r}
=\frac{a_t(W_1\Sigma^{1/2})_{rj}}{\sigma_r\sqrt{2\dx}}.
\end{align*}
Thus $(g_r/\sigma_r,\epsilon_j)$ and $(g_r/\sigma_r,\xi_j)$ are
standard Gaussian pairs with correlations $\gamma_{rj}^{\epsilon}$ and
$\gamma_{rj}^{\xi}$. Applying Mehler's formula to the first pair gives
\begin{align*}
[\Vtilde]_{rj}
&=\E[\act(g_r)\epsilon_j]
=\E_{\phi_{\gamma_{rj}^{\epsilon}}}
 [\act(\sigma_rz_0)z_1]\\
&=\sum_{k\ge0}\frac{(\gamma_{rj}^{\epsilon})^k}{k!}
 \E[\act(\sigma_rz_0)H_k(z_0)]\E[z_1H_k(z_1)]\\
&=\gamma_{rj}^{\epsilon}\widehat\mu_1(\sigma_r)
=\sqrt{h_t}\frac{\widehat\mu_1(\sigma_r)}{\sigma_r}
 \frac{(W_1)_{rj}}{\sqrt{2\dx}}.
\end{align*}
Only $k=1$ remains because $z_1=H_1(z_1)$ and Hermite orthogonality gives
$\E[z_1H_k(z_1)]=\mathbf1\{k=1\}$.
For the second pair, the same expansion yields
\begin{align*}
[\Vtilde_\xi]_{rj}
&=\E[\act(g_r)\xi_j]
=\E_{\phi_{\gamma_{rj}^{\xi}}}[\act(\sigma_rz_0)z_1]\\
&=\sum_{k\ge0}\frac{(\gamma_{rj}^{\xi})^k}{k!}
 \E[\act(\sigma_rz_0)H_k(z_0)]\E[z_1H_k(z_1)]\\
&=\gamma_{rj}^{\xi}\widehat\mu_1(\sigma_r)
=a_t\frac{\widehat\mu_1(\sigma_r)}{\sigma_r}
 \frac{(W_1\Sigma^{1/2})_{rj}}{\sqrt{2\dx}}.
\end{align*}
By the variance concentration following \eqref{eq:qcov},
$\sigma_r^2\to\sigma_t^2$ in probability. Since $\sigma\mapsto\widehat\mu_1(\sigma)/\sigma$ is
continuous, replacing
$\widehat\mu_1(\sigma_r)/\sigma_r$ by
$\widehat\mu_1(\sigma_t)/\sigma_t=\mu_1$ gives
\begin{equation*}
\Vtilde\simeq\sqrt{h_t}\mu_1\Wa,\qquad
\Vtilde_\xi\simeq a_t\mu_1\Wa\Sigma^{1/2}
\end{equation*}
in the normalized traces under consideration.
Finally, \eqref{eq:testmoments} gives
\begin{align*}
\Vtilde_\eta
&=a_t\Vtilde_\xi\Sigma^{1/2}+\sqrt{h_t}\Vtilde\\
&\simeq a_t^2\mu_1\Wa\Sigma+h_t\mu_1\Wa
=\mu_1\Wa\Sigma_t.
\end{align*}
\end{proof}

\begin{lemma}[High-dimensional equivalent of $\Utilde$]
\label{lem:Utilde}
In the proportional limit, $\Utilde$ can be replaced by
\begin{equation}
\mu_1^2\,\Wop\Gamma\WopT+v^2\,\Id{p}.
\label{eq:Utilde}
\end{equation}
\end{lemma}

\begin{proof}
For $r\ne s$, set
\begin{equation*}
\widetilde\gamma_{rs}
:=\frac{(\Wop\Gamma\WopT)_{rs}}{\sigma_t^2}.
\end{equation*}
Mehler's formula gives
\begin{align*}
[\Utilde]_{rs}
&=\E\!\left[
\act\!\left((\Wop[x_t;\,\kappa_dy])_r\right)
\act\!\left((\Wop[x_t;\,\kappa_dy])_s\right)
\right]\\
&=\E_{\phi_{\widetilde\gamma_{rs}}}
\big[\act(\sigma_tz_0)\act(\sigma_tz_1)\big]\\
&=\sum_{k\ge0}\frac{\widetilde\gamma_{rs}^k}{k!}\,
\E[\act(\sigma_tz_0)H_k(z_0)]\,
\E[\act(\sigma_tz_1)H_k(z_1)]\\
&=\sum_{k\ge0}k!\,\widehat{\mu}_k(\sigma_t)^2\widetilde\gamma_{rs}^k\\
&=\mu_1^2
(\Wop\Gamma\WopT)_{rs}+O(\dx^{-1})\\
&=\big[\mu_1^2\Wop\Gamma\WopT\big]_{rs}+O(\dx^{-1}),
\qquad r\ne s.
\end{align*}
For $r=s$,
\begin{align*}
[\Utilde]_{rr}
&=\E\!\left[\act\!\left((\Wop[x_t;\,\kappa_dy])_r\right)^2\right]\\
&=\big\|\act(\sigma_t\,\cdot)\big\|_2^2,\\
\big[\mu_1^2\Wop\Gamma\WopT\big]_{rr}
&=\mu_1^2\sigma_t^2.
\end{align*}
Therefore
\begin{align*}
\Utilde
&=\mu_1^2\Wop\Gamma\WopT
+\left(\big\|\act(\sigma_t\,\cdot)\big\|_2^2-\sigma_t^2\mu_1^2\right)\Id{p}\\
&=\mu_1^2\Wop\Gamma\WopT+v^2\Id{p}.
\end{align*}
\end{proof}

\begin{lemma}[High-dimensional equivalent of $\Utilde_y$]
\label{lem:Uy}
Recall $\Lambda$ from \eqref{eq:Lambda-mean}.
In the proportional limit, $\Utilde_y$ can be replaced by
\begin{equation}
\mu_1^2\,\Wop\Lambda\WopT+v_y^2\,\Id{p}.
\label{eq:Uy}
\end{equation}
\end{lemma}

\begin{proof}
Fix $W_1,W_2$. By \eqref{eq:testmoments-y},
\begin{equation*}
\begin{aligned}
\varphi_r&=\act((\Wop[x_t;\kappa_dy])_r),\\
[\Utilde_y]_{rs}
&=\E_y\!\left[\E[\varphi_r\mid y,W_1,W_2]\,
\E[\varphi_s\mid y,W_1,W_2]\right].
\end{aligned}
\end{equation*}
To compute these conditional means, let $w_{1r}\T,w_{2r}\T$ denote the $r$th rows
of $W_1,W_2$. Since $\eta_t=x_t-a_tBy\sim\N(0,\Sigma_t)$ is independent of $y$,
\begin{align*}
(\Wop[x_t;\kappa_dy])_r
&=\underbrace{\frac{a_tw_{1r}\T By}{\sqrt{2\dx}}
 +\frac{w_{2r}\T y}{\sqrt{2\dy}}}_{=:g_r(y)}
 +\frac{w_{1r}\T\eta_t}{\sqrt{2\dx}},\\
\frac{w_{1r}\T\eta_t}{\sqrt{2\dx}}\,\bigg|\,y,W_1,W_2
&\sim\N(0,\nu_r^2),
\qquad \nu_r^2:=\frac{w_{1r}\T\Sigma_tw_{1r}}{2\dx}.
\end{align*}
Over the Gaussian weights,
\begin{align*}
\E_{w_{1r}}[\nu_r^2]
&=\frac{\tr\Sigma_t}{2\dx}\longrightarrow\frac\tau2,\\
\Var_{w_{1r}}(\nu_r^2)
&=\frac{\tr(\Sigma_t^2)}{2\dx^2}=O(\dx^{-1}),
\qquad \nu_r^2\xrightarrow{\Prob}\frac\tau2.
\end{align*}
Using $\nu_r^2\xrightarrow{\Prob}\tau/2$,
\begin{equation}
\begin{aligned}
\E[\act((\Wop[x_t;\kappa_dy])_r)\mid y,W_1,W_2]
&=\E_{z'}[\act(g_r(y)+\nu_rz')]\\
&\simeq\E_{z'}[\act(g_r(y)+\sqrt{\tau/2}z')]\\
&=\act_\tau(g_r(y)).
\end{aligned}
\label{eq:phiy}
\end{equation}
With $g(y):=(g_r(y))_{r=1}^p$, \eqref{eq:Lambda-mean} and \eqref{eq:sigmaiso} give
\begin{align*}
g(y)\mid W_1,W_2&\sim\N(0,\Wop\Lambda\WopT),\\
(\Wop\Lambda\WopT)_{rr}&\xrightarrow{\Prob}\sigma_{y,t}^2,\\
\widetilde\gamma^{\,y}_{rs}
&:=\frac{(\Wop\Lambda\WopT)_{rs}}{\sigma_{y,t}^2}
=O_{\Prob}(\dx^{-1/2}),\qquad r\ne s.
\end{align*}
Define
\begin{equation*}
\widehat\mu_k^{\,\tau}(\sigma_{y,t})
:=\frac1{k!}\E_z[\act_\tau(\sigma_{y,t}z)H_k(z)].
\end{equation*}
Since $\sigma_{y,t}^2+\tau/2=\sigma_t^2=1$, centering and Mehler's formula yield
\begin{align*}
\widehat\mu_0^{\,\tau}(\sigma_{y,t})
&=\E[\act(\sigma_{y,t}z+\sqrt{\tau/2}z')]
=\E[\act(\sigma_tz)]=0,\\
\widehat\mu_1^{\,\tau}(\sigma_{y,t})
&=\E[z\,\act(\sigma_{y,t}z+\sqrt{\tau/2}z')]\\
&=\E_{\phi_{\sigma_{y,t}/\sigma_t}}[z_0\,\act(\sigma_tz_1)]
=\frac{\sigma_{y,t}}{\sigma_t}\widehat\mu_1(\sigma_t)
=\sigma_{y,t}\mu_1.
\end{align*}
At the retained order, Mehler's expansion for $r\ne s$ is
\begin{align*}
[\Utilde_y]_{rs}
&\simeq\E_{\phi_{\widetilde\gamma^{\,y}_{rs}}}
[\act_\tau(\sigma_{y,t}z_0)\act_\tau(\sigma_{y,t}z_1)]\\
&=\sum_{k\ge1}k!\,\widehat\mu_k^{\,\tau}(\sigma_{y,t})^2
 (\widetilde\gamma^{\,y}_{rs})^k\\
&=\mu_1^2(\Wop\Lambda\WopT)_{rs}+O_{\Prob}(\dx^{-1}),\\
[\Utilde_y]_{rr}
&\xrightarrow{\Prob}\E_z[\act_\tau(\sigma_{y,t}z)^2]
=\sigma_{y,t}^2\mu_1^2+v_y^2.
\end{align*}
Consequently, with $\simeq$ in the trace sense specified in this appendix,
\begin{align*}
\Utilde_y
&\simeq\mu_1^2\Wop\Lambda\WopT
+\Big(\E_z[\act_\tau(\sigma_{y,t}z)^2]-\sigma_{y,t}^2\mu_1^2\Big)\Id{p}\\
&=\mu_1^2\Wop\Lambda\WopT+v_y^2\Id{p}.
\end{align*}
\end{proof}

\paragraph{Training moments.}
Recall the clean preactivation matrix $L$ from \eqref{eq:feature}.

\begin{lemma}[High-dimensional equivalent of $V$]
\label{lem:Vt}
In the proportional limit, $V$ can be replaced by
\begin{equation}
\sqrt{h_t}\,\mu_1\,\Wa.
\label{eq:Vtfinal}
\end{equation}
\end{lemma}

\begin{proof}
For the $\ell$th training sample, by the definition of $V$,
\[
V=\frac1n\sum_{\ell=1}^n V^{(\ell)},\qquad
V^{(\ell)}:=\E_\epsilon\big[\act(\Wop \bigl[x_{\ell,t}(\epsilon);\,\kappa_dy_\ell\bigr])\epsilon\T\big].
\]
By \eqref{eq:noisyL}, conditionally on $(x_\ell,y_\ell,W_1,W_2)$, the shift $L_{r\ell}$ is fixed and
$\|w_{1r}\|^2/\dx=1+O_{\Prob}(\dx^{-1/2})$. Put
\begin{equation*}
\gamma_{rj}:=\frac{(W_1)_{rj}}{\sqrt{\dx}}.
\end{equation*}
Then Mehler's formula gives the consecutive equalities
\begin{align*}
[V^{(\ell)}]_{rj}
&=\big[\E_\epsilon[\act(\Wop \bigl[x_{\ell,t}(\epsilon);\,\kappa_dy_\ell\bigr])\epsilon\T]\big]_{rj}\\
&=\E_\epsilon\!\left[\act\!\left(L_{r\ell}
+\sqrt{h_t}\frac{w_{1r}\T \epsilon}{\sqrt{2\dx}}\right)\epsilon_j\right]\\
&=\E_{\phi_{\gamma_{rj}}}
\big[\act(L_{r\ell}+\sqrt{h_t/2}z_0)z_1\big]\\
&=\sum_{k\ge0}\frac{\gamma_{rj}^k}{k!}\,
\E_{z_0}\big[\act(L_{r\ell}+\sqrt{h_t/2}z_0)H_k(z_0)\big]\,
\E_{z_1}[z_1H_k(z_1)]\\
&=\gamma_{rj}\,\act_1(L_{r\ell}),
\end{align*}
where $\act_1(u):=\E_{z'}[\act(u+\sqrt{h_t/2}z')z']$; the final equality holds because only
$k=1$ survives.
After averaging over the training samples,
\begin{align*}
[V]_{rj}
&=\frac1n\sum_{i=1}^n[V^{(i)}]_{rj}\\
&=\gamma_{rj}\frac1n\sum_{i=1}^n\act_1(L_{ri})\\
&=\gamma_{rj}\E_z\big[\act_1(\sigma_{\ell,t}z)\big]+O(\dx^{-1})\\
&=\gamma_{rj}\E_{z,z'}\big[
\act(\sigma_{\ell,t}z+\sqrt{h_t/2}z')z'\big]+O(\dx^{-1})\\
&=\gamma_{rj}\E_{\phi_{\sqrt{h_t/2}/\sigma_t}}
\big[\act(\sigma_tz_0)z_1\big]+O(\dx^{-1})\\
&=\gamma_{rj}\sqrt{h_t/2}\,\mu_1+O(\dx^{-1})\\
&=\sqrt{h_t}\,\mu_1
\frac{(W_1)_{rj}}{\sqrt{2\dx}}+O(\dx^{-1})\\
&=\big[\sqrt{h_t}\,\mu_1\Wa\big]_{rj}+O(\dx^{-1}).
\end{align*}
\end{proof}

\begin{lemma}[High-dimensional equivalent of $U$]
\label{lem:Ut}
For an independent standard Gaussian $\Omega\in\R^{p\times n}$, set
\begin{equation}
\Gt:=\mu_1 L+v_\ell\Omega.
\label{eq:Gtdef}
\end{equation}
Then $s^2\ge0$. In the proportional limit, $U$ can be replaced by
\begin{equation}
\frac{\Gt\Gt\T}{n}
+h_t\mu_1^2\,\Wa\WaT+s^2\Id{p}.
\label{eq:Ugauss}
\end{equation}
\end{lemma}

\begin{proof}
For the $\ell$th training sample, set
\[
U^{(\ell)}:=\E_\epsilon\big[\act(\Wop \bigl[x_{\ell,t}(\epsilon);\,\kappa_dy_\ell\bigr])
\act(\Wop \bigl[x_{\ell,t}(\epsilon);\,\kappa_dy_\ell\bigr])\T\big].
\]
Recall $\act_1$ from the proof of Lemma~\ref{lem:Vt}.
Then $U=n^{-1}\sum_{\ell=1}^nU^{(\ell)}$. For $r\ne s$, the row-norm concentration and
the fact that the DSM noise enters only through the data block give the asymptotic correlation
\begin{equation*}
\gamma_{rs}:=\frac{w_{1r}\T w_{1s}}{\dx}=2[\Wa\WaT]_{rs}.
\end{equation*}
Using \eqref{eq:noisyL}, Mehler's formula therefore gives directly
\begin{align*}
[U^{(\ell)}]_{rs}
&=\E_\epsilon\!\left[
\act\!\left(L_{r\ell}+\sqrt{h_t}\frac{w_{1r}\T \epsilon}{\sqrt{2\dx}}\right)
\act\!\left(L_{s\ell}+\sqrt{h_t}\frac{w_{1s}\T \epsilon}{\sqrt{2\dx}}\right)
\right]\\
&=\E_{\phi_{\gamma_{rs}}}
\big[\act(L_{r\ell}+\sqrt{h_t/2}z_0)
      \act(L_{s\ell}+\sqrt{h_t/2}z_1)\big]\\
&=\sum_{k\ge0}\frac{\gamma_{rs}^k}{k!}\,
\E_{z_0}\big[\act(L_{r\ell}+\sqrt{h_t/2}z_0)H_k(z_0)\big]\,
\E_{z_1}\big[\act(L_{s\ell}+\sqrt{h_t/2}z_1)H_k(z_1)\big]\\
&=\act_0(L_{r\ell})\act_0(L_{s\ell})
+\gamma_{rs}\act_1(L_{r\ell})\act_1(L_{s\ell})
+O(\dx^{-1}),
\end{align*}
where the $k\ge2$ part is $O(\gamma_{rs}^2)=O(\dx^{-1})$. Summing over the training samples,
\begin{align*}
[U]_{rs}
&=\frac1n\sum_{i=1}^n[U^{(i)}]_{rs}\\
&=\frac1n\sum_{i=1}^n\act_0(L_{ri})\act_0(L_{si})
+\gamma_{rs}\frac1n\sum_{i=1}^n\act_1(L_{ri})\act_1(L_{si})
+O_{\Prob}(\dx^{-1})\\
&=\frac1n\sum_{i=1}^n\act_0(L_{ri})\act_0(L_{si})
+\gamma_{rs}\E_{x,y}\!\left[
\act_1\!\left(\frac{a_tw_{1r}\T x}{\sqrt{2\dx}}+\frac{w_{2r}\T y}{\sqrt{2\dy}}\right)
\right.\\
&\hspace{7em}\left.
\times
\act_1\!\left(\frac{a_tw_{1s}\T x}{\sqrt{2\dx}}+\frac{w_{2s}\T y}{\sqrt{2\dy}}\right)
\right]
+O_{\Prob}(\dx^{-1})\\
&=\frac1n\sum_{i=1}^n\act_0(L_{ri})\act_0(L_{si})
+\gamma_{rs}\E_z\big[\act_1(\sigma_{\ell,t}z)\big]^2
+O_{\Prob}(\dx^{-1})\\
&=\left[\frac{\act_0(L)\act_0(L)\T}{n}
+h_t\mu_1^2\Wa\WaT\right]_{rs}+O_{\Prob}(\dx^{-1}),
\qquad r\ne s,
\end{align*}
where the penultimate equality uses the Gaussian limit of the two clean preactivations.

For $r=s$,
\begin{align*}
[U]_{rr}
&=\frac1n\sum_{i=1}^n[U^{(i)}]_{rr}\\
&=\frac1n\sum_{i=1}^n\E_\epsilon\!\left[
\act\!\left(L_{ri}+\sqrt{h_t}\frac{w_{1r}\T \epsilon}{\sqrt{2\dx}}\right)^2
\right]\\
&=\E_{x,y,\epsilon}\!\left[
\act\!\left(
\frac{a_tw_{1r}\T x}{\sqrt{2\dx}}+\frac{w_{2r}\T y}{\sqrt{2\dy}}
+\sqrt{h_t}\frac{w_{1r}\T \epsilon}{\sqrt{2\dx}}
\right)^2\right]+O_{\Prob}(\dx^{-1/2})\\
&=\E_z\!\left[\act\!\left(\sqrt{\sigma_{\ell,t}^2+h_t/2}\,z\right)^2\right]
+O_{\Prob}(\dx^{-1/2}) \\
&=\big\|\act(\sigma_t\,\cdot)\big\|_2^2+O_{\Prob}(\dx^{-1/2}).
\end{align*}
Combining the off-diagonal and diagonal terms gives
\begin{equation}
U
\simeq\frac{\act_0(L)\act_0(L)\T}{n}
+h_t\mu_1^2\Wa\WaT
+s^2\Id{p}.
\label{eq:Utfinal}
\end{equation}
Moreover, by~\eqref{eq:sfloor},
\begin{align*}
s^2
&=\big\|\act(\sigma_t\,\cdot)\big\|_2^2
-\E_z\!\left[\act_0(\sigma_{\ell,t}z)^2\right]
-\frac{h_t\mu_1^2}{2} \\
&=\E_{z,z'}\!\left[
\left(
\act(\sigma_{\ell,t}z+\sqrt{h_t/2}z')
-\act_0(\sigma_{\ell,t}z)
-\sqrt{h_t/2}\,\mu_1 z'
\right)^2
\right]\ge0.
\end{align*}

Following \citet{george2026}, we apply the Gaussian equivalence principle to the
nonlinearity $\act_0(L)$ with $L=\Wop\Tdet\Zcat$. Using \eqref{eq:v0t},
$\act_0(L)$ can be replaced by its Gaussian equivalent
$\Gt=\mu_1 L+v_\ell\Omega$. Substituting this replacement into \eqref{eq:Utfinal} gives
\eqref{eq:Ugauss}.
\end{proof}

\subsection{Decomposition of the test errors and the training loss}
\label{sec:app-decomp}
We now prove Lemma~\ref{lem:main-decomp}. For $\ME$, $\PV$, and $\RA$,
the proof combines the exact test identities~\eqref{eq:me-traceid} with the test moments of
Appendix~\ref{sec:hermite} and the exact expressions~\eqref{eq:exact-quantities}; it uses
neither the equivalent of $V$ nor that of $U$. For the training loss, it combines the exact training
identity~\eqref{eq:traintraceid} with the equivalent of $V$. In both cases the resolvent
$Q=(U+\lambda\Id{p})^{-1}$ is kept as it is, so the GEP for $U$ (Lemma~\ref{lem:Ut}) is not used. We use the decomposition of $\Gamma$ into the covariance of its conditional
mean and that of its residual,
\begin{equation*}
\Gamma_{\mathrm{res}}=h_tI_x+a_t^2\Sigma_x
 =\operatorname{diag}(\Sigma_t,0_{\dy}),\qquad
\Lambda=\Gamma-\Gamma_{\mathrm{res}}.
\end{equation*}
Both matrices are positive semidefinite.

\begin{proof}[Proof of Lemma~\ref{lem:main-decomp}]
Since $\E[gg\T]=\Wop\Gamma\WopT$ by~\eqref{eq:qcov} and, as each entry of $g$ may be
replaced by $z\sim\N(0,1)$ with $\E[zr(z)]=0$, $r(g)$ is asymptotically
orthogonal to $g$, Lemma~\ref{lem:Utilde} gives
$\E[r(g)r(g)\T]\simeq\Utilde-\mu_1^2\Wop\Gamma\WopT\simeq v^2\Id{p}$. Likewise
$\E[g\mid y]$ has covariance $\Wop\Lambda\WopT$, and Lemma~\ref{lem:Uy} gives
$\E\big[\E[r(g)\mid y]\E[r(g)\mid y]\T\big]\simeq v_y^2\Id{p}$. Hence
\begin{equation*}
\frac1\dx\E\|s_{\Ahat}^{\mathrm{nl}}\|^2=\frac1{\dx p}\tr\big(\Ahat\,\E[r(g)r(g)\T]\Ahat\T\big)
\simeq\frac{v^2}{\dx p}\|\Ahat\|_F^2=v^2\Anorm,
\end{equation*}
and in the same way $\dx^{-1}\E\|\E[s_{\Ahat}^{\mathrm{nl}}\mid y]\|^2\simeq v_y^2\Anorm$, which
are the first two equivalences of Lemma~\ref{lem:main-decomp}.
By Lemma~\ref{lem:Uy} and~\eqref{eq:exact-quantities},
\begin{align*}
\ME(\Ahat)=\frac{h_t}{a_t^2\dx}\tr\big(V\T Q\Utilde_yQV\big)
&\simeq\frac{h_t\mu_1^2}{a_t^2\dx}\tr\big(V\T Q\Wop\Lambda\WopT QV\big)
+\frac{h_tv_y^2}{a_t^2\dx}\tr\big(V\T Q^2V\big)\\
&=\frac{\|\widehat B-B\|_F^2}{\dx}+\frac{h_t^2}{a_t^2}v_y^2\Anorm .
\end{align*}
By Lemmas~\ref{lem:Vtilde} and~\ref{lem:Utilde}, $\Vtilde_\xi\simeq a_t\mu_1\Wa\Sigma^{1/2}$,
$\Vtilde_\eta\simeq\mu_1\Wa\Sigma_t$, and
$\Utilde-\Utilde_y\simeq\mu_1^2\Wa\Sigma_t\WaT+(v^2-v_y^2)\Id{p}$, where we used
$\Wop\Gamma_{\mathrm{res}}\WopT=\Wa\Sigma_t\WaT$. Inserting these into
Lemma~\ref{lem:test-traceid} and comparing with $G_x$ in~\eqref{eq:exact-quantities} gives
\begin{align*}
\RA(\Ahat)&\simeq\upsilon_{\Sigma,d}^2-\frac{\sqrt{h_t}\,\mu_1}{\dx}\tr\big(V\T Q\Wa\Sigma\big)
=\frac{a_t}\dx\tr\big(\Sigma G_x\big),\\
\PV(\Ahat)&\simeq\frac1{a_t^2\dx}\Big[\tr\Sigma_t-2\sqrt{h_t}\,\mu_1\tr\big(V\T Q\Wa\Sigma_t\big)
+h_t\mu_1^2\tr\big(V\T Q\Wa\Sigma_t\WaT QV\big)\Big]\\
&\quad+\frac{h_t^2}{a_t^2}(v^2-v_y^2)\Anorm\\
&=\frac1\dx\tr\big(G_x\Sigma_tG_x\T\big)+\frac{h_t^2}{a_t^2}(v^2-v_y^2)\Anorm .
\end{align*}
These are the equivalences for $\ME$, $\PV$, and $\RA$.
Finally, by~\eqref{eq:traintraceid} and~\eqref{eq:exact-quantities},
\begin{equation*}
\Ltrain(\Ahat)+h_t\lambda \Anorm=1-\frac1\dx\tr\big(V\T QV\big),\qquad
\frac{a_t}\dx\tr G_x=1-\frac{\sqrt{h_t}\,\mu_1}{\dx}\tr\big(V\T Q\Wa\big)
\end{equation*}
exactly. Their difference is $\dx^{-1}\tr\big(V\T Q(\sqrt{h_t}\mu_1\Wa-V)\big)$, and
Lemma~\ref{lem:Vt} gives the equivalence for $\Ltrain$.
\end{proof}

For the true denoiser, $G_x=a_t\Sigma\Sigma_t^{-1}$, $\widehat B=B$, and there are no terms in $\Anorm$.

\subsection{Reduction to resolvent traces}
\label{sec:source}
With the resolvent $Q=(U+\lambda\Id{p})^{-1}$, define
\begin{align*}
E_1&=\dx^{-1}\tr(\WaT Q\Wa),&
E_\Sigma&=\dx^{-1}\tr(\Sigma\WaT Q\Wa),\\
E_3&=\dx^{-1}\tr(\WaT Q^2\Wa),&
E_{\mathrm m}&=\dx^{-1}\tr(\WaT Q\Wop\Lambda\WopT Q\Wa),\\
E_{\mathrm r}&=\dx^{-1}\tr(\WaT Q\Wop\Gamma_{\mathrm{res}}\WopT Q\Wa).
\end{align*}
By Lemma~\ref{lem:Vt},
$V\simeq\sqrt{h_t}\mu_1\Wa$, so~\eqref{eq:traintraceid} and~\eqref{eq:exact-quantities} give,
in normalized expectation,
\begin{align}
\Ltrain&\simeq1-h_t\mu_1^2E_1-h_t\lambda\mu_1^2E_3,\label{eq:red-train}\\
\Anorm&\simeq\mu_1^2E_3,\qquad
\frac{\|\widehat B-B\|_F^2}{\dx}\simeq\frac{h_t^2\mu_1^4}{a_t^2}E_{\mathrm m},\label{eq:red-N}\\
G_x&\simeq\frac1{a_t}\Big(I-h_t\mu_1^2\WaT Q\Wa\Big).\label{eq:red-Gx}
\end{align}
Using $\Sigma_t=a_t^2\Sigma+h_tI$ and $\Wop\Gamma_{\mathrm{res}}\WopT=\Wa\Sigma_t\WaT$,
\eqref{eq:red-Gx} gives
\begin{equation}
\begin{aligned}
\frac{a_t}\dx\tr G_x&\simeq1-h_t\mu_1^2E_1,\qquad
\frac{a_t}\dx\tr(\Sigma G_x)\simeq\upsilon_{\Sigma,d}^2-h_t\mu_1^2E_\Sigma,\\
\frac1\dx\tr\big(G_x\Sigma_tG_x\T\big)&\simeq\frac1{a_t^2}\Big[\frac1\dx\tr\Sigma_t
-2h_t\mu_1^2\big(h_tE_1+a_t^2E_\Sigma\big)+h_t^2\mu_1^4E_{\mathrm r}\Big].
\end{aligned}
\label{eq:red-Gx-traces}
\end{equation}
Thus $E_1$ and $E_\Sigma$ are traces of $G_x$, $E_{\mathrm r}$ its quadratic form in
$\Sigma_t$, $E_{\mathrm m}$ the error of the conditional-mean map, and $E_3$ the weight norm.

\begin{lemma}[Perturbed resolvent]
\label{lem:source}
Define $\mathbf q=(q_{\mathrm m},q_{\mathrm r})$ and
\begin{align}
U(\mathbf q)&=U+q_{\mathrm m}\Wop\Lambda\WopT
 +q_{\mathrm r}\Wop\Gamma_{\mathrm{res}}\WopT,\label{eq:Uq}\\
R(\mathbf q,z)&=(U(\mathbf q)-zI_p)^{-1},\label{eq:Rqz}\\
K(\mathbf q,z)&=\dx^{-1}\tr(\WaT R\Wa),\qquad
K_\Sigma(\mathbf q,z)=\dx^{-1}\tr(\Sigma\WaT R\Wa).\label{eq:Kqz}
\end{align}
At $(\mathbf q,z)=(\mathbf0,-\lambda)$,
\begin{equation}
E_1=K,\quad E_3=\partial_zK,\quad
E_{\mathrm m}=-\partial_{q_{\mathrm m}}K,\quad E_{\mathrm r}=-\partial_{q_{\mathrm r}}K,\quad E_\Sigma=K_\Sigma.
\label{eq:Efromsource}
\end{equation}
\end{lemma}
\begin{proof}
At $(\mathbf q,z)=(\mathbf0,-\lambda)$, definitions \eqref{eq:Uq} and
\eqref{eq:Rqz} give $R(\mathbf0,-\lambda)=(U+\lambda I_p)^{-1}=Q$.
The inverse differentiation identity
$\partial A^{-1}=-A^{-1}(\partial A)A^{-1}$,
applied to $A=U(\mathbf q)-zI_p$, with
$H_{\mathrm m}=\Lambda$ and $H_{\mathrm r}=\Gamma_{\mathrm{res}}$, gives
\begin{equation}
\partial_{q_j}R=-R\Wop H_j\WopT R,\qquad \partial_zR=R^2,\qquad j\in\{\mathrm m,\mathrm r\}.
\label{eq:resolventderiv}
\end{equation}
Inserting these derivatives into \eqref{eq:Kqz} and evaluating at
$(\mathbf q,z)=(\mathbf0,-\lambda)$ gives \eqref{eq:Efromsource}
by the definitions of the normalized traces.
\end{proof}
Write the deterministic limits as
\begin{equation}
\Kbar=\lim\E K,\qquad\Kbar_\Sigma=\lim\E K_\Sigma.
\label{eq:Kbardef}
\end{equation}

\subsection{Linearization and the perturbed resolvent}
\label{sec:pencil}

\begin{lemma}[Deterministic equivalent of the perturbed resolvent]
\label{lem:compact}
Set
\begin{equation}
\Cone(\mathbf q):=h_t\mu_1^2I_x+q_{\mathrm m}\Lambda+q_{\mathrm r}\Gamma_{\mathrm{res}},
\label{eq:Cone}
\end{equation}
and recall $\mathrm{Tr}_n$ from~\eqref{eq:Trn}. For the scalars $m$ and $\rho$ defined in \eqref{eq:scalars}, set
\begin{equation*}
\Kop(\rho;\mathbf q):=\Cone(\mathbf q)
+\rho\mu_1^2\Tdet\Tdet\T,
\qquad
\Aop:=\Id{\dop}+\frac{\psi_pm}{2}\Kop(\rho;\mathbf q).
\end{equation*}
Assume that the limits defining the normalized inverse-block traces below exist. Then the
relevant nonnegative pair $(m,\rho)$ satisfies
\begin{align}
\frac{1}{m}
&=s^2-z+\frac{\psi_n}{2}\mathrm{Tr}_n\big[\Aop^{-1}\Kop\big]+v_\ell^2\rho,
\label{eq:compact1}\\
\frac{1}{\rho}
&=1+v_\ell^2\frac{\psi_p}{\psi_n}m
+\frac{\psi_pm}{2}\mu_1^2
\mathrm{Tr}_n\big[\Tdet\T\Aop^{-1}\Tdet\big],
\label{eq:compact2}
\end{align}
and
\begin{equation}
\Kbar(\mathbf q,z)
=\psi_n\frac{\psi_pm}{2}\,\mathrm{Tr}_n\big[I_x\Aop^{-1}\big].
\label{eq:readoutscalar}
\end{equation}
Moreover,
\begin{equation}
\Kbar_\Sigma=\psi_n\frac{\psi_pm}{2}\,\mathrm{Tr}_n[\Sigma_x\Aop^{-1}].
\label{eq:weighted-readout}
\end{equation}
\end{lemma}

\begin{proof}
\medskip\noindent\textbf{Gaussian representation.} Recall the design matrices $\Wcat$, $\Zcat$, and $\Tdet$ from \eqref{eq:feature} and the data-block projector $I_x$.
The three random matrices $\Wcat$, $\Omega$, and $\Zcat$ are mutually independent with
i.i.d.\ standard Gaussian entries and enter through the normalized Gaussian blocks
$\Wop$, $\Eop$, and $\Zop$. Then \eqref{eq:feature} and \eqref{eq:Gtdef} give
\begin{align}
\frac{\Gt}{\sqrt n}
&=\frac{\mu_1 L+v_\ell\Omega}{\sqrt n}\notag\\
&=\Eop+\frac{\mu_1}{\sqrt n}\Wop
\begin{bmatrix}a_tX\\ \kappa_dY\end{bmatrix}\notag\\
&=\Eop+\Wop\Ctwo\Zop ,
\notag\\[3pt]
R(\mathbf q,z)^{-1}
&=U(\mathbf q)-z\Id{p}\notag\\
&=\frac{\Gt\Gt\T}{n}
+\Wop\big(h_t\mu_1^2I_x+q_{\mathrm m}\Lambda+q_{\mathrm r}\Gamma_{\mathrm{res}}\big)\WopT
+(s^2-z)\Id{p}
\notag\\
&=(s^2-z)\Id{p}+\Wop\Cone(\mathbf q)\WopT+\frac{\Gt\Gt\T}{n} .
\label{eq:Rgauss}
\end{align}
\medskip\noindent\textbf{Six-block linear pencil.} Use the block order
\begin{equation*}
(N_1,\dots,N_6)=(p,\dop,\dop,n,\dop,\dop)
\end{equation*}
and define
\begin{equation}
\Lpen(\mathbf q,z):=
\begin{bmatrix}
(s^2-z)\Id{p} & \Wop & 0 & \Eop & 0 & 0\\
0 & I & -\Cone & 0 & -\Ctwo & 0\\
-\WopT & 0 & I & 0 & 0 & 0\\
-\EopT & 0 & 0 & \Id{n} & 0 & -\ZopT\\
0 & 0 & 0 & -\Zop & I & 0\\
0 & 0 & -\CtwoT & 0 & 0 & I
\end{bmatrix}.
\label{eq:pencil}
\end{equation}
Every unlabelled identity in \eqref{eq:pencil} has size $\dop$. Each random block in
\eqref{eq:pencil} is an i.i.d.\ Gaussian matrix or its transpose; all products with $B$ and $\Sigma^{1/2}$
occur only in the deterministic part. Inverting $\Lpen(\mathbf q,z)$ gives
\begin{align}
\big[\Lpen(\mathbf q,z)^{-1}\big]^{11}
&=R(\mathbf q,z)\notag\\
\big[\Lpen(\mathbf q,z)^{-1}\big]^{32}
&=-\WopT R(\mathbf q,z)\Wop ,
\notag\\
K(\mathbf q,z)
&=\frac1{\dx}\tr\!\big(\WaT R(\mathbf q,z)\Wa\big)
=-\frac1{\dx}\tr\!\Big(
I_x\big[\Lpen(\mathbf q,z)^{-1}\big]^{32}\Big),\notag\\
K_\Sigma(\mathbf q,z)
&=\frac1{\dx}\tr\!\big(\Sigma\WaT R(\mathbf q,z)\Wa\big)
=-\frac1{\dx}\tr\!\Big(
\Sigma_x\big[\Lpen(\mathbf q,z)^{-1}\big]^{32}\Big).
\label{eq:readout}
\end{align}
At the evaluation point, \eqref{eq:Rgauss} and $s^2\ge0$ give
$R(\mathbf0,-\lambda)^{-1}\succeq(s^2+\lambda)\Id{p}\succeq\lambda\Id{p}$.
For $q_{\mathrm m},q_{\mathrm r}\ge0$, standard Gaussian operator-norm bounds and the same block elimination give
$\|\Lpen(\mathbf q,z)^{-1}\|_{\op}=O_{\Prob}(1)$ at the ridge evaluation point and on a one-sided neighborhood in $\mathbf q$; the analytic continuation and expectation-level limits are assumed as specified above.

\medskip\noindent\textbf{Block equation.} Set $\Bpen:=\E\Lpen$, the deterministic part of the pencil, and
$M^{\mathrm B}:=\Lpen-\Bpen$. Let
\begin{equation*}
\mathcal{I}:=\{1,\ldots,6\},
\qquad
S:=\{(i,j)\in\mathcal{I}^2:N_i=N_j\},
\end{equation*}
and set $J_{ij}:=\mathbf{1}\{(i,j)\in S\}$. For any conformal block matrix $H$,
\begin{equation*}
\big[(J\otimes\mathrm{Tr}_n)[H]\big]_{ij}:=
\begin{cases}
\mathrm{Tr}_n[H^{ij}], & (i,j)\in S,\\
0, & (i,j)\notin S.
\end{cases}
\end{equation*}
Following \cite[(3.2)--(3.5)]{bodin2024}, define, for $(i,\ell),(j,k)\in S$,
\begin{equation*}
\sigma^{k\ell}_{ij}:=n\,\E\big[(M^{\mathrm B})^{ij}_{uv}(M^{\mathrm B})^{k\ell}_{vu}\big],
\qquad
[\etapen(g)]_{i\ell}:=\sum_{(j,k)\in S}\sigma^{k\ell}_{ij}\,g_{jk}.
\end{equation*}
The block version $(\etapen\otimes I)(g)$ equals
$[\etapen(g)]_{i\ell}\Id{N_i}$ for $(i,\ell)\in S$ and is zero otherwise. Set
\begin{equation*}
g_{ij}:=\mathrm{Tr}_n\big[[(\Lpen)^{-1}]^{ij}\big],\qquad (i,j)\in S,
\end{equation*}
and $g_{ij}=0$ outside $S$. Under the assumptions above,
\cite[Ch.~3, Result~3.1]{bodin2024} gives
\begin{equation}
g=(J\otimes\mathrm{Tr}_n)[G],
\qquad
G:=\big(\Bpen-(\etapen\otimes I)(g)\big)^{-1}.
\label{eq:Gdef}
\end{equation}

\medskip\noindent\textbf{The linear operator $\etapen$.} The random blocks are
\begin{equation*}
\begin{array}{c|c}
\text{primitive} & \text{blocks}\\ \hline
\Wop & (1,2),\ (3,1)\\
\Eop & (1,4),\ (4,1)\\
\Zop & (4,6),\ (5,4).
\end{array}
\end{equation*}
The nonzero covariances are
\begin{align}
\sigma^{31}_{12}=\sigma^{12}_{31}&=-\frac{\psi_n}{2},\notag\\
\sigma^{41}_{14}=\sigma^{14}_{41}&=-v_\ell^2,\notag\\
\sigma^{54}_{46}=\sigma^{46}_{54}&=1,
\label{eq:sigmatable}
\end{align}
so that
\begin{equation}
\etapen(g)=
\begin{pmatrix}
-\frac{\psi_n}{2}g_{23}-v_\ell^2g_{44} & 0 & 0 & 0 & 0 & 0\\
0 & 0 & 0 & 0 & 0 & 0\\
0 & -\frac{\psi_n}{2}g_{11} & 0 & 0 & 0 & 0\\
0 & 0 & 0 & g_{65}-v_\ell^2g_{11} & 0 & 0\\
0 & 0 & 0 & 0 & 0 & g_{44}\\
0 & 0 & 0 & 0 & 0 & 0
\end{pmatrix}.
\label{eq:etamatrix}
\end{equation}

\medskip\noindent\textbf{Applying the block equation.} Explicitly, inserting \eqref{eq:etamatrix} into $\Bpen-(\etapen\otimes I)(g)$ gives
\begin{equation}
{\setlength{\arraycolsep}{1pt}
\begin{gathered}
\Bpen-(\etapen\otimes I)(g)=\\
\begin{bmatrix}
\big(s^2-z+\frac{\psi_n}{2}g_{23}+v_\ell^2g_{44}\big)\Id{p} & 0 & 0 & 0 & 0 & 0\\
0 & I & -\Cone & 0 & -\Ctwo & 0\\
0 & \frac{\psi_n}{2}g_{11}I & I & 0 & 0 & 0\\
0 & 0 & 0 & \big(1+v_\ell^2g_{11}-g_{65}\big)\Id{n} & 0 & 0\\
0 & 0 & 0 & 0 & I & -g_{44}I\\
0 & 0 & -\CtwoT & 0 & 0 & I
\end{bmatrix}
\end{gathered}},
\label{eq:dysonmatrix}
\end{equation}
Here the four entries carrying $g$ are the contributions $-[\etapen(g)]_{i\ell}$ of
\eqref{eq:etamatrix}, and every other entry is inherited from $\Bpen$. The first and fourth
rows and columns of \eqref{eq:dysonmatrix} meet only on the diagonal, so those two sectors
decouple. Define
\begin{equation}
\begin{aligned}
m&:=\big(s^2-z+\frac{\psi_n}{2}g_{23}+v_\ell^2g_{44}\big)^{-1},
& G^{11}&=m\Id{p},\\
\rho&:=\big(1+v_\ell^2g_{11}-g_{65}\big)^{-1},
& G^{44}&=\rho\Id{n}.
\end{aligned}
\label{eq:scalars}
\end{equation}
Consequently, $g_{11}=(\psi_p/\psi_n)m$ and $g_{44}=\rho$. Equation \eqref{eq:etamatrix} depends on $g$ only through
$g_{11}$, $g_{23}$, $g_{44}$, and $g_{65}$. After substituting it into \eqref{eq:Gdef}, taking the
normalized traces of the $(1,1)$, $(2,3)$, $(4,4)$, and $(6,5)$ blocks gives four equations for
these same four scalars; no other entry of $g$ is needed.

\medskip\noindent\textbf{Block elimination.} By \eqref{eq:scalars}, $\frac{\psi_n}{2}g_{11}=\frac{\psi_pm}{2}$ and $g_{44}=\rho$. Hence the
$(2,3,5,6)$ sector of \eqref{eq:dysonmatrix} is
\begin{equation*}
\begin{bmatrix}
I & -\Cone & -\Ctwo & 0\\
\frac{\psi_pm}{2}I & I & 0 & 0\\
0 & 0 & I & -\rho I\\
0 & -\CtwoT & 0 & I
\end{bmatrix}.
\end{equation*}
Using $\Kop=\Cone+\rho\mu_1^2\Tdet\Tdet\T$ and
$\Aop=I+\frac{\psi_pm}{2}\Kop$, block elimination gives
\begin{align*}
G^{23}
&=\big(I+\frac{\psi_pm}{2}\Kop\big)^{-1}\Kop
=\Aop^{-1}\Kop,\\
G^{32}
&=-\frac{\psi_pm}{2}\big(I+\frac{\psi_pm}{2}\Kop\big)^{-1}
=-\frac{\psi_pm}{2}\Aop^{-1},\\
G^{65}
&=-\frac{\psi_pm}{2}\mu_1^2\Tdet\T
\big(I+\frac{\psi_pm}{2}\Kop\big)^{-1}\Tdet\\
&=-\frac{\psi_pm}{2}\mu_1^2
\Tdet\T\Aop^{-1}\Tdet.
\end{align*}
Taking normalized traces gives
\begin{align*}
g_{23}&=\mathrm{Tr}_n\big[\Aop^{-1}\Kop\big],\\
g_{65}&=-\frac{\psi_pm}{2}\mu_1^2
\mathrm{Tr}_n\big[\Tdet\T\Aop^{-1}\Tdet\big].
\end{align*}
Substituting these traces and
$g_{11}=(\psi_p/\psi_n)m$, $g_{44}=\rho$ into \eqref{eq:scalars} yields
\begin{align*}
\frac1{m}
&=s^2-z+\frac{\psi_n}{2}g_{23}+v_\ell^2g_{44}
=s^2-z+\frac{\psi_n}{2}\mathrm{Tr}_n\big[\Aop^{-1}\Kop\big]+v_\ell^2\rho,
\\
\frac1\rho
&=1+v_\ell^2g_{11}-g_{65}
=1+v_\ell^2\frac{\psi_p}{\psi_n}m
+\frac{\psi_pm}{2}\mu_1^2
\mathrm{Tr}_n\big[\Tdet\T\Aop^{-1}\Tdet\big].
\end{align*}

\medskip\noindent\textbf{$\Kbar$ and $\Kbar_\Sigma$.}
We first obtain the data-block trace.
Split only blocks $2$ and $3$
of \eqref{eq:pencil} into their $\dx$- and $\dy$-dimensional components, and write
$3_x,2_x$ for the $\dx$-dimensional parts. By
\eqref{eq:readout},
\begin{align*}
K(\mathbf q,z)
&=-\frac1\dx\tr\!\big(I_x\big[\Lpen(\mathbf q,z)^{-1}\big]^{32}\big)
=-\frac{n}{\dx}\cdot\frac1n\tr\!\big(\big[\Lpen(\mathbf q,z)^{-1}\big]^{3_x2_x}\big).
\end{align*}
Since $N_{3_x}=N_{2_x}=\dx$, the pair $(3_x,2_x)$ belongs to the index set $S$ of the split
pencil, so \cite[Ch.~3, Result~3.1]{bodin2024} applies to this block directly. The entries of
$\Wop=W^{\oplus}/\sqrt{2\dx}$ have the same variance in the $W_1$ and $W_2$ columns, so each
split block carries the covariance of \eqref{eq:sigmatable}, and the covariances between
different column blocks vanish. Since $\mathrm{Tr}_n$ normalizes every block by $n$,
$g_{23}=g_{2_x3_x}+g_{2_y3_y}$. Hence $(\etapen\otimes I)(g)$ of the split pencil is that of
\eqref{eq:etamatrix} with its blocks re-indexed, and the solution of \eqref{eq:Gdef} for the
split pencil is $G$ re-indexed; in particular, $G^{3_x2_x}=[G^{32}]_{xx}$.
Taking expectations, $n/\dx\to\psi_n$, \eqref{eq:Gdef}, and \eqref{eq:Kbardef} therefore give
\begin{align*}
\Kbar(\mathbf q,z)
&=-\psi_n\mathrm{Tr}_n\big[G^{3_x2_x}\big]
=-\psi_n\mathrm{Tr}_n\big[I_xG^{32}\big]\\
&=\psi_n\frac{\psi_pm}{2}\,\mathrm{Tr}_n\big[I_x\Aop^{-1}\big],
\end{align*}
where the last equality uses $G^{32}=-\frac{\psi_pm}{2}\Aop^{-1}$ from the block elimination
above.

For $K_\Sigma$, change the data coordinates by an orthogonal $O\in\R^{\dx\times\dx}$ with
$O\T\Pi O=\operatorname{diag}(I_{\dy},0)$. This replaces $W_1$ by $W_1O$, which has the same
law by Gaussian rotational invariance, transforms the deterministic blocks accordingly, and
leaves $K_\Sigma$ unchanged; we may therefore assume $\Pi=\operatorname{diag}(I_{\dy},0)$.
Split the $\dx$-dimensional parts of blocks $2$ and $3$ further into the range of $\Pi$ and
its orthogonal complement. The corresponding column blocks of $W_1$ again have equal entry
variance and vanishing cross covariances, so, as for the $x/y$ split above,
$(\etapen\otimes I)(g)$ of this split pencil is that of \eqref{eq:etamatrix} re-indexed. The
two inverse-block traces on the range of $\Pi$ and on that of $\Pi_\perp$ give the traces
weighted by $\Pi$ and $\Pi_\perp$. Multiplying the first by $1-c^2$ and the second by $1$, and
summing, proves \eqref{eq:weighted-readout} because
$\Sigma_x=\operatorname{diag}\big((1-c^2)\Pi+\Pi_\perp,\,0_{\dy}\big)$. If $\dy=\dx$, only the
first block is present.
\end{proof}

\subsection{Spectral reduction and the two-scalar fixed-point equations}
\label{sec:spectralclosure}
For a function $f$, define
\begin{equation}
\begin{aligned}
\mathcal M[f(\varsigma)]&:=\lim\frac1\dx\tr f(\Sigma)\\
&\phantom{:}=\psi_yf(1-c^2)+(1-\psi_y)f(1).
\end{aligned}
\label{eq:direction-average}
\end{equation}
Note that $\Sigma$ has a fraction $\psi_y$ of eigenvalues equal to $1-c^2$ on the observed subspace $\operatorname{range}(\Pi)$ and a fraction $1-\psi_y$ equal to $1$ on its orthogonal complement; thus, $\mathcal M[f(\varsigma)]$ represents the average of $f$ over the spectrum of $\Sigma$.

\begin{lemma}[Spectral evaluation of the normalized traces]
\label{lem:spectral}

For the data model~\eqref{eq:data-model}, take $\beta$, $\delta$, $\ell(\varsigma)$, and
$D(\varsigma)$ as defined in Theorem~\ref{thm:appendix-final}, and $\mathcal M$
from~\eqref{eq:direction-average}.

The normalized traces satisfy
\begin{align}
\mathrm{Tr}_n[I_x\Aop^{-1}]&=2\psi_n^{-1}\mathcal M[1/D(\varsigma)],
\qquad \mathrm{Tr}_n[\Sigma_x\Aop^{-1}]=2\psi_n^{-1}\mathcal M[\varsigma/D(\varsigma)],
\label{eq:trPxA}\\
\mathrm{Tr}_n[\Aop^{-1}\Kop]&=\psi_n^{-1}
 \Big(\frac\beta\delta+2\mathcal M\Big[\frac{\ell(\varsigma)+a_t^2\beta(1-\varsigma)/\delta^2}{D(\varsigma)}\Big]\Big),
\label{eq:trAK}\\
\mu_1^2\mathrm{Tr}_n[\Tdet\T\Aop^{-1}\Tdet]&=\psi_n^{-1}
 \Big(\frac{\mu_1^2}\delta+2a_t^2\mu_1^2\mathcal M\Big[\frac{\varsigma+(1-\varsigma)/\delta^2}{D(\varsigma)}\Big]\Big).
\label{eq:trCAC}
\end{align}
\end{lemma}

\begin{proof}
By \eqref{eq:Cone}, \eqref{eq:TTt}, and the covariance decomposition in
Appendix~\ref{sec:app-decomp},
\begin{align*}
\Kop
&=\Cone(\mathbf q)+\rho\mu_1^2\Tdet\Tdet\T\\
&=h_t\mu_1^2I_x+q_{\mathrm m}\Lambda
 +q_{\mathrm r}\Gamma_{\mathrm{res}}
 +\rho\mu_1^2(\Lambda+\Gamma_{\mathrm{res}}-h_tI_x)\\
&=h_t\mu_1^2(1-\rho)I_x+\beta\Lambda
 +(\mu_1^2\rho+q_{\mathrm r})\Gamma_{\mathrm{res}},\\[2pt]
\Aop
&=\Id{\dop}+\frac{\psi_pm}{2}\Kop\\
&=\begin{bmatrix}
I_{\dx}+\frac{\psi_pm}{2}\ell(\Sigma)+a_t^2\frac{\psi_pm}{2}\beta BB\T
&\frac{\psi_pm}{2}\kappa_da_t\beta B\\
\frac{\psi_pm}{2}\kappa_da_t\beta B\T
&(1+\kappa_d^2\frac{\psi_pm}{2}\beta)I_{\dy}
\end{bmatrix},
\end{align*}
where
\begin{equation*}
\ell(\Sigma):=h_t(\mu_1^2+q_{\mathrm r})I_{\dx}
 +a_t^2(\mu_1^2\rho+q_{\mathrm r})\Sigma.
\end{equation*}

Define
\begin{align*}
\delta_d&:=1+\kappa_d^2\frac{\psi_pm}{2}\beta,\\
\Delta_x(\alpha,\zeta)
&:=I_{\dx}+\frac{\psi_pm}{2}\ell(\Sigma)+\alpha I_{\dx}+\zeta\Sigma
 +\frac{a_t^2\psi_pm\beta}{2\delta_d}BB\T,
\qquad \Delta_x:=\Delta_x(0,0).
\end{align*}
Since $BB\T=c^2\Pi=I_{\dx}-\Sigma$, every matrix above is a function of $\Sigma$.
In particular $\Delta_x=D_d(\Sigma)/2$, where $D_d$ is $D$ with $\delta_d$ in
place of $\delta$. The matrix $\Sigma$ has eigenvalue $1-c^2$ with multiplicity
$\dy$ and eigenvalue $1$ with multiplicity $\dx-\dy$. Hence, with
$\psi_{y,d}:=\dy/\dx$,
\begin{equation*}
\frac1\dx\tr f(\Sigma)=\mathcal M_d[f(\varsigma)]
:=\psi_{y,d}f(1-c^2)+(1-\psi_{y,d})f(1).
\end{equation*}
The Schur-complement identity gives
\begin{align}
\det(\Aop+\alpha I_x+\zeta\Sigma_x)
&=\delta_d^{\dy}\det\left[
 I_{\dx}+\frac{\psi_pm}{2}\ell(\Sigma)+\alpha I_{\dx}+\zeta\Sigma
 +\left(a_t^2\frac{\psi_pm}{2}\beta
 -\frac{a_t^2\kappa_d^2(\psi_pm)^2\beta^2}{4\delta_d}\right)BB\T
 \right]\notag\\
&=\delta_d^{\dy}\det\Delta_x(\alpha,\zeta).
\label{eq:variant-schur}
\end{align}
Consequently, Jacobi's formula yields
\begin{align*}
\frac1n\tr(I_x\Aop^{-1})
&=\left.\frac1n\partial_\alpha
 \log\det(\Aop+\alpha I_x+\zeta\Sigma_x)\right|_{\alpha=\zeta=0}\\
&=\left.\frac1n\tr\!\left(
 \Delta_x(\alpha,\zeta)^{-1}\partial_\alpha\Delta_x(\alpha,\zeta)
 \right)\right|_{\alpha=\zeta=0}
 =\frac{2\dx}{n}\mathcal M_d[1/D_d(\varsigma)],\\
\frac1n\tr(\Sigma_x\Aop^{-1})
&=\left.\frac1n\partial_\zeta
 \log\det(\Aop+\alpha I_x+\zeta\Sigma_x)\right|_{\alpha=\zeta=0}\\
&=\left.\frac1n\tr\!\left(
 \Delta_x(\alpha,\zeta)^{-1}\partial_\zeta\Delta_x(\alpha,\zeta)
 \right)\right|_{\alpha=\zeta=0}
 =\frac{2\dx}{n}\mathcal M_d[\varsigma/D_d(\varsigma)].
\end{align*}

Next, differentiate $\det\Aop=\delta_d^{\dy}\det\Delta_x$ with respect to $m$,
holding $\rho$ and $\mathbf q$ fixed. Since
\begin{equation*}
\partial_m\delta_d=\kappa_d^2\psi_p\beta/2,
\qquad
\partial_mD_d(\varsigma)=\psi_p\ell(\varsigma)
 +\frac{a_t^2\psi_p\beta}{\delta_d^2}(1-\varsigma),
\end{equation*}
Jacobi's formula and $\dy\kappa_d^2=\dx$ give
\begin{align*}
\frac1n\tr(\Aop^{-1}\Kop)
&=\frac2{n\psi_p}\tr(\Aop^{-1}\partial_m\Aop)
 =\frac2{n\psi_p}\partial_m\log\det\Aop\\
&=\frac{2\dy}{n\psi_p}\frac{\partial_m\delta_d}{\delta_d}
 +\frac2{n\psi_p}\tr\!\left(\Delta_x^{-1}\partial_m\Delta_x\right)\\
&=\frac{\dx}{n}\left(
 \frac{\beta}{\delta_d}
 +2\mathcal M_d\Big[\frac{\ell(\varsigma)+a_t^2\beta(1-\varsigma)/\delta_d^2}{D_d(\varsigma)}\Big]\right).
\end{align*}

Finally, differentiate with respect to $\rho$, holding $m$ and $\mathbf q$ fixed.
Using $\partial_\rho\ell(\varsigma)=a_t^2\mu_1^2\varsigma$, we obtain
\begin{align*}
\partial_\rho\delta_d&=\kappa_d^2\frac{\psi_pm}{2}\mu_1^2,\\
\partial_\rho D_d(\varsigma)&=a_t^2\psi_pm\mu_1^2\Big(\varsigma
 +\frac{1-\varsigma}{\delta_d^2}\Big),\\
\partial_\rho\Aop&=\frac{\psi_pm}{2}\,\partial_\rho\Kop
 =\frac{\psi_pm}{2}\mu_1^2\Tdet\Tdet\T.
\end{align*}
Cyclicity of the trace and Jacobi's formula therefore give
\begin{align*}
\frac{\psi_pm}{2n}\mu_1^2\tr(\Tdet\T\Aop^{-1}\Tdet)
&=\frac1n\tr(\Aop^{-1}\partial_\rho\Aop)
 =\frac1n\partial_\rho\log\det\Aop\\
&=\frac{\dy}{n}\frac{\partial_\rho\delta_d}{\delta_d}
 +\frac1n\tr\!\left(\Delta_x^{-1}\partial_\rho\Delta_x\right)\\
&=\frac{\dx}{n}\frac{\psi_pm}{2}\left(
 \frac{\mu_1^2}{\delta_d}
 +2a_t^2\mu_1^2\mathcal M_d\Big[\frac{\varsigma+(1-\varsigma)/\delta_d^2}{D_d(\varsigma)}\Big]\right).
\end{align*}

As $\psi_{y,d}\to\psi_y$, we have $\delta_d\to\delta$, $D_d\to D$, and
$\mathcal M_d\to\mathcal M$. Taking limits in the trace identities above, using
$\dx/n\to1/\psi_n$ and dividing the last identity by $\frac{\psi_pm}{2}$,
proves \eqref{eq:trPxA}--\eqref{eq:trCAC}.
\end{proof}

\begin{theorem}[Proportional training loss and denoiser decomposition]
\label{thm:appendix-final}
Under the data model~\eqref{eq:data-model}, the activation assumption of Section~\ref{sec:setting},
and the assumptions at the start of this appendix,
fix $t>0$ and $\lambda>0$, and let $n,p,\dx,\dy\to\infty$ with
$n/\dx\to\psi_n>0$, $p/\dx\to\psi_p>0$, and $\dy/\dx\to\psi_y\in(0,1]$.
With the activation coefficients defined above and $\mathcal M$ from~\eqref{eq:direction-average},
for $\mathbf q=(q_{\mathrm m},q_{\mathrm r})$, define
\begin{align}
\beta&=q_{\mathrm m}+\mu_1^2\rho,\qquad
\delta=1+\frac{\psi_pm\beta}{2\psi_y},
\label{eq:appendix-beta-delta}\\
\ell(\varsigma)&=h_t(\mu_1^2+q_{\mathrm r})+a_t^2(\mu_1^2\rho+q_{\mathrm r})\varsigma,
\label{eq:appendix-ell}\\
D(\varsigma)&=2+\psi_pm\ell(\varsigma)+\frac{a_t^2\psi_pm\beta(1-\varsigma)}{\delta}.
\label{eq:appendix-denominator}
\end{align}
Let $(m,\rho)$ be the solution, as a function of $(\mathbf q,z)$ near $(\mathbf0,-\lambda)$,
with $m>0$ and $0<\rho<1$ at $(\mathbf q,z)=(\mathbf0,-\lambda)$, of
\begin{align}
\frac1m&=s^2-z+v_\ell^2\rho+\frac\beta{2\delta}
 +\mathcal M\Big[\frac{\ell(\varsigma)+a_t^2\beta(1-\varsigma)/\delta^2}{D(\varsigma)}\Big],
\label{eq:appendix-fixed-m}\\
\frac1\rho&=1+\frac{\psi_pm}{\psi_n}
 \left[v_\ell^2+\frac{\mu_1^2}{2\delta}
 +a_t^2\mu_1^2\mathcal M\Big[\frac{\varsigma+(1-\varsigma)/\delta^2}{D(\varsigma)}\Big]\right],
\label{eq:appendix-fixed-rho}
\end{align}
and define
\begin{equation}
\Kbar=\psi_pm\mathcal M\Big[\frac1{D(\varsigma)}\Big],\qquad
\Kbar_\Sigma=\psi_pm\mathcal M\Big[\frac{\varsigma}{D(\varsigma)}\Big].
\label{eq:appendix-readouts}
\end{equation}
Evaluate these functions and their derivatives at
$(\mathbf q,z)=(\mathbf0,-\lambda)$. Then the quantities of
Section~\ref{sec:test-decomposition} converge in expectation to
\begin{align}
\lim\E[\Anorm]&=\mu_1^2\,\partial_z\Kbar,\label{eq:appendix-N}\\
\lim\E\Big[\frac{\|\widehat B-B\|_F^2}{\dx}\Big]&=-\frac{h_t^2\mu_1^4}{a_t^2}\,\partial_{q_{\mathrm m}}\Kbar,\\
\lim\E\Big[\frac1\dx\tr\big(G_x\Sigma_tG_x\T\big)\Big]&=\upsilon_\Sigma^2+\frac{h_t}{a_t^2}
-\frac{2h_t^2\mu_1^2}{a_t^2}\Kbar-2h_t\mu_1^2\Kbar_\Sigma
-\frac{h_t^2\mu_1^4}{a_t^2}\,\partial_{q_{\mathrm r}}\Kbar,\label{eq:appendix-GSG}\\
\lim\E\Big[\frac{a_t}{\dx}\tr\big(\Sigma G_x\big)\Big]&=\upsilon_\Sigma^2-h_t\mu_1^2\Kbar_\Sigma,\label{eq:appendix-trSG}\\
\lim\E\Big[\frac{a_t}{\dx}\tr G_x\Big]&=1-h_t\mu_1^2\Kbar .\label{eq:appendix-trGx}
\end{align}
Consequently, by Lemma~\ref{lem:main-decomp}, the limits of $\ME$, $\PV$, $\RA$, and
$\Ltrain$ follow, that of
$\RE=\upsilon_{\Sigma,d}^2+\PV-2\RA$ follows, and
$\lim\E[\Ltest(\Ahat)]=(a_t^2/h_t)\lim\E[\ME(\Ahat)+\RE(\Ahat)]$. In particular, the training
loss of the ridge minimizer $\Ahat$ converges to
\begin{equation}
\lim\E[\Ltrain(\Ahat)]=1-h_t\mu_1^2\Kbar-h_t\lambda\mu_1^2\partial_z\Kbar .
\label{eq:appendix-train-loss}
\end{equation}
\end{theorem}
\begin{proof}[Proof]
Substitution of Lemma~\ref{lem:spectral} into
\eqref{eq:compact1}--\eqref{eq:readoutscalar} and
\eqref{eq:weighted-readout} gives the fixed-point equations
\eqref{eq:appendix-fixed-m}--\eqref{eq:appendix-fixed-rho} and the functions
\eqref{eq:appendix-readouts}. By Lemma~\ref{lem:source}, \eqref{eq:Kbardef}, and the
assumed interchange of limits and derivatives in $\mathbf q$ and $z$,
$E_1\to\Kbar$, $E_\Sigma\to\Kbar_\Sigma$, $E_{\mathrm m}\to-\partial_{q_{\mathrm m}}\Kbar$,
$E_{\mathrm r}\to-\partial_{q_{\mathrm r}}\Kbar$, and $E_3\to\partial_z\Kbar$ in expectation.
Inserting these limits into~\eqref{eq:red-train}--\eqref{eq:red-Gx-traces}, with
$\dx^{-1}\tr\Sigma_t\to a_t^2\upsilon_\Sigma^2+h_t$, gives~\eqref{eq:appendix-N}--\eqref{eq:appendix-train-loss}.
This proves Theorem~\ref{thm:appendix-final}, and hence Theorem~\ref{thm:main}.
\end{proof}

For numerical evaluation, we obtain $(m,\rho)$ by continuation in the ridge parameter from a
large value of $\lambda$, and then compute the derivatives in $\mathbf q$ and $z$ by implicit
differentiation at the solution.

\section{Training loss at large \texorpdfstring{$\psi_p$ and $\psi_n$}{psi\_p and psi\_n}}
\label{sec:app-train-test-ratio}
Throughout this section, fix $t>0$ and $s^2,v_\ell^2,\mu_1^2>0$.
The limits below are taken after the proportional limit of Theorem~\ref{thm:appendix-final}.

\subsection{Training loss in the joint large-aspect-ratio limit}
\label{sec:app-joint-train}
This subsection proves Proposition~\ref{thm:train}.

\paragraph{Proof of Proposition~\ref{thm:train}.}
Write $r=\psi_p/\psi_n$, $S=s^2+\lambda$, and
$\omega=\psi_p^{-1}+\psi_n^{-1}$. We use $\mathcal M$ from
\eqref{eq:direction-average}. All error bounds below are uniform in
$r>0$ and $0<\lambda\le\lambda_0$ for any fixed $\lambda_0>0$.

\paragraph{Fixed-point bounds.}
At $(\mathbf q,z)=(\mathbf0,-\lambda)$, \eqref{eq:appendix-ell} gives
\begin{equation}
\ell(\varsigma)=\mu_1^2(h_t+a_t^2\rho\varsigma)\ge\mu_1^2h_t>0.
\label{eq:ell-positive}
\end{equation}
Positivity of $m$ and $\rho$ gives $\delta>1$ by
\eqref{eq:appendix-beta-delta} and $0<\rho<1$ by \eqref{eq:appendix-fixed-rho}. By \eqref{eq:appendix-beta-delta},
$0\le\beta/(2\delta)\le\psi_y/(\psi_pm)$.
The averaged term in \eqref{eq:appendix-fixed-m} is also nonnegative and
is at most $(1+a_t^2\psi_y)/(\psi_pm)$.
Thus both terms after $S+v_\ell^2\rho$ in \eqref{eq:appendix-fixed-m}
are nonnegative and $O((\psi_pm)^{-1})$.
Multiplying \eqref{eq:appendix-fixed-m} by $m$ gives
\begin{equation}
\frac{1-O(\psi_p^{-1})}{S+v_\ell^2}\le m\le\frac1S.
\label{eq:joint-m-bounds}
\end{equation}
Thus $m$ is bounded above and away from zero. Equations
\eqref{eq:appendix-fixed-m}--\eqref{eq:appendix-readouts} then reduce to
\begin{equation}
\begin{aligned}
\frac1m&=S+v_\ell^2\rho+O(\psi_p^{-1}),\\
1&=\rho(1+rv_\ell^2m)+O(\psi_n^{-1}),\\
\Kbar&=\mathcal M[1/\ell(\varsigma)]+O(\psi_p^{-1}).
\end{aligned}
\label{eq:joint-leading-reduction}
\end{equation}
For the second line, multiply \eqref{eq:appendix-fixed-rho} by $\rho$ and use
$\psi_pm\mu_1^2\rho/(2\delta)\le\psi_y$ and
$D\ge\psi_pm\mu_1^2h_t$. For the third line,
$0\le D-\psi_pm\ell\le2(1+a_t^2\psi_y)$ in \eqref{eq:appendix-denominator}
gives $D/(\psi_pm)=\ell+O(\psi_p^{-1})$; taking reciprocals preserves
this error bound by \eqref{eq:ell-positive}.

\paragraph{Quadratic equation.}
We keep $\rho$ for the full fixed-point value and write $\rho_\star$
for the leading-order root below.
Divide the second line of \eqref{eq:joint-leading-reduction} by $m$ and
substitute the first line. Since $m$ is bounded away from zero and
$0<\rho<1$, this gives
\begin{equation*}
(1-\rho)(S+v_\ell^2\rho)=rv_\ell^2\rho+O(\omega).
\end{equation*}
Using $S=s^2+\lambda$ and rearranging, we obtain
\begin{equation}
v_\ell^2\rho^2+[s^2+v_\ell^2(r-1)]\rho-s^2=O(\omega+\lambda).
\label{eq:joint-psi-rho-error}
\end{equation}
The corresponding leading-order equation is
\begin{equation}
v_\ell^2\rho_\star^2+[s^2+v_\ell^2(r-1)]\rho_\star-s^2=0.
\label{eq:joint-psi-rho}
\end{equation}
Its left-hand side is negative at $\rho_\star=0$ and equals $rv_\ell^2>0$ at $\rho_\star=1$,
so it has a unique root $\rho_\star\in(0,1)$.

Subtracting \eqref{eq:joint-psi-rho} from
\eqref{eq:joint-psi-rho-error} gives
\begin{equation*}
(\rho-\rho_\star)
\bigl[v_\ell^2(\rho_\star+\rho)+s^2+v_\ell^2(r-1)\bigr]
=O(\omega+\lambda).
\end{equation*}
Dividing \eqref{eq:joint-psi-rho} by $\rho_\star>0$ yields
\begin{equation*}
v_\ell^2\rho_\star+s^2+v_\ell^2(r-1)=\frac{s^2}{\rho_\star}.
\end{equation*}
Hence the bracket above equals $v_\ell^2\rho+s^2/\rho_\star$, which is at
least $s^2$ because $\rho>0$ and $0<\rho_\star<1$. Therefore,
\begin{equation}
s^2|\rho-\rho_\star|
\le\left|(\rho-\rho_\star)
\left(v_\ell^2\rho+\frac{s^2}{\rho_\star}\right)\right|
=O(\omega+\lambda).
\label{eq:joint-root-stability}
\end{equation}
Thus $|\rho-\rho_\star|=O(\omega+\lambda)$ uniformly in $r>0$.

\paragraph{Training loss.}
The matrix in \eqref{eq:Ugauss} is bounded below by $s^2I_p$, so
\eqref{eq:resolventderiv} gives $0\preceq\partial_zR\preceq R/S$ at
$z=-\lambda$. Also, $D\ge\psi_pm\mu_1^2h_t$ in
\eqref{eq:appendix-denominator} gives $\Kbar\le1/(\mu_1^2h_t)$.
Applying the resolvent bound to the trace \eqref{eq:Kqz} and taking its limit yields
\begin{equation}
0\le\partial_z\Kbar\le\frac{\Kbar}{S}\le\frac1{S h_t\mu_1^2}.
\label{eq:joint-ridge-bound}
\end{equation}
Thus the ridge term in \eqref{eq:appendix-train-loss} is $O(\lambda)$.
By \eqref{eq:ell-positive} and \eqref{eq:joint-root-stability}, replacing
the full fixed-point value $\rho$ by $\rho_\star$ in
$\mathcal M[1/\ell(\varsigma)]$ changes its value
by $O(\omega+\lambda)$. Substituting the last line of
\eqref{eq:joint-leading-reduction} and \eqref{eq:joint-ridge-bound} into
\eqref{eq:appendix-train-loss} therefore gives
\begin{align}
\lim\E[\Ltrain(\Ahat)]
&=R_{\mathrm{train}}(r)+O(\psi_p^{-1}+\psi_n^{-1}+\lambda),
\label{eq:joint-train-error}\\
R_{\mathrm{train}}(r)
&=\mathcal M\Big[\frac{a_t^2\rho_\star\varsigma}{h_t+a_t^2\rho_\star\varsigma}\Big],
\label{eq:joint-train-leading}
\end{align}
uniformly in $r>0$, where $\rho_\star$ in \eqref{eq:joint-train-leading} denotes
the unique root of \eqref{eq:joint-psi-rho} in $(0,1)$.
This proves Proposition~\ref{thm:train}.

\subsection{The scalar \texorpdfstring{$\rho$}{rho}}
\label{sec:app-rho}
In the Gaussian-equivalent model of Lemma~\ref{lem:Ut}, \eqref{eq:Rgauss} separates the
noise-averaged training features $\Gt$ from the remaining terms:
\begin{equation*}
R(\mathbf q,z)^{-1}=M(\mathbf q,z)+\frac{\Gt\Gt\T}{n},\qquad
M(\mathbf q,z):=(s^2-z)\Id{p}+\Wop\Cone(\mathbf q)\WopT .
\end{equation*}
Solve $\Lpen(\mathbf q,z)X=E_4$ in \eqref{eq:pencil}, where $E_4$ has $\Id{n}$ in the
fourth block and zeros elsewhere, eliminating the third, fifth, sixth, second, and first block rows in turn. This gives
\begin{equation*}
\big[\Lpen(\mathbf q,z)^{-1}\big]^{44}
=\Big(\Id{n}+\frac1n\Gt\T M(\mathbf q,z)^{-1}\Gt\Big)^{-1},
\end{equation*}
whose normalized trace has limit $\rho$ by \eqref{eq:scalars}. At
$(\mathbf q,z)=(\mathbf0,-\lambda)$, where $R(\mathbf0,-\lambda)=Q$ and hence
$Q^{-1}=M+\Gt\Gt\T/n$ with $M:=M(\mathbf0,-\lambda)$, set
$H:=n^{-1}\Gt\T Q\,\Gt\in\R^{n\times n}$. The Woodbury identity gives
$\Id{n}-H=(\Id{n}+n^{-1}\Gt\T M^{-1}\Gt)^{-1}$, so
\begin{equation}
\rho=\lim\frac1n\tr(\Id{n}-H)=1-\lim\frac1n\tr H .
\end{equation}
By Lemma~\ref{lem:Ut}, the term of \eqref{eq:lossquadratic} quadratic in $A$ equals
$\frac{h_t}{\dx p}\big[\frac1n\|A\Gt\|_F^2+\tr(AMA\T)\big]$; the other terms are linear in
$A$ or constant and do not affect the hat matrix. This is ridge regression with design
matrix $\Gt\T/\sqrt n$, whose rows are the training samples, and the generalized
(Tikhonov) penalty $\tr(AMA\T)$, where $M$ is the feature covariance induced by the DSM
noise plus the ridge term; its hat matrix is $H$. If
$\kappa_1,\dots,\kappa_n$ are the eigenvalues of $n^{-1}\Gt\T M^{-1}\Gt$, then
$\tr H=\sum_j\kappa_j/(1+\kappa_j)$ is the effective number of degrees of freedom of this
regression. Thus $1-\rho$ is the effective fraction of the $n$ training samples fitted by
the features, and $\rho$ is the fraction left unfitted.

For a linear activation, fitting is limited by the rank of the features: $v_\ell=0$ gives
$\Gt=\mu_1L$, and $L$ in \eqref{eq:feature} has rank at most $\dop$, so
$\tr H<\operatorname{rank}\Gt\le\dop$ and $1-\rho\le\lim\dop/n=(1+\psi_y)/\psi_n$ for
every width. For $v_\ell>0$, the component $v_\ell\Omega$ gives each training sample its own
direction, since $\Omega\T\Omega/p$ is close to $\Id{n}$ for $p\gg n$; this is what allows
$1-\rho\to1$. Finally, the root $\rho_\star$ of \eqref{eq:joint-psi-rho} depends only on
$r$, $s^2$, and $v_\ell^2$, so the condition enters $R_{\mathrm{train}}$ only through
$\varsigma$.

\section{Training loss and the large-width limit at a fixed sample ratio}
\label{sec:app-interpretation}
This section relates the training loss to $G_x$
(Appendix~\ref{sec:app-score-form}), proves the large-width limits at a fixed sample ratio,
and evaluates $\RE$, $\PV$ and $\ME$ at several $t$ (Appendix~\ref{sec:app-fixed-n-pv}).
Throughout, fix $t>0$ and $s^2,v_\ell^2,\mu_1^2>0$. Unless stated otherwise, the identities
below are equivalences $\simeq$ in the sense of Appendix~\ref{sec:proofs}, obtained from the
high-dimensional equivalents of Appendix~\ref{sec:hermite} as
in~\eqref{eq:red-train}--\eqref{eq:red-Gx-traces}, and an inequality between such quantities
is an inequality between their limits $\lim\E[\,\cdot\,]$.

\subsection{Training loss and \texorpdfstring{$G_x$}{G\_x}}
\label{sec:app-score-form}
This subsection relates the training loss to $G_x$ and bounds the
linear part of $\PV$ in Lemma~\ref{lem:main-decomp} by it. Recall from~\eqref{eq:main-Gx} the
Jacobian of the linear component $x_{\Ahat}^{\mathrm{lin}}=G_xx_t+G_yy$ of the denoiser,
\begin{equation*}
G_x=\nabla_{x_t}x_{\Ahat}^{\mathrm{lin}}=\frac1{a_t}\Big(I+h_t\mu_1\frac{\Ahat}{\sqrt p}\Wa\Big)
=\frac1{a_t}\Big(I-\sqrt{h_t}\,\mu_1V\T Q\Wa\Big),
\end{equation*}
where the last equality substitutes the minimizer~\eqref{eq:minimizer}, with $V$ and $Q$ as in
Lemma~\ref{lem:traceid}.

The following identity is exact at finite dimension.

\begin{lemma}
\label{lem:stein-train}
Let $\act$ be Lipschitz, and write
$x_A=(x_t+h_ts_A(x_t,y))/a_t$. Then
\begin{equation*}
\Lridge(\Ahat)=\Ltrain(\Ahat)+h_t\lambda\Anorm
=\frac{a_t}{\dx n}\sum_{i=1}^n\E_\epsilon
\tr\nabla_{x_t}x_{\Ahat}\big(x_{i,t}(\epsilon),y_i\big).
\end{equation*}
\end{lemma}

\begin{proof}
Write $\varphi_i:=\act(\Wop[x_{i,t}(\epsilon);\,\kappa_dy_i])$, so that
$s_A(x_{i,t}(\epsilon),y_i)=A\varphi_i/\sqrt p$, and set $s_i:=\Ahat\varphi_i/\sqrt p$ and
$r_i:=\sqrt{h_t}\,s_i+\epsilon$. Since $\varphi_i$ does not depend on $A$, differentiating
\eqref{eq:loss} gives
\begin{equation*}
\nabla\Lridge(A)
=\frac{2\sqrt{h_t}}{\dx n\sqrt p}\sum_{i=1}^n\E_\epsilon\Big[\Big(\sqrt{h_t}\,\frac{A\varphi_i}{\sqrt p}+\epsilon\Big)\varphi_i\T\Big]
+\frac{2h_t\lambda}{\dx p}A .
\end{equation*}
This vanishes at $A=\Ahat$. Taking $\tr(\Ahat\T\,\cdot\,)$ and using
$\tr(\Ahat\T r_i\varphi_i\T)/\sqrt p=r_i\T s_i$ and $\|\Ahat\|_F^2/(\dx p)=\Anorm$,
\begin{equation*}
0=\tr\big(\Ahat\T\nabla\Lridge(\Ahat)\big)
=\frac{2}{\dx n}\sum_{i=1}^n\E_\epsilon\big[r_i\T\sqrt{h_t}\,s_i\big]+2h_t\lambda\Anorm .
\end{equation*}
Since $\Ltrain(\Ahat)=\frac{1}{\dx n}\sum_{i=1}^n\E_\epsilon\|r_i\|^2$, splitting
$\|r_i\|^2=r_i\T\sqrt{h_t}\,s_i+r_i\T\epsilon$ and using $\E_\epsilon\|\epsilon\|^2=\dx$ gives
\begin{equation*}
\Ltrain(\Ahat)+h_t\lambda\Anorm
=\frac{1}{\dx n}\sum_{i=1}^n\E_\epsilon\big[r_i\T\epsilon\big]
=1+\frac{\sqrt{h_t}}{\dx n}\sum_{i=1}^n\E_\epsilon\big[s_i\T\epsilon\big].
\end{equation*}
Since $x_{i,t}(\epsilon)=a_tx_i+\sqrt{h_t}\epsilon$, one has
$\nabla_\epsilon s_i=\sqrt{h_t}\,\nabla_{x_t}s_{\Ahat}(x_{i,t}(\epsilon),y_i)$, so Stein's identity gives
\begin{equation*}
\E_\epsilon\big[s_i\T\epsilon\big]=\E_\epsilon\tr\nabla_\epsilon s_i
=\sqrt{h_t}\,\E_\epsilon\tr\nabla_{x_t}s_{\Ahat}(x_{i,t}(\epsilon),y_i).
\end{equation*}
Finally, $a_t\nabla_{x_t}x_{\Ahat}=I+h_t\nabla_{x_t}s_{\Ahat}$ and $\tr I=\dx$, so
\begin{equation*}
\Ltrain(\Ahat)+h_t\lambda\Anorm
=1+\frac{h_t}{\dx n}\sum_{i=1}^n\E_\epsilon\tr\nabla_{x_t}s_{\Ahat}(x_{i,t}(\epsilon),y_i)
=\frac{a_t}{\dx n}\sum_{i=1}^n\E_\epsilon\tr\nabla_{x_t}x_{\Ahat}(x_{i,t}(\epsilon),y_i).
\qedhere
\end{equation*}
\end{proof}

\begin{lemma}[Linear part of the prediction variance and training loss]
\label{lem:pv-train-bound}
In the proportional limit, with each quantity standing for its limit $\lim\E[\,\cdot\,]$,
\begin{equation*}
\frac{\lambda_{\min}(\Sigma_t)}{a_t^2}\,\Lridge(\Ahat)^2
\;\le\;\frac1\dx\tr\big(G_x\Sigma_tG_x\T\big)\;\le\;\frac{\Lridge(\Ahat)}{a_t^2},
\qquad
\RA\le\Lridge(\Ahat).
\end{equation*}
\end{lemma}

\begin{proof}
Substituting $V\simeq\sqrt{h_t}\mu_1\Wa$ (Lemma~\ref{lem:Vt}) into $G_x$ gives
\eqref{eq:red-Gx}, and replacing $U$ by its equivalent~\eqref{eq:Ugauss} turns $Q$ into
\begin{equation*}
\tilde Q:=\big(P+h_t\mu_1^2\Wa\WaT\big)^{-1},\qquad P:=\Gt\Gt\T/n+(s^2+\lambda)\Id{p}\succ0 ,
\end{equation*}
with $\Gt$ from~\eqref{eq:Gtdef}; this is the matrix written $Q$ in
Appendix~\ref{sec:app-rho}. Thus $G_x$ is equivalent to
\begin{equation*}
\tilde G_x:=\frac1{a_t}\Big(I-h_t\mu_1^2\WaT\tilde Q\Wa\Big).
\end{equation*}
By the Woodbury identity,
\begin{equation*}
h_t\mu_1^2\WaT\tilde Q\Wa
=h_t\mu_1^2\WaT\big(P+h_t\mu_1^2\Wa\WaT\big)^{-1}\Wa
=I-\big(I+h_t\mu_1^2\WaT P^{-1}\Wa\big)^{-1},
\end{equation*}
which lies between $0$ and $I$ since $\WaT P^{-1}\Wa\succeq0$. Hence, exactly,
\begin{equation}
0\preceq \tilde G_x\preceq a_t^{-1}I .
\label{eq:Gx-bounds}
\end{equation}
In Lemma~\ref{lem:main-decomp}, the linear part of $\PV$, $\RA$, and $\Ltrain$ involve $G_x$
only through the normalized traces
$\dx^{-1}\tr G_x$, $\dx^{-1}\tr(\Sigma G_x)$, and $\dx^{-1}\tr(G_x\Sigma_tG_x\T)$, which are
equivalent to the corresponding traces of $\tilde G_x$, as in~\eqref{eq:red-Gx-traces}.
Replace $G_x$ by $\tilde G_x$ in the traces. By Lemma~\ref{lem:main-decomp},
$\dx^{-1}\tr\tilde G_x\simeq\Lridge(\Ahat)/a_t$, and by
\eqref{eq:Gx-bounds} the eigenvalues $g_j$ of $\tilde G_x$ lie in $[0,a_t^{-1}]$. Hence
$\dx^{-1}\sum_jg_j^2\le a_t^{-1}\dx^{-1}\sum_jg_j$ and
$\dx^{-1}\sum_jg_j^2\ge(\dx^{-1}\sum_jg_j)^2$. The claim follows from
$\lambda_{\min}(\Sigma_t)\,\tilde G_x\tilde G_x\T\preceq \tilde G_x\Sigma_t\tilde G_x\T$ in trace,
$\|\Sigma_t\|_{\mathrm{op}}\le a_t^2+h_t=1$, and $\|\Sigma\|_{\mathrm{op}}\le1$.
\end{proof}

\subsection{Large width at a fixed sample ratio}
\label{sec:app-fixed-n-pv}
This subsection proves Proposition~\ref{prop:main-fixed-n}
(Propositions~\ref{prop:readout-decay}--\ref{prop:ridge-me}) and evaluates $\RE$, $\PV$
and $\ME$ at several $t$ (Figures~\ref{fig:re-pv-t} and~\ref{fig:me-t}).

\begin{proposition}[Large width at a fixed sample ratio]
\label{prop:readout-decay}
Fix $\psi_n>0$, $\psi_y\in(0,1]$, $c\in[0,1]$, and $\lambda>0$, and let $\psi_p\to\infty$.
Then $\Ltrain\to0$ and $\Anorm\to0$, so the nonlinear part of $\ME$ vanishes, and
for the proportional-limit quantities of Theorem~\ref{thm:appendix-final},
\begin{equation}
\RA\longrightarrow0,\qquad
\RE\longrightarrow\upsilon_\Sigma^2,\qquad
\PV\longrightarrow0.
\label{eq:fixed-n-pv-limit}
\end{equation}
In particular, $\dx^{-1}\|G_x\|_F^2\to0$ in the same sense.
\end{proposition}

\begin{proof}
Keep $\mathbf q=\mathbf0$ and let $z\in[-\lambda,-\lambda/2]$, so that
$S:=s^2-z\ge s^2+\lambda/2>0$. The argument of~\eqref{eq:joint-m-bounds} applies with this
$S$, so $m$ is bounded away from zero uniformly in $z$, and
by~\eqref{eq:appendix-fixed-rho}
\begin{equation*}
0<\rho\le\frac{\psi_n}{\psi_pmv_\ell^2}=O(\psi_p^{-1}).
\end{equation*}
The function $\Kbar$ in~\eqref{eq:appendix-readouts} is
\begin{equation*}
\Kbar=\mathcal M\Big[\frac1X\Big],\qquad
X:=\frac2{\psi_pm}+\ell(\varsigma)+a_t^2\frac\beta\delta(1-\varsigma),
\end{equation*}
where $\ell(\varsigma)=\mu_1^2(h_t+a_t^2\rho\varsigma)$ and, by~\eqref{eq:appendix-beta-delta},
$\beta/\delta=\mu_1^2\rho/\delta\le\mu_1^2\rho$. Since $\varsigma\in[0,1]$, this gives
$0\le X-h_t\mu_1^2\le2/(\psi_pm)+a_t^2\mu_1^2\rho=O(\psi_p^{-1})$, and hence
\begin{equation*}
\Big|\Kbar(z)-\frac1{h_t\mu_1^2}\Big|
\le\mathcal M\Big[\frac{X-h_t\mu_1^2}{(h_t\mu_1^2)^2}\Big]=O(\psi_p^{-1})
\end{equation*}
uniformly in $z\in[-\lambda,-\lambda/2]$. In particular, $h_t\mu_1^2\Kbar\to1$ at $z=-\lambda$.

Next, $U\succeq0$, so for $z<0$, \eqref{eq:resolventderiv} gives
$\partial_zK=\dx^{-1}\tr(\WaT R^2\Wa)\ge0$ and $\partial_z^2K=2\dx^{-1}\tr(\WaT R^3\Wa)\ge0$.
Thus $K(\mathbf0,\cdot)$ is nondecreasing and convex on $(-\infty,0)$, and so is its limit
$\Kbar$. By convexity,
\begin{equation*}
0\le\partial_z\Kbar(-\lambda)\le\frac{\Kbar(-\lambda/2)-\Kbar(-\lambda)}{\lambda/2}
=O(\psi_p^{-1}).
\end{equation*}
By \eqref{eq:appendix-train-loss} and
\eqref{eq:appendix-N}, this gives $\Ltrain\to0$ and $\Anorm\to0$.
Hence $\Lridge(\Ahat)\to0$, and Lemma~\ref{lem:pv-train-bound} sends the linear part
of $\PV$ and $\RA$ to zero, where $\RA\ge0$ by \eqref{eq:Gx-bounds}. Since
$\Sigma_t=a_t^2\Sigma+h_tI\succeq h_tI$, $\dx^{-1}\|G_x\|_F^2\le h_t^{-1}\dx^{-1}\tr(G_x\Sigma_tG_x\T)$,
so $\dx^{-1}\|G_x\|_F^2\to0$ as well. The nonlinear parts of
$\PV$ and $\ME$ are multiples of $\Anorm$ by Lemma~\ref{lem:main-decomp} and vanish, so
$\PV\to0$, and
$\RE=\upsilon_\Sigma^2+\PV-2\RA\to\upsilon_\Sigma^2$.
\end{proof}

Figure~\ref{fig:re-pv-t} evaluates $\RE$ and $\PV$ of Theorem~\ref{thm:appendix-final}
against $\psi_p$ at various $t$. At every $t$, $\PV\to0$ as $\psi_p\to\infty$, which means $x_A-\E[x_A\mid y]\to0$, so only $x-By$ remains
in~\eqref{eq:re}, and $\RE\to\dx^{-1}\E\|x-By\|^2=\upsilon_\Sigma^2$ at every $t$.
Since $x_{\Ahat}=x_t/a_t+(h_t/a_t)s_{\Ahat}$, where $x_t/a_t=x+(\sqrt{h_t}/a_t)\epsilon$, and
$\sqrt{h_t}s_{\Ahat}$ predicts $-\epsilon$ in at most $p$ of the $\dx$ directions, the correction
removes a vanishing fraction of this noise as $\psi_p\to0$, so $\RE\to h_t/a_t^2$, the
residual error of $x_t/a_t$.
$\RE$ thus rises with width at $t=0.01$, where
$h_t/a_t^2<\upsilon_\Sigma^2$, and falls at $t\ge0.1$, where $h_t/a_t^2>\upsilon_\Sigma^2$.

\begin{figure}[t]
  \centering
  \includegraphics[width=\textwidth]{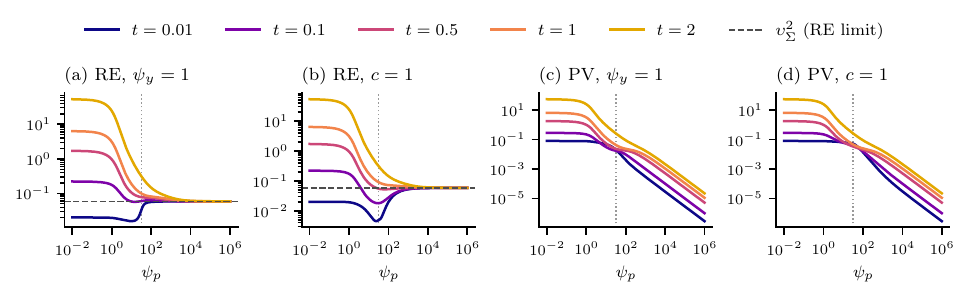}
  \caption{$\RE$ (a, b) and $\PV$ (c, d) against width for $t\in\{0.01,0.1,0.5,1,2\}$, in
  Gaussian denoising (a, c; $\psi_y=1$) and super-resolution (b, d; $c=1$) at
  $\mathrm{SNR}=16$, $\psi_n=32$, $\act=\tanh$ and $\lambda=10^{-5}$. Dashed: $\upsilon_\Sigma^2$,
  the limit of $\RE$ in~\eqref{eq:fixed-n-pv-limit}; dotted: $\psi_p=\psi_n$.}
  \label{fig:re-pv-t}
\end{figure}

\begin{proposition}[Conditional-mean error at large width]
\label{prop:fixed-n-me}
In the setting of Proposition~\ref{prop:readout-decay}, let $\gamma:=\psi_y/\psi_n$ and
$\lambda_{\mathrm{eff}}:=2\psi_yv_\ell^2/(\psi_n\mu_1^2)$, let $\delta_\infty>1$ be the root of
\begin{equation}
\lambda_{\mathrm{eff}}\delta^2-(1-\gamma+\lambda_{\mathrm{eff}})\delta-\gamma=0,
\label{eq:me-lim-root}
\end{equation}
and let $\kappa:=\psi_y(\delta_\infty-1)^2/(\psi_n\delta_\infty^2)$. Then
\begin{equation}
\ME\longrightarrow\frac{\kappa\,\upsilon_\Sigma^2+c^2\psi_y/\delta_\infty^2}{1-\kappa}.
\label{eq:fixed-n-me-limit}
\end{equation}
\end{proposition}

\begin{proof}
Throughout, $\mathbf q=(q_{\mathrm m},0)$, $z=-\lambda$, and derivatives in $q_{\mathrm m}$ are
taken at $q_{\mathrm m}=0$.

\textbf{Limiting equations.}
By Proposition~\ref{prop:readout-decay}, the nonlinear part of $\ME$ vanishes, so by
Theorem~\ref{thm:appendix-final} we need the limit of
$-(h_t^2\mu_1^4/a_t^2)\,\partial_{q_{\mathrm m}}\Kbar$. By~\eqref{eq:appendix-readouts}
and~\eqref{eq:appendix-denominator},
\begin{equation*}
\Kbar=\mathcal M\Big[\frac1{X(\varsigma)}\Big],\qquad
X(\varsigma):=\frac{D(\varsigma)}{\psi_pm}=\frac2{\psi_pm}+\ell(\varsigma)+a_t^2\frac\beta\delta(1-\varsigma).
\end{equation*}
By~\eqref{eq:joint-m-bounds}, $\psi_pm\to\infty$ as $\psi_p\to\infty$, while $\rho\to0$ at rate
$1/(\psi_pm)$, as shown below. To work with variables that have finite limits, define
\begin{equation*}
\varepsilon:=\frac1{\psi_pm},\qquad u:=\frac\rho\varepsilon=\psi_pm\rho,\qquad
\tilde q:=\frac{q_{\mathrm m}}\varepsilon=\psi_pm\,q_{\mathrm m}.
\end{equation*}
Then, by~\eqref{eq:appendix-beta-delta} and~\eqref{eq:appendix-ell} at $q_{\mathrm r}=0$,
\begin{equation}
\beta=\varepsilon(\tilde q+\mu_1^2u),\qquad
2\psi_y(\delta-1)=\tilde q+\mu_1^2u,\qquad
\ell(\varsigma)=h_t\mu_1^2+\varepsilon a_t^2\mu_1^2u\varsigma.
\label{eq:me-lim-delta}
\end{equation}
Hence $\beta/\delta=2\psi_y\varepsilon(1-1/\delta)$ and
\begin{equation*}
X(\varsigma)=h_t\mu_1^2+\varepsilon Y(\varsigma),\qquad
Y(\varsigma):=2+a_t^2\mu_1^2u\varsigma+2a_t^2\psi_y\Big(1-\frac1\delta\Big)(1-\varsigma)\ge0.
\end{equation*}
Multiplying~\eqref{eq:appendix-fixed-rho} by $\rho=\varepsilon u$,
\begin{equation*}
1=\varepsilon u+\frac u{\psi_n}\Big[v_\ell^2+\frac{\mu_1^2}{2\delta}+T\Big],\qquad
T:=a_t^2\mu_1^2\mathcal M\Big[\frac{\varsigma+(1-\varsigma)/\delta^2}{D(\varsigma)}\Big].
\end{equation*}
Here $u\le\psi_n/v_\ell^2$, and $T=O(\varepsilon)$ since $D=X/\varepsilon\ge h_t\mu_1^2/\varepsilon$.
Since $\varepsilon\to0$, the limiting equation is
\begin{equation}
u\Big(v_\ell^2+\frac{\mu_1^2}{2\delta}\Big)=\psi_n.
\label{eq:me-lim-u}
\end{equation}

\textbf{The root $\delta_\infty$.}
At $\tilde q=0$, substituting $\mu_1^2u=2\psi_y(\delta-1)$ from~\eqref{eq:me-lim-delta}
into~\eqref{eq:me-lim-u} gives $\psi_y(\delta-1)(2v_\ell^2+\mu_1^2/\delta)=\mu_1^2\psi_n$.
Multiplying by $\delta/(\mu_1^2\psi_n)$,
\begin{equation*}
\lambda_{\mathrm{eff}}\delta^2-(1-\gamma+\lambda_{\mathrm{eff}})\delta-\gamma=0.
\end{equation*}
The left side equals $-1$ at $\delta=1$, so this has a unique root $\delta_\infty>1$.

\textbf{The derivative of $\Kbar$.}
In the new variables, every term on the right of~\eqref{eq:appendix-fixed-m} other than $s^2+\lambda$ is
$\varepsilon$ times a smooth function of $(u,\delta,\tilde q,\varepsilon)$. Substituting $\delta$
from~\eqref{eq:me-lim-delta} into~\eqref{eq:me-lim-u} gives one equation for $u$, whose left side
has $u$-derivative $v_\ell^2+\mu_1^2/(2\delta_\infty^2)>0$ at $\tilde q=0$, using
$\mu_1^2u=2\psi_y(\delta_\infty-1)$. By the implicit function theorem, $u$ and $\delta$ are thus
smooth functions of $(\tilde q,\varepsilon)$ near $(0,0)$. Moreover, by~\eqref{eq:me-lim-u}, this
derivative times $u/\psi_n$ equals $1-\kappa$, so $0<\kappa<1$. Since
$\partial_{q_{\mathrm m}}\tilde q=1/\varepsilon$, this gives $\partial_{q_{\mathrm m}}m=O(1)$, so
$\partial_{q_{\mathrm m}}\varepsilon=O(\varepsilon)$ and
$\varepsilon\,\partial_{q_{\mathrm m}}Y=\partial_{\tilde q}Y+O(\varepsilon)$. Hence
\begin{equation*}
-\frac{h_t^2\mu_1^4}{a_t^2}\,\partial_{q_{\mathrm m}}\Kbar
=\frac{h_t^2\mu_1^4}{a_t^2}\,
\mathcal M\Big[\frac{(\partial_{q_{\mathrm m}}\varepsilon)Y+\varepsilon\,\partial_{q_{\mathrm m}}Y}{X^2}\Big]
\longrightarrow\frac1{a_t^2}\,\mathcal M[\partial_{\tilde q}Y],
\end{equation*}
where, in the limit, $u$ and $\delta$ in $Y$ are functions of $\tilde q$
through~\eqref{eq:me-lim-delta}--\eqref{eq:me-lim-u}, and all derivatives are evaluated at
$\tilde q=0$, where $\delta=\delta_\infty$.

\textbf{The derivatives of $\delta$ and $u$.}
Differentiating~\eqref{eq:me-lim-delta} and~\eqref{eq:me-lim-u} in $\tilde q$,
\begin{equation*}
2\psi_y\partial_{\tilde q}\delta=1+\mu_1^2\partial_{\tilde q}u,\qquad
\partial_{\tilde q}u\Big(v_\ell^2+\frac{\mu_1^2}{2\delta}\Big)=\frac{\mu_1^2u}{2\delta^2}\,\partial_{\tilde q}\delta.
\end{equation*}
In the second equation, replace $v_\ell^2+\mu_1^2/(2\delta)$ by $\psi_n/u$ using~\eqref{eq:me-lim-u},
and then $\mu_1^2u$ by $2\psi_y(\delta_\infty-1)$ using~\eqref{eq:me-lim-delta}:
\begin{equation*}
\mu_1^2\partial_{\tilde q}u=\frac{(\mu_1^2u)^2}{2\psi_n\delta_\infty^2}\,\partial_{\tilde q}\delta
=\frac{2\psi_y^2(\delta_\infty-1)^2}{\psi_n\delta_\infty^2}\,\partial_{\tilde q}\delta
=2\psi_y\kappa\,\partial_{\tilde q}\delta.
\end{equation*}
With the first equation, this gives
\begin{equation*}
2\psi_y\partial_{\tilde q}\delta=\frac1{1-\kappa},\qquad
\mu_1^2\partial_{\tilde q}u=\frac\kappa{1-\kappa}.
\end{equation*}

\textbf{The limit.}
Differentiating $Y$ and using these derivatives,
\begin{equation*}
\partial_{\tilde q}Y
=a_t^2\varsigma\,\mu_1^2\partial_{\tilde q}u
+a_t^2(1-\varsigma)\,\frac{2\psi_y\partial_{\tilde q}\delta}{\delta_\infty^2}
=\frac{a_t^2}{1-\kappa}\Big(\kappa\varsigma+\frac{1-\varsigma}{\delta_\infty^2}\Big).
\end{equation*}
Since $\mathcal M[\varsigma]=\upsilon_\Sigma^2$ and $\mathcal M[1-\varsigma]=c^2\psi_y$
by~\eqref{eq:direction-average},
\begin{equation*}
\frac1{a_t^2}\,\mathcal M[\partial_{\tilde q}Y]
=\frac{\kappa\,\upsilon_\Sigma^2+c^2\psi_y/\delta_\infty^2}{1-\kappa},
\end{equation*}
which is~\eqref{eq:fixed-n-me-limit}.
\end{proof}

The next proposition shows that the limit~\eqref{eq:fixed-n-me-limit} equals the risk of ridge
regression of $x$ on $y$~\citep{hastie2022surprises}; its proof derives this risk in closed form
for completeness.

\begin{proposition}[Ridge regression of $x$ on $y$]
\label{prop:ridge-me}
Let $\widehat B_\lambda:=\arg\min_{B'}\frac1n\|X-B'Y\|_F^2+\lambda\|B'\|_F^2$ be the ridge regression
of $x$ on $y$ from the $n$ training pairs and $\mathcal R(\lambda):=\lim\dx^{-1}\E\|\widehat B_\lambda-B\|_F^2$
its risk. Then the limit in~\eqref{eq:fixed-n-me-limit} equals $\mathcal R(\lambda_{\mathrm{eff}})$, with
$\lambda_{\mathrm{eff}}$ as in Proposition~\ref{prop:fixed-n-me}.
\end{proposition}

\begin{proof}
Let $S:=YY\T/n$ and $\theta_k:=\lim\dy^{-1}\tr(S+\lambda I)^{-k}$ for $k=1,2$, with $\gamma$ as
in Proposition~\ref{prop:fixed-n-me}.
Since $\widehat B_\lambda=\frac1nXY\T(S+\lambda I)^{-1}$ and $X=BY+\Sigma^{1/2}Z_x$,
\begin{equation*}
\widehat B_\lambda-B=-\lambda B(S+\lambda I)^{-1}+\frac1n\Sigma^{1/2}Z_xY\T(S+\lambda I)^{-1}.
\end{equation*}
The cross term has mean zero. Since $B\T B=c^2I$, $\E[Z_x\T\Sigma Z_x]=\tr\Sigma\,I$, and
$\dx^{-1}\tr\Sigma\to\upsilon_\Sigma^2$, the risk is
\begin{equation*}
\mathcal R(\lambda)=c^2\psi_y\lambda^2\theta_2+\upsilon_\Sigma^2\gamma(\theta_1-\lambda\theta_2).
\end{equation*}
By the Marchenko--Pastur law, $\gamma\lambda\theta_1^2+(1-\gamma+\lambda)\theta_1=1$. The normalized
traces are analytic and bounded in $\lambda>0$, so by Vitali's theorem their derivatives converge
as well, and $\theta_2=-\partial_\lambda\theta_1$; differentiating the Marchenko--Pastur equation
in $\lambda$ gives
$\theta_2=\theta_1(1+\gamma\theta_1)/(1-\gamma+\lambda+2\gamma\lambda\theta_1)$. Set
$\delta:=1/(\lambda\theta_1)>1$ and
the Marchenko--Pastur equation becomes $\lambda\delta^2-(1-\gamma+\lambda)\delta-\gamma=0$, which at
$\lambda=\lambda_{\mathrm{eff}}$ is~\eqref{eq:me-lim-root}, so $\delta=\delta_\infty$. Substituting $\theta_1=1/(\lambda\delta)$ and
$\theta_2=(\lambda\delta+\gamma)/(\lambda^2\delta(\lambda\delta^2+\gamma))$ into $\mathcal R(\lambda)$
and simplifying with the quadratic and $\kappa=\gamma(1-1/\delta_\infty)^2$ shows that
$\mathcal R(\lambda_{\mathrm{eff}})$ equals the limit in~\eqref{eq:fixed-n-me-limit}.
\end{proof}
\begin{figure}[t]
  \centering
  \includegraphics[width=0.65\textwidth]{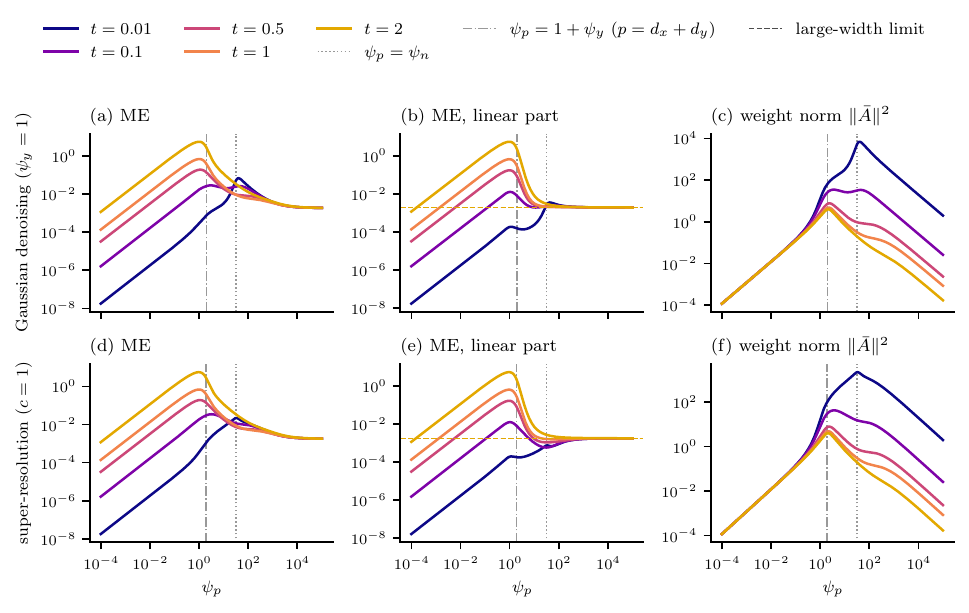}
  \caption{$\ME$ (a, d), its linear part (b, e), and the weight norm $\Anorm$ (c, f), to which the
  nonlinear part of $\ME$ is proportional, against width for
  $t\in\{0.01,0.1,0.5,1,2\}$, with $\psi_y=1$ (top) and $c=1$ (bottom); settings as in
  Figure~\ref{fig:re-pv-t}. Dashed in (b, e): the limit~\eqref{eq:fixed-n-me-limit}; dotted:
  $\psi_p=\psi_n$; dash-dotted: $\psi_p=1+\psi_y$.}
  \label{fig:me-t}
\end{figure}
Figure~\ref{fig:me-t} shows $\ME$, its linear part, and the weight norm $\Anorm$, to which its
nonlinear part is proportional (Lemma~\ref{lem:main-decomp}). We explain the behavior of the two
parts in turn.

\paragraph{Linear part.}
The linear part is $\|\widehat B-B\|_F^2/\dx$ with $\widehat B=a_tG_x B+G_y$. Since
$\E[x_t/a_t\mid y]=By$, $\widehat B-B=(h_t/a_t)\Theta$, where $\Theta y=\E[s_{\Ahat}^{\mathrm{lin}}\mid y]$, which
vanishes for the population score $-\Sigma_t^{-1}\eta_t$.
\begin{itemize}
\item \emph{Small $p$.} As $\psi_p\to0$, $s_{\Ahat}\to0$ and $x_{\Ahat}\to x_t/a_t$, which is unbiased given $y$ but carries the full noise, so $\Theta=0$ and $\ME\to0$ trivially at every $t$. As $p$
grows, $\Theta$ becomes nonzero, and the factor $h_t/a_t$ amplifies it more at larger $t$, so that at
large $t$ the linear part exceeds its large-$\psi_p$ limit.
\item \emph{Large $p$.} As $\psi_p\to\infty$, $G_x\to0$ (Proposition~\ref{prop:readout-decay}) and $\widehat B=G_y$ estimates $B$ from the $n$
training pairs using $y$ alone, so the linear part converges at every $t$ to the
limit~\eqref{eq:fixed-n-me-limit}, the risk of ridge regression of $x$ on $y$
(Proposition~\ref{prop:ridge-me}). This limit depends on $t$ only through $v_\ell^2$, which sets
the ridge $\lambda_{\mathrm{eff}}$; since this ridge is small, the limit varies little with $t$.
\end{itemize}

\paragraph{Nonlinear part.}
At every $t$, the nonlinear part vanishes as $\psi_p\to\infty$
(Proposition~\ref{prop:readout-decay}), so $\ME$ converges to the same limit. The nonlinear part
peaks near $\psi_p\approx\psi_n$ at small $t$ and near $\psi_p\approx1+\psi_y$, that is, near
$p\approx\dx+\dy$, the input dimension, at large $t$. By Lemma~\ref{lem:Ut},
\begin{equation*}
U\simeq\frac1n(\mu_1L+v_\ell\Omega)(\mu_1L+v_\ell\Omega)\T+h_t\mu_1^2\frac{W_1W_1\T}{2\dx}+s^2\Id{p},
\end{equation*}
where $\mu_1L$ and $v_\ell\Omega$ stand for the linear and nonlinear parts of the noise-averaged
training features, and the last two terms come from the diffusion noise. The nonlinear-feature
term $v_\ell^2\Omega\Omega\T/n$ has rank $\min(p,n)$, and the linear-feature term
$\mu_1^2LL\T/n+h_t\mu_1^2W_1W_1\T/(2\dx)$ has rank $\min(p,\dx+\dy)$. Heuristically, as in ridge
regression, $\Anorm$ peaks where one of these terms stops having full rank, and a ridge large
relative to the signal of that term flattens the peak. As $t$ grows, the diffusion noise smooths
the activation and moves variance from $v_\ell^2$ to $s^2$, which acts as a ridge, so the peak at
$p=n$ flattens. Near $p=\dx+\dy\ll n$, the nonlinear-feature term is close to $v_\ell^2\Id{p}$ and adds
to the ridge, which is then $s^2+v_\ell^2=v^2$ at every $t$. The signal of the linear-feature term is
the variance $\mu_1^2\tau/2$ that $x_t-a_tBy$, which contains $\epsilon$, contributes to each
feature. It grows with $t$, and the peak at $p=\dx+\dy$ forms once it exceeds $v^2$.

In the generation experiment of Section~\ref{sec:exp-rfnn}, the conditional error peaks only
near $\psi_p=\psi_n$ (Figure~\ref{fig:rfnn-width}). While we do not analyze the full reverse
process (Remark~\ref{rem:single-step}), we conjecture that this is because the sampler uses the
score, whose error is that of the denoiser multiplied by $a_t/h_t$, so that errors at small $t$
weigh most; at small $t$, the nonlinear part peaks at $p=n$. The large-$t$ peaks, of the linear
part near $\psi_p\approx1$ and of the nonlinear part near $p\approx\dx+\dy$, lie below the
smallest width $\psi_p=4$ of Figure~\ref{fig:rfnn-width}.

\section{Experimental details}
\label{sec:app-exp-details}

\subsection{Theory curves and finite simulations}
\label{sec:app-finite-rfnn-details}

The theory curves in Figure~\ref{fig:width-dependent-errors}
use $\act(u)=\tanh(u)$ in the normalized model of
Section~\ref{sec:rf-score-learning}, the activation of the generation experiment in
Section~\ref{sec:exp-rfnn}.

The dots in Figure~\ref{fig:width-dependent-errors} are finite nonlinear RFNN
simulations of the same model, with $n=32\,d_x$ and the activation, diffusion time and
ridge parameter of the caption, and $d_x=256$ throughout. In the top row $d_y=256$ and
$B=cI$. In the bottom row $c=1$ and $d_y=204$, $240$ and $252$, so that
$\psi_y=51/64$, $60/64$ and $63/64$ as in the generation experiment of
Section~\ref{sec:exp-rfnn}, and $\mathrm{SNR}(\psi_y,1)=d_y/(d_x-d_y)=3.92$, $15$ and $63$;
$U_\star$ selects the first $d_y$ coordinates, and the theory curves use the same $\psi_y$. Each fit uses
$512$ training-noise draws per clean pair, and each test estimate uses $128$ conditions
with $256$ independent draws per condition. Each dot is the mean over five independent realizations of the data,
features and noise. The min--max bars across realizations are smaller than the dots, and the relative standard deviation of $\ME$, $\RE$, $\PV$ and the test loss has median $0.4\%$. The
median relative deviation from Theorem~\ref{thm:main} is below $3\%$ for each of $\ME$,
$\RE$, $\PV$ and $\ME+\RE$ in both rows, where the simulations also reproduce the
non-monotone residual error of super-resolution. The largest deviations are at the widest
models, where $\PV$ is smallest: at $\psi_p=256$ the simulated $\PV$ exceeds
Theorem~\ref{thm:main} by up to $17\%$. These deviations are consistent with finite-size
corrections: at SNR $16$ and $\psi_p=128$ in the top row, with $512$ noise draws per pair at
both dimensions, the simulated $\ME$ exceeds Theorem~\ref{thm:main} by $6.0\%$
and $3.2\%$ at $d_x=128$ and $256$, and $\PV$ by $9.6\%$ and $5.1\%$ (five realizations each), while $\RE$
stays within $0.5\%$; the excess halves when $d_x$ doubles.

\subsection{Random-feature sampler}
\label{sec:app-rfnn-details}

The data of Section~\ref{sec:exp-rfnn} have $\dx=64$ and $n=2048$, so $\psi_n=32$ as in
Figure~\ref{fig:width-dependent-errors}. Gaussian denoising uses $\dy=64$, $B=cI$ and
$\mathrm{SNR}(1,c)\in\{4,16,64\}$; super-resolution uses $c=1$ and $\dy=51$, $60$ and $63$,
so that $\mathrm{SNR}(\psi_y,1)=3.92$, $15$ and $63$. The marginal distribution of $x$,
and hence of $x_t$ at every time, is $\N(0,\Id{\dx})$ throughout, so only the joint
distribution with $y$ changes along the sweep. The sweep starts at $\psi_p=4$
(Section~\ref{sec:exp-rfnn}).
The second-layer weights $A$ of Section~\ref{sec:exp-rfnn} at each diffusion time minimize the
ridge-regularized denoising objective~\eqref{eq:loss} with $\lambda=10^{-5}$ over $64$
Monte Carlo noise draws per training pair, which approximates the noise expectation the theory takes. As in \citet{george2026}, each of the $200$ sampling times uses a
separate instance of the model: $W_1$, $W_2$ and the noise draws are drawn independently for each
time, and $A$ is learned separately. Sampling is the ordinary ancestral chain: $200$ steps of the
$1000$-step linear schedule the CelebA models also use, run from the end of the schedule,
with no clipping of the $x_0$ estimate. The initial state
is drawn from $\mathcal N(0,I)$, which is the exact marginal distribution of
$x_t$ at every time, without conditioning on $y$. Each point in Figure~\ref{fig:rfnn-width} is one realization of the
data, the features and the noise. The activation is $\tanh$ with the
$1/\sqrt2$ input normalization of Section~\ref{sec:rf-score-learning}.
We evaluate $512$ independent test conditions and all $2048$ training
conditions, with $128$ generations per condition. Generated variance is the
unbiased sample variance across these generations, averaged over conditions;
no across-realization uncertainty bars are shown.

\subsection{CelebA training and evaluation}
\label{sec:app-celeba-details}

The images are $32\times32$ grayscale center crops of CelebA. The denoiser is a
conditional $\epsilon$-prediction U-Net of base width $W$ with channel multipliers
$1,2,4$, two residual blocks per level in the encoder and three in the decoder, and attention on the two coarsest levels, as in
the reference architecture~\citep{bonnaire2025}; the condition enters as a second input
channel, nearest-upsampled to $32\times32$ in super-resolution, and the parameter count
grows as $W^2$, while the architecture does not depend on the condition. Pixel intensities
are in $[0,1]$, the unit of $\sigma$ in denoising, and the noise of each training condition
is drawn once and kept fixed. Unless stated otherwise, every model is trained for $200$k steps on $n=32768$
images, with three seeds per setting.
Training uses Adam with learning rate $3\times10^{-4}$ and $\beta=(0.9,0.999)$, batch
$512$ (microbatch $128$), no weight decay, no EMA and no gradient clipping, and the linear
$1000$-step schedule with $\beta$ from $10^{-4}$ to $0.02$; sampling is the full
$1000$-step ancestral chain. The three seeds per setting change the training block, the
initialization, the batch order and the sampling noise. Memorization is measured on
$4096$ generations for training conditions with the $d_1/d_2<1/3$ criterion, and the variance
on $128$ test conditions with $16$ draws each.

\paragraph{Conditional error.}
For a test condition $y$ with original image $x$, let $\hat x^{(1)},\dots,\hat x^{(K)}$ be the
samples generated for $y$, with mean $\bar x$ and per-pixel sample variance $\hat v$, and let
$\hat\mu_y$ and $\hat\Sigma_y$ be the mean and covariance of the generated distribution. The
conditional mean $\mu_y=\E[x\mid y]$ is unknown for CelebA, so the conditional error compares
$\bar x$ with $x$ instead:
\begin{equation*}
\E\|\bar x-x\|^2=\|\hat\mu_y-\mu_y\|^2+\tr\Sigma_y+\frac1K\tr\hat\Sigma_y ,
\end{equation*}
where $\Sigma_y=\Cov(x\mid y)$. The first term is the conditional-mean error, the counterpart of
$\ME$; the second depends only on the data, not on the model; and the third is the variance left
in a mean of $K$ samples. We use $K=64$ draws for each of the $128$ test conditions and report
$\dx^{-1}\|\bar x-x\|^2$ averaged over conditions.

Figures~\ref{fig:celeba-error-appendix} and~\ref{fig:celeba-width-n8192} compare this
error with three others, at $n=32768$ and $n=8192$. The per-sample error, the average of
$\dx^{-1}\|\hat x^{(k)}-x\|^2$ over the $K$ samples, contains the whole variance
$\tr\hat\Sigma_y$ instead of $\tr\hat\Sigma_y/K$, so it decreases with width as the samples for a
test condition collapse. In super-resolution, the pooled error is the per-sample error between the pooled sample and
the pooled original image, which measures only the directions that the condition observes; in
denoising, where the condition covers every pixel, it would coincide with the per-sample error
and is not shown.
Removing the variance entirely gives the debiased error
\begin{equation*}
\frac1K\sum_{k=1}^K\frac{\|\hat x^{(k)}-x\|^2}{\dx}-\hat v
=\frac{\|\bar x-x\|^2}{\dx}-\frac{\hat v}{K},
\end{equation*}
whose expectation is $\dx^{-1}(\|\hat\mu_y-\mu_y\|^2+\tr\Sigma_y)$. The conditional error and the
debiased error nearly coincide, and for the conditions that are memorized both peak near the
onset of memorization.

\begin{figure}[H]
  \centering
  \includegraphics[width=0.85\textwidth]{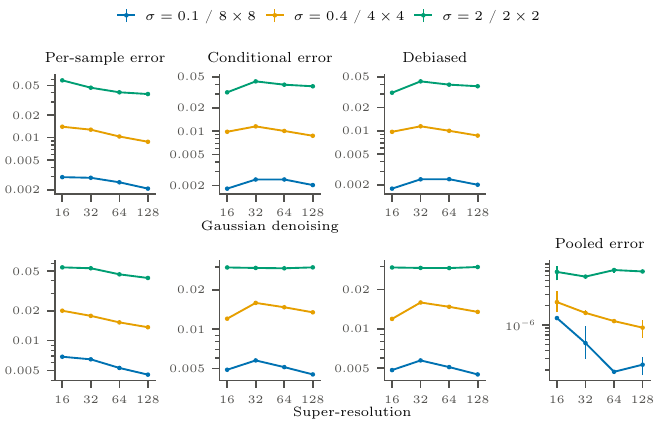}
  \caption{Error metrics at test conditions on CelebA at $n=32768$, as in
  Figure~\ref{fig:celeba-width}: the per-sample error, the conditional error of
  Figure~\ref{fig:celeba-width} (the error of the mean of $K=64$ samples), the debiased error, and
  the pooled error (super-resolution only). Markers are the mean and bars the min--max range over
  three seeds.}
  \label{fig:celeba-error-appendix}
\end{figure}

\section{Additional results}
\label{sec:app-additional-results}

\subsection{Smaller training set}
\label{sec:app-data-sweep}

Figure~\ref{fig:celeba-width-n8192} repeats Figure~\ref{fig:celeba-width} at $n=8192$, a
quarter of the training set; apart from $n$, the setting and the metrics are those of
Appendix~\ref{sec:app-celeba-details}. With fewer training images every transition moves
to a smaller width, and the ordering by condition strength is unchanged. In Gaussian
denoising the memorization fraction at $W=16$ is $0.99$, $0.69$ and $0.18$ for
$\sigma=0.1$, $0.4$ and $2$, against $0.52$, $0.00$ and $0.00$ at $n=32768$, and every
condition memorizes from $W=32$ on. In super-resolution the $8\times8$ and $4\times4$
conditions memorize from $W=32$ on ($0.96$ and $0.45$ at $W=16$), and the $2\times2$
condition reaches $0.75$ at $W=64$ and $1.00$ at $W=128$. The within-condition variance
falls with width in every row, by one to two orders of magnitude from $W=16$ to $W=128$,
except for $\sigma=2$, where it levels off from $W=64$ on, and for the $2\times2$ condition,
where it falls about threefold. For the $2\times2$ condition, the conditional error is still rising at $W=128$,
where memorization is already complete, so its peak, if any, lies beyond the widths considered.

\begin{figure}[H]
  \centering
  \includegraphics[width=\textwidth]{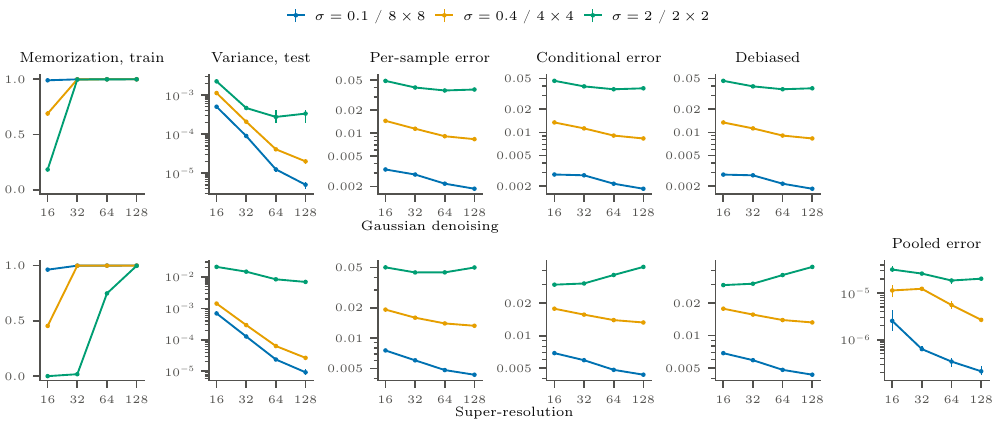}
  \caption{Figure~\ref{fig:celeba-width} at $n=8192$, with the error metrics of
  Figure~\ref{fig:celeba-error-appendix} in the last four columns: full reverse-process
  generation with U-Nets on $32\times32$ CelebA against the width $W\in\{16,32,64,128\}$, every
  model trained for $200$k steps. Markers are the mean and bars the min--max range over three
  seeds.}
  \label{fig:celeba-width-n8192}
\end{figure}

\subsection{Test and training losses over width and condition strength}
\label{sec:app-heatmaps}

Figures~\ref{fig:heatmap-psin8}--\ref{fig:heatmap-psin256} evaluate
Theorem~\ref{thm:main} over the range $10^{-1}\le\mathrm{SNR}\le10^{3}$ of both
specializations, at $t=0.01$, $\act(u)=\tanh(u)$, $\lambda=10^{-5}$ and
$\psi_n\in\{8,32,64,256\}$. In each figure the top row is Gaussian denoising, $\psi_y=1$,
against $\mathrm{SNR}(1,c)$, and the bottom row is super-resolution, $c=1$, against
$\mathrm{SNR}(\psi_y,1)$. The columns are those of Figure~\ref{fig:width-dependent-errors}; the
first three share a logarithmic color scale, the white dashed line is $\psi_p=\psi_n$, and dotted
lines are level sets. In every figure the training loss decreases with width, and its level
sets move to smaller widths as the SNR grows.

\begin{figure}[H]
  \centering
  \includegraphics[width=\textwidth]{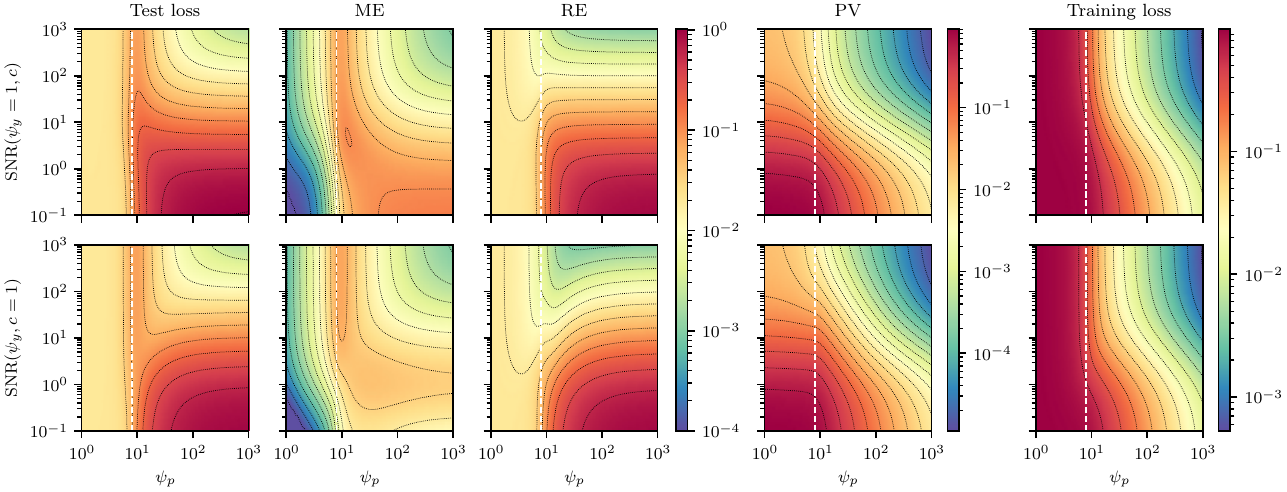}
  \caption{Test loss $(h_t/a_t^2)\Ltest=\ME+\RE$, $\ME$, $\RE$, $\PV$ and $\Ltrain$ of Theorem~\ref{thm:main} at
  $\psi_n=8$; settings as in Appendix~\ref{sec:app-heatmaps}.}
  \label{fig:heatmap-psin8}
\end{figure}

\begin{figure}[H]
  \centering
  \includegraphics[width=\textwidth]{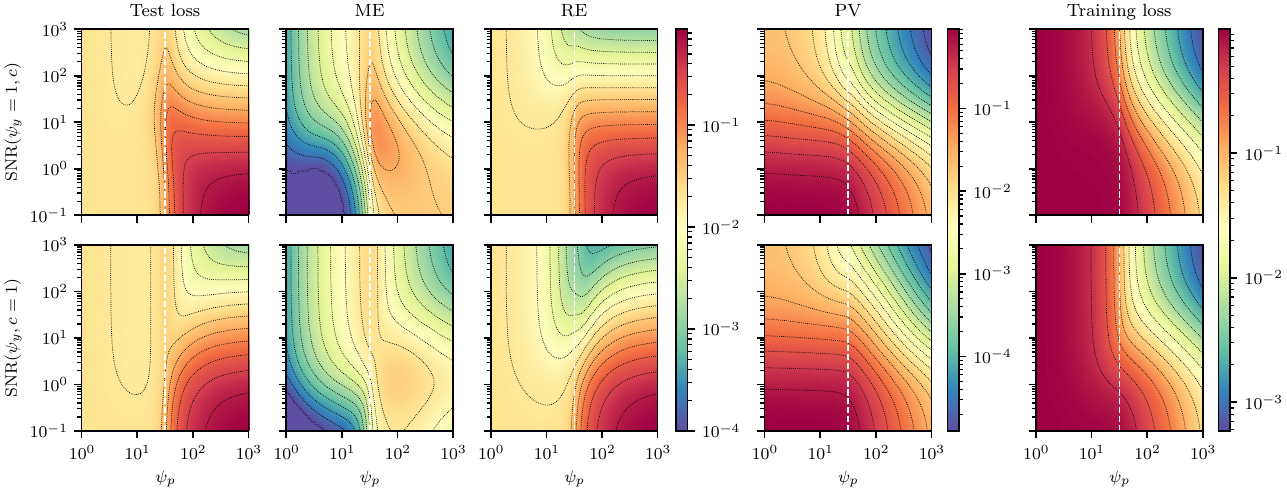}
  \caption{Test loss $(h_t/a_t^2)\Ltest=\ME+\RE$, $\ME$, $\RE$, $\PV$ and $\Ltrain$ of Theorem~\ref{thm:main} at
  $\psi_n=32$, as in Figures~\ref{fig:width-dependent-errors} and~\ref{fig:rfnn-width};
  settings as in Appendix~\ref{sec:app-heatmaps}.}
  \label{fig:heatmap-psin32}
\end{figure}

\begin{figure}[H]
  \centering
  \includegraphics[width=\textwidth]{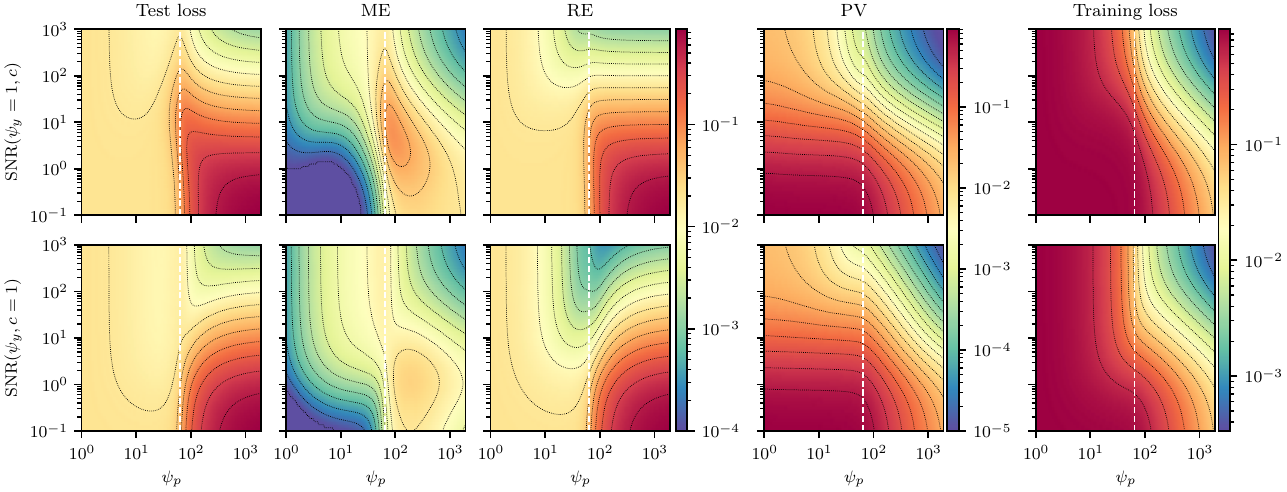}
  \caption{Test loss $(h_t/a_t^2)\Ltest=\ME+\RE$, $\ME$, $\RE$, $\PV$ and $\Ltrain$ of Theorem~\ref{thm:main} at
  $\psi_n=64$; settings as in Appendix~\ref{sec:app-heatmaps}.}
  \label{fig:heatmap-psin64}
\end{figure}

\begin{figure}[H]
  \centering
  \includegraphics[width=\textwidth]{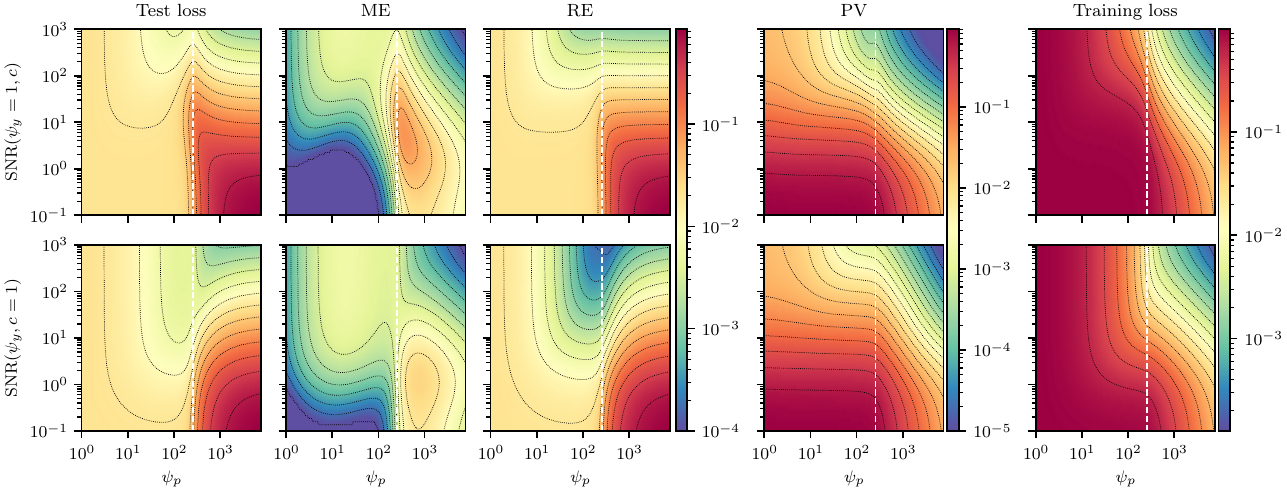}
  \caption{Test loss $(h_t/a_t^2)\Ltest=\ME+\RE$, $\ME$, $\RE$, $\PV$ and $\Ltrain$ of Theorem~\ref{thm:main} at
  $\psi_n=256$; settings as in Appendix~\ref{sec:app-heatmaps}.}
  \label{fig:heatmap-psin256}
\end{figure}

Figures~\ref{fig:heatmap-t0001}--\ref{fig:heatmap-t1} repeat
Figure~\ref{fig:heatmap-psin32} at $t\in\{0.001,0.1,1\}$. A smaller $t$ sharpens the peak
near $\psi_p=\psi_n$, and a larger $t$ smooths it; at $t=1$ the errors no longer peak there.

\begin{figure}[H]
  \centering
  \includegraphics[width=\textwidth]{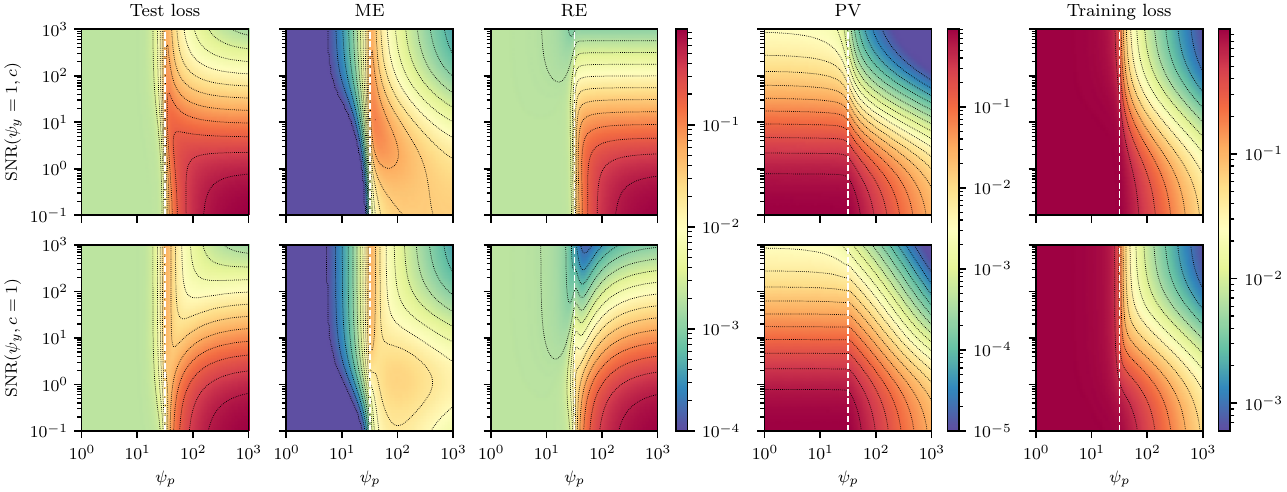}
  \caption{As Figure~\ref{fig:heatmap-psin32} ($\psi_n=32$) at $t=0.001$.}
  \label{fig:heatmap-t0001}
\end{figure}

\begin{figure}[H]
  \centering
  \includegraphics[width=\textwidth]{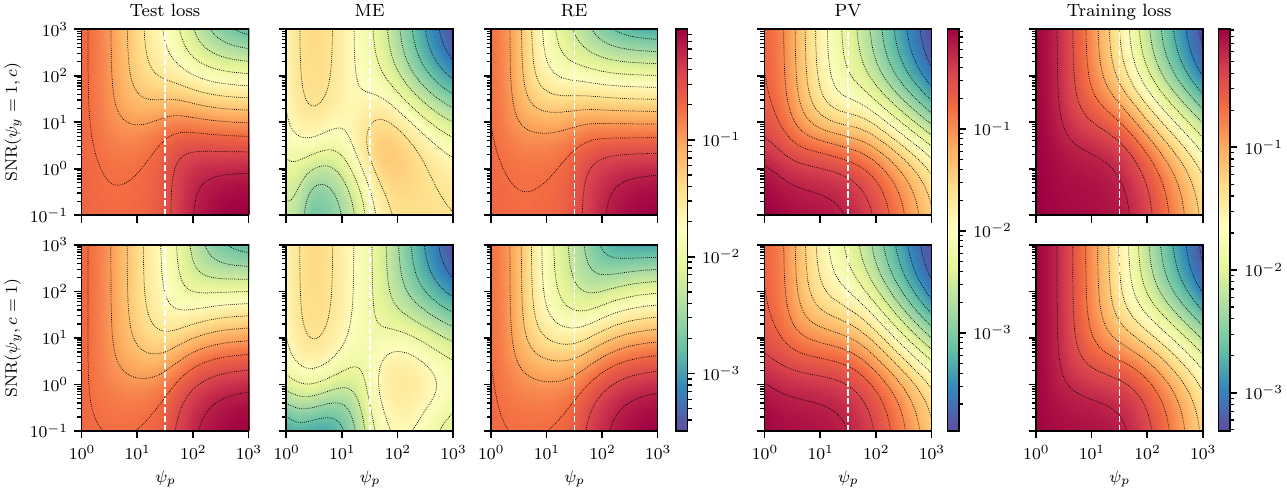}
  \caption{As Figure~\ref{fig:heatmap-psin32} ($\psi_n=32$) at $t=0.1$.}
  \label{fig:heatmap-t01}
\end{figure}

\begin{figure}[H]
  \centering
  \includegraphics[width=\textwidth]{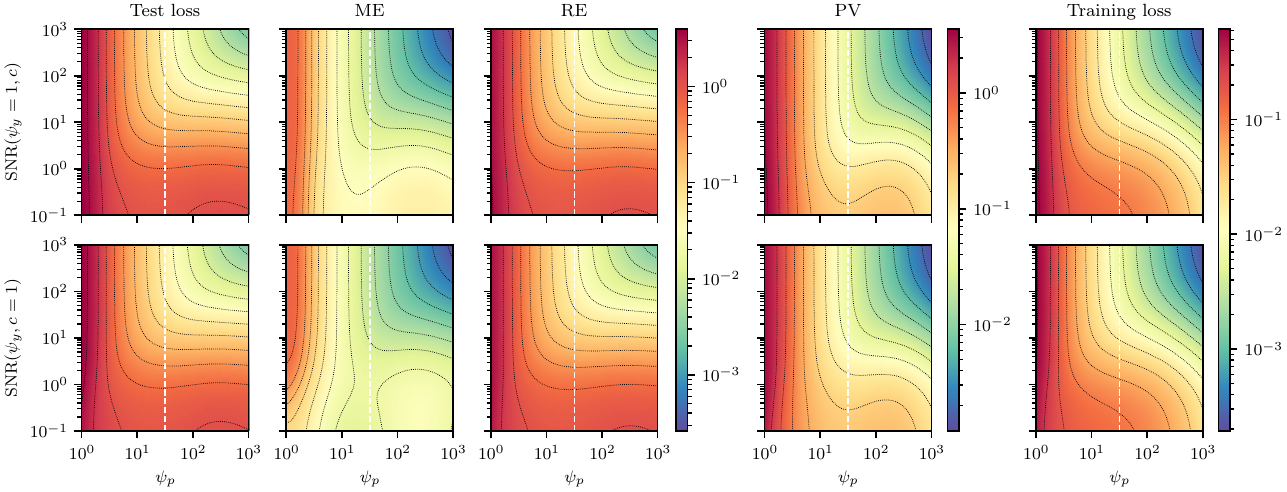}
  \caption{As Figure~\ref{fig:heatmap-psin32} ($\psi_n=32$) at $t=1$.}
  \label{fig:heatmap-t1}
\end{figure}

Figure~\ref{fig:heatmap-condition} varies $c$ and $\psi_y$ jointly at $\psi_n=256$ and
$\psi_p/\psi_n\in\{2,8,32,64\}$, with the other settings as above. Its axes are the SNRs of
the two specializations, $\mathrm{SNR}(1,c)=c^2/(1-c^2)$ and $\mathrm{SNR}(\psi_y,1)=\psi_y/(1-\psi_y)$.
At every width the training loss decreases in both $c$ and $\psi_y$, but it becomes small only
when both are large, that is, when $\mathrm{SNR}(\psi_y,c)$ is large, as in
Proposition~\ref{thm:train}. The training loss differs from $R_{\mathrm{train}}(\psi_p/\psi_n)$
in~\eqref{eq:Rtrain} by at most $2\times10^{-3}$, but by up to $46\%$ in relative terms where both
SNRs are largest and the loss is smallest.

\begin{figure}[H]
  \centering
  \includegraphics[width=\textwidth]{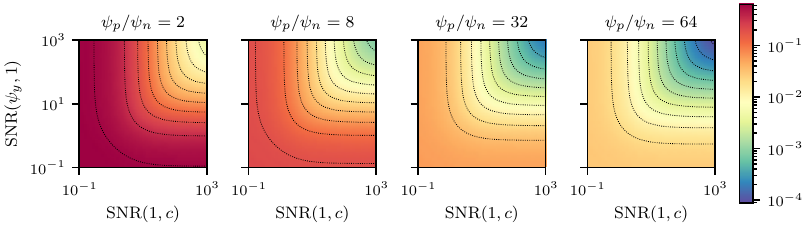}
  \caption{Training loss $\Ltrain$ of Theorem~\ref{thm:main} over $\mathrm{SNR}(1,c)$ and
  $\mathrm{SNR}(\psi_y,1)$ for $\psi_p/\psi_n\in\{2,8,32,64\}$, at $\psi_n=256$, $t=0.01$,
  $\act(u)=\tanh(u)$ and $\lambda=10^{-5}$. The panels share a logarithmic color scale, and
  dotted lines are level sets.}
  \label{fig:heatmap-condition}
\end{figure}

\subsection{Additional CelebA samples}
\label{sec:app-celeba-samples}

Figures~\ref{fig:celeba-condition-appendix}--\ref{fig:celeba-denoise-width-appendix} show
samples at $n=32768$, the setting of Figure~\ref{fig:celeba-width}, for the same six training
and six test conditions in every figure. Figures~\ref{fig:celeba-condition-appendix}
and~\ref{fig:celeba-denoise-appendix} vary the condition at $W=32$, where, averaged over three seeds, $91\%$,
$19\%$ and $0\%$ of the samples for training conditions are memorized for the $8\times8$,
$4\times4$ and $2\times2$ conditions, and $97\%$, $45\%$ and $3\%$ for $\sigma=0.1$, $0.4$ and $2$.
Figures~\ref{fig:celeba-qualitative-appendix} and~\ref{fig:celeba-denoise-width-appendix} vary
the width $W\in\{16,32,64\}$ with the strongest conditions, where these fractions are $12\%$,
$91\%$ and $99\%$ for the $8\times8$ condition and $52\%$, $97\%$ and $99\%$ for $\sigma=0.1$. In these figures, every memorized sample for a training condition copies its
source image, so whenever a row contains one, its last column is the source image. The samples
of Figure~\ref{fig:celeba-teaser} use $n=8192$ and $W=16$.

\clearpage
\begin{figure}[H]
  \centering
  \includegraphics[width=0.9\textwidth]{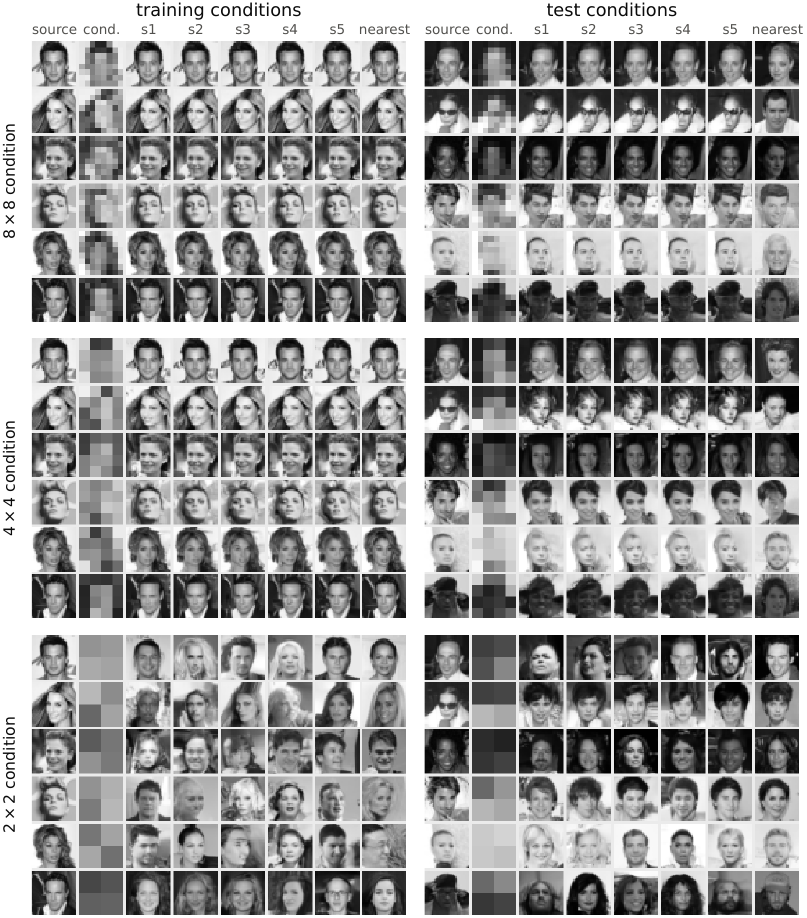}
  \caption{Samples at $n=32768$ and $W=32$ ($200$k steps) for super-resolution with the
  $8\times8$, $4\times4$ and $2\times2$ conditions. Each row is one condition: the source image,
  the condition made from it (nearest-upsampled), five samples generated from it, and the
  training image nearest to any of them. \textbf{Left:} training conditions; \textbf{right:}
  test conditions.}
  \label{fig:celeba-condition-appendix}
\end{figure}

\begin{figure}[H]
  \centering
  \includegraphics[width=0.9\textwidth]{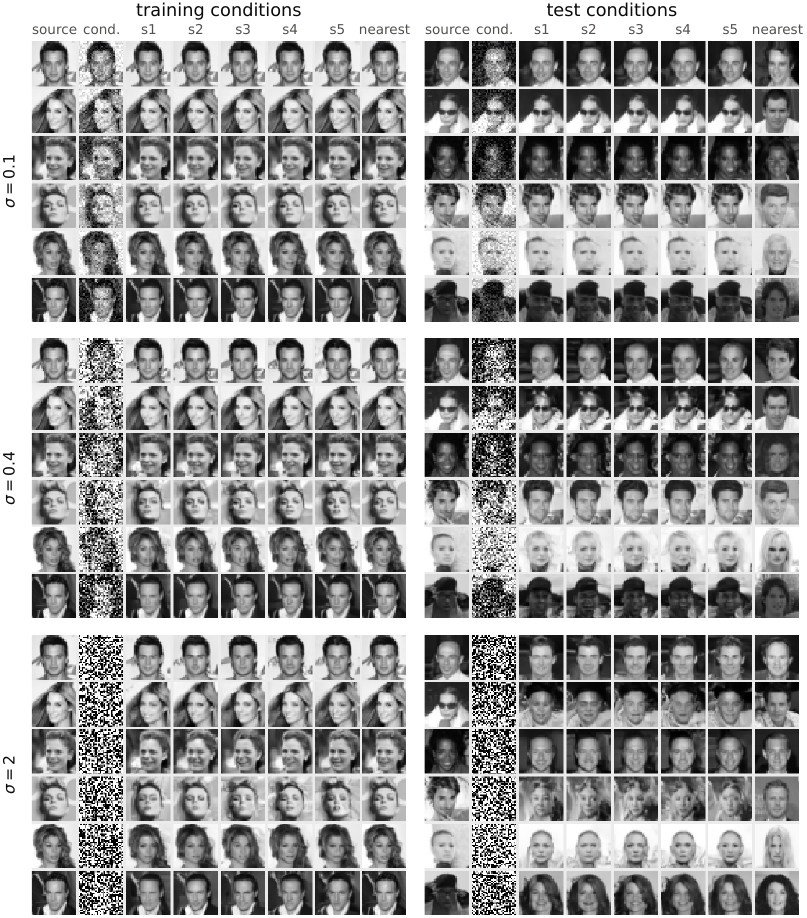}
  \caption{Samples at $n=32768$ and $W=32$ ($200$k steps) for Gaussian denoising with
  $\sigma=0.1$, $0.4$ and $2$; the noisy condition is clipped to $[0,1]$ for display. Rows and
  columns as in Figure~\ref{fig:celeba-condition-appendix}.}
  \label{fig:celeba-denoise-appendix}
\end{figure}

\begin{figure}[H]
  \centering
  \includegraphics[width=0.9\textwidth]{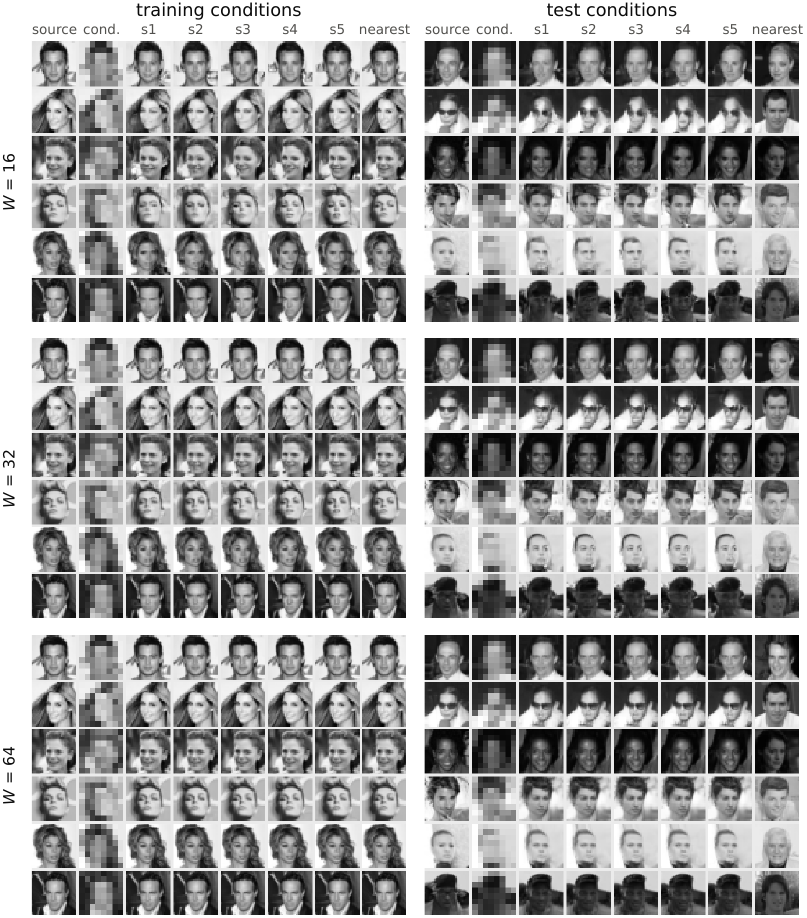}
  \caption{Samples at $n=32768$ for $W=16$, $32$ and $64$ ($200$k steps) for super-resolution
  with the $8\times8$ condition. Rows and columns as in
  Figure~\ref{fig:celeba-condition-appendix}.}
  \label{fig:celeba-qualitative-appendix}
\end{figure}

\begin{figure}[H]
  \centering
  \includegraphics[width=0.9\textwidth]{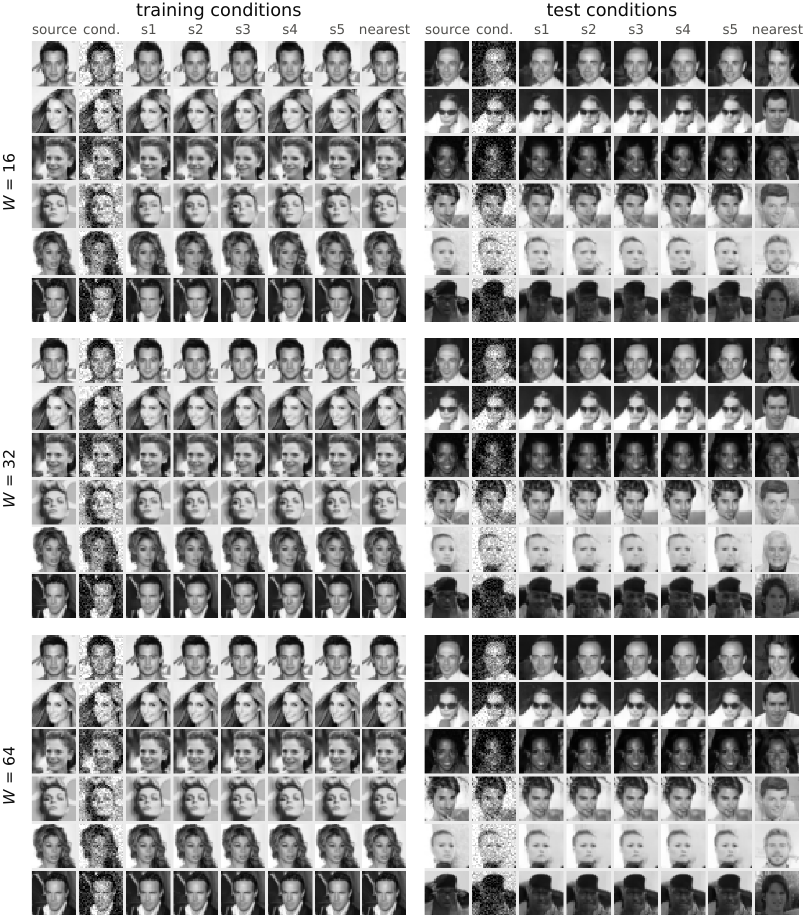}
  \caption{Samples at $n=32768$ for $W=16$, $32$ and $64$ ($200$k steps) for Gaussian
  denoising with $\sigma=0.1$. Rows and columns as in
  Figure~\ref{fig:celeba-condition-appendix}.}
  \label{fig:celeba-denoise-width-appendix}
\end{figure}

\section{AI use statement}
\label{sec:ai-use}
In this work, we used generative AI tools to assist in the
writing of proofs, design and provide feedback on the experiments, implement methods, assist with
translation, clean and reformat datasets, and interpret results. We have not used generative AI
tools to formulate mathematical claims, to generate synthetic data sets, to help develop theoretical models or conceptual
frameworks, to provide critical ingredients for proving mathematical claims, or to propose or
refine hypotheses, and qualitative and
thematic data analysis is not applicable to this work. Additionally, we used generative AI tools
to create and modify scientific figures, to create and edit software code, and to edit the paper
to improve readability. We have reviewed all AI-assisted work:
the authors checked all AI-assisted derivations, proofs, and code. We take responsibility for the final content
of this work, including text, claims, or artifacts produced with the aid of generative AI.

\end{document}